\documentclass[10pt]{article}
\usepackage[preprint]{tmlr}

\usepackage{amsmath,amsfonts,bm}

\def\eqref#1{(\ref{#1})}

\def\1{\bm{1}}

\DeclareMathAlphabet{\mathsfit}{\encodingdefault}{\sfdefault}{m}{sl}
\SetMathAlphabet{\mathsfit}{bold}{\encodingdefault}{\sfdefault}{bx}{n}

\usepackage{xcolor}
\colorlet{revcolor}{blue}
\newcommand{\rev}[1]{\ifvmode\leavevmode\fi\begingroup\color{revcolor}#1\endgroup}

\definecolor{cvprblue}{rgb}{0.21,0.49,0.74}
\usepackage[breaklinks,colorlinks,allcolors=cvprblue]{hyperref}

\let\cite\citep
\usepackage{url}
\usepackage{graphicx}
\usepackage{amsmath}
\usepackage{amssymb}
\usepackage{mathtools}
\usepackage{amsthm}
\usepackage{etoolbox}
\makeatletter
\newcommand{\UseSmallDisplay}[1]{%
  \BeforeBeginEnvironment{#1}{\begingroup\small}%
  \AfterEndEnvironment{#1}{\endgroup}%
}
\makeatother
\UseSmallDisplay{equation}
\UseSmallDisplay{equation*}

\newtheorem{corollary}{Corollary}

\usepackage{booktabs}
\usepackage{multirow}
\usepackage{colortbl}
\usepackage{arydshln}
\usepackage{subcaption}
\usepackage{caption}
\usepackage{float}
\usepackage{placeins}
\usepackage{wrapfig}
\usepackage{needspace}
\usepackage[most]{tcolorbox}
\definecolor{promptblue}{HTML}{000000}
\definecolor{promptback}{HTML}{F5F8FB}
\newtcolorbox{promptbox}[1]{
  enhanced,
  breakable,
  colback=promptback,
  colframe=promptblue!55,
  colbacktitle=promptblue,
  coltitle=white,
  fonttitle=\small\bfseries,
  title={#1},
  boxrule=0.45pt,
  arc=1.5mm,
  left=2.5mm,
  right=2.5mm,
  top=1.5mm,
  bottom=1.5mm,
  before skip=6pt,
  after skip=6pt
}

\usepackage{algorithm}

\usepackage{algorithmic}

\usepackage{xspace}
\makeatletter
\DeclareRobustCommand\onedot{\futurelet\@let@token\@onedot}
\def\@onedot{\ifx\@let@token.\else.\null\fi\xspace}

\makeatother

\newcommand{\txf}[1]{}
\makeatletter
\newcommand{\AppendixTocStart}{%
  \setcounter{tocdepth}{2}%
  \renewcommand{\l@section}{\@dottedtocline{1}{0em}{1.8em}}%
  \renewcommand{\l@subsection}{\@dottedtocline{2}{1.8em}{2.4em}}%
}
\makeatother

\newcommand{\printSupplementToc}{%
  \begingroup
  \renewcommand{\contentsname}{Table of Contents}
  \setcounter{tocdepth}{-1}%
  \setlength{\parskip}{0pt}%
  \renewcommand{\baselinestretch}{0.95}\selectfont
  \tableofcontents
  \endgroup
}

\title{Dual-Grained Policy Optimization for Hallucination Reduction in Flow-Based RL}

\author{%
\name Xiaofeng Tan$^{1,3,\dag,*}$\quad Jun Liu$^{3,\dag}$\quad
Yuanting Fan$^{3}$\quad Bin-Bin Gao$^{3}$\quad Xi Jiang$^{2}$\\
\name Xiaochen Chen$^{3}$\quad Jinlong Peng$^{3}$\quad
Chengjie Wang$^{3}$\quad Hongsong Wang$^{1,\ddag}$\quad Feng Zheng$^{2,\ddag}$\\
\addr $^{1}$Southeast University\quad
$^{2}$Southern University of Science and Technology\quad
$^{3}$Tencent Youtu Lab
}

\def\month{MM}
\def\year{YYYY}
\def\openreview{\url{https://openreview.net/forum?id=XXXX}}

\begin{document}

\maketitle
\renewcommand{\thefootnote}{\fnsymbol{footnote}}%
\footnotetext[1]{Work done during Xiaofeng Tan's internship at Tencent Youtu Lab.\quad
$^{\dag}$Equal contribution,\quad $^{\ddag}$Corresponding authors.}%
\renewcommand{\thefootnote}{\arabic{footnote}}

\vspace{-20pt}
\begin{figure}[H]
\centering
\setlength{\abovecaptionskip}{4pt}%
\setlength{\belowcaptionskip}{0pt}%
\includegraphics[width=0.86\linewidth]{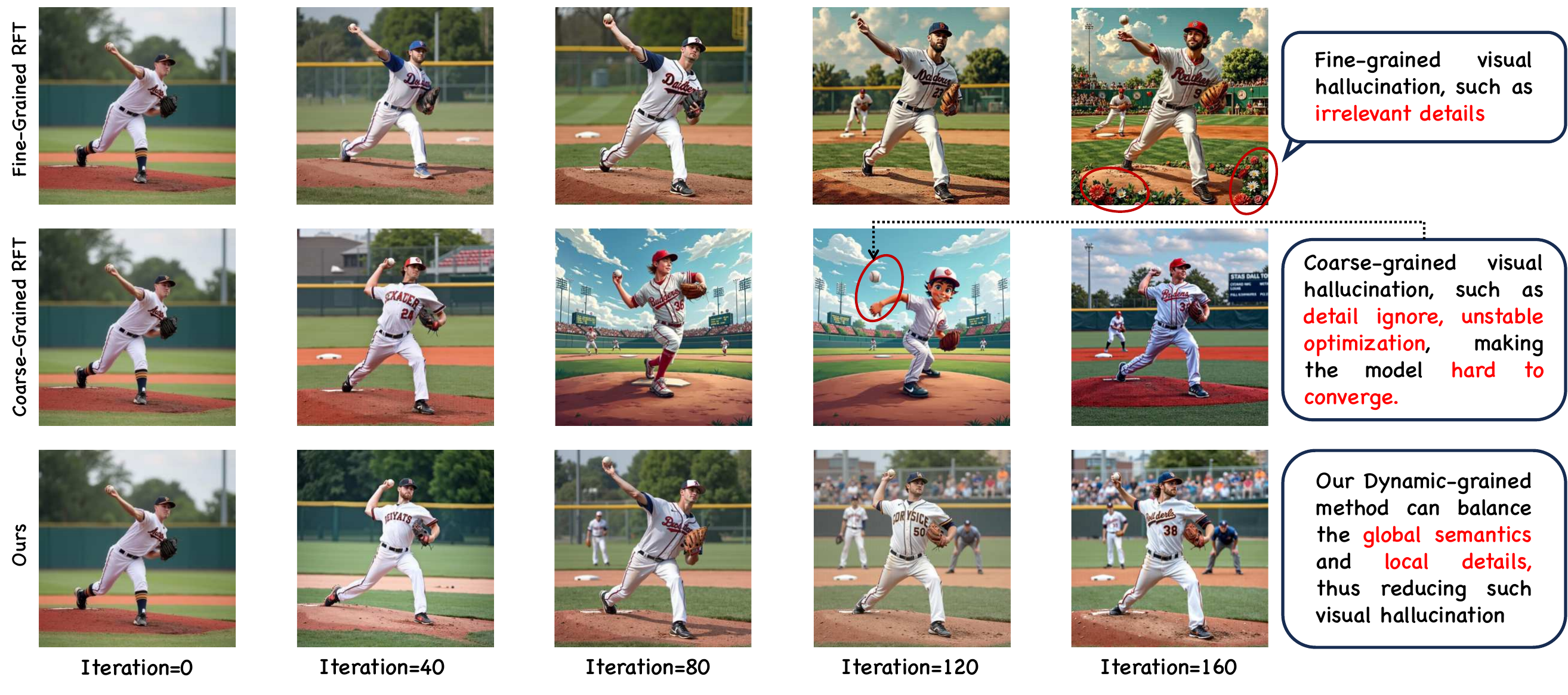}
\caption{\textbf{The optimization process of existing flow-based RFT methods.} Existing methods only perform fine- or coarse-grained optimization, leading to visual hallucinations. \textit{Prompt: ``A baseball player pitching a baseball on a field.''}}
\label{fig:intro}
\end{figure}
\vspace{-10pt}

\begin{abstract}
    Reinforcement Fine-Tuning (RFT) on flow-based models is crucial for alignment. However, they often introduce \textit{visual hallucinations} like over-optimized details and semantic misalignment. This work explains \textit{why} visual hallucinations arise and shows \textit{how} to reduce them. We first investigate RFT methods from a unified perspective and trace hallucinations to two mechanisms. Limited exploration from stochastic differential equation (SDE) rollouts yields detail-dominated within-group differences, leading to \textit{fine-grained optimization} that over-emphasizes local details and neglects global semantics. Trajectory imitation writes reward-selected SDE noise into the velocity field, leading to \textit{inconsistent optimization} that degrades cross-step consistency. Building on this, we propose DGPO, a general framework to mitigate hallucinations. Specifically, we design Dynamic Granularity Rollout (DGR) to balance global semantics and local details by dynamically scheduling noise sources. We then introduce a Consistent Policy Gradient Optimization that preserves the consistency by aligning the policy with a more stable prior. Experiments show that DGPO reduces low-/high-level hallucinations by 49\%/38\% and improves the average of out-of-domain (OoD) metrics by +4.9\% on FLUX\hspace{0pt}1.dev, compared with the baseline's -0.7\%. Our \href{https://xiaofeng-tan.github.io/projects/ConsistentRFT/}{project page} is publicly available.
\end{abstract}

\section{Introduction}
\label{sec:intro}

Flow matching \cite{liu2022flow} models are prominent across image \cite{lipman2022flow}, video \cite{kong2024hunyuanvideo}, and robot manipulation \cite{zhang2024affordance}, yet often lack semantic or spatial understanding~\cite{su2026generationenhancesunderstandingunified}. Reinforcement Fine-Tuning (RFT) \cite{fan2023dpok, wallace2024diffusion} effectively aligns diffusion models with complex semantic objectives \cite{liu2024alignment} using external reward models \cite{wu2023human}. To relax the deterministic ODE sampling in flow models \cite{lipman2022flow}, recent work \cite{liu2025flow} explores RFT by recasting sampling as an SDE~\cite{song2021score}.

Despite their successes, these methods \cite{li2025mixgrpo} often struggle to generate consistent, realistic images due to \textit{visual hallucination}, as illustrated in Fig.~\ref{fig:intro}. This includes \textit{detail over-optimization}~\cite{zhang2024large}, such as irrelevant details, over-sharpening, and grid artifacts~\cite{dance_grpo_issue1}. It also manifests as \textit{semantic misalignment} with the prompt and \textit{cross-step inconsistency}, characterized by semantic degradation and noticeable divergence under few-step sampling. Together, these issues limit the reliability of current methods, particularly in OoD scenarios.

Although hallucinations of generative models have been extensively studied in large language models \cite{huang2025survey}, research on visual hallucinations remains limited. Prior works \cite{aithal2024understanding, oorloff2025mitigating, fu2025counting, kim2024tackling} primarily examine pretraining-phase hallucinations, which may arise from inappropriate smoothing of training data. In contrast to supervised learning (SL), RL optimizes models without real-image data, rendering such explanations insufficient. This raises a central yet underexplored question: \textit{Why does reinforcement fine-tuning lead to, and even exacerbate, visual hallucinations in flow-based models?}

We consider this issue from the perspective of optimization in RL. Different from SL, in RL, the model learns through interaction with a reward model. Thus, the rollout data determines the optimization direction, while the reward signal determines the optimization effectiveness. As shown in Fig.~\ref{fig:intro_teaser}, we observe that the hallucinations in RFT are stem from the \textit{inaccurate reward signals} caused by \textit{fine-grained and inconsistent rollout behavior}. However, improving the reward model remains challenging, as it is trained on real-world distribution and struggles to generalize to the AI-style, diverse, and unpredictable  modern generators distribution. Moreover, reward signals alone cannot explain detail over-optimization and cross-step inconsistency, motivating more reliable RL algorithms to mitigate hallucinations and reward errors.

\begin{figure}[t]
    \centering
    \setlength{\abovecaptionskip}{3pt}%
    \setlength{\belowcaptionskip}{-10pt}%
    \includegraphics[width=0.8\linewidth]{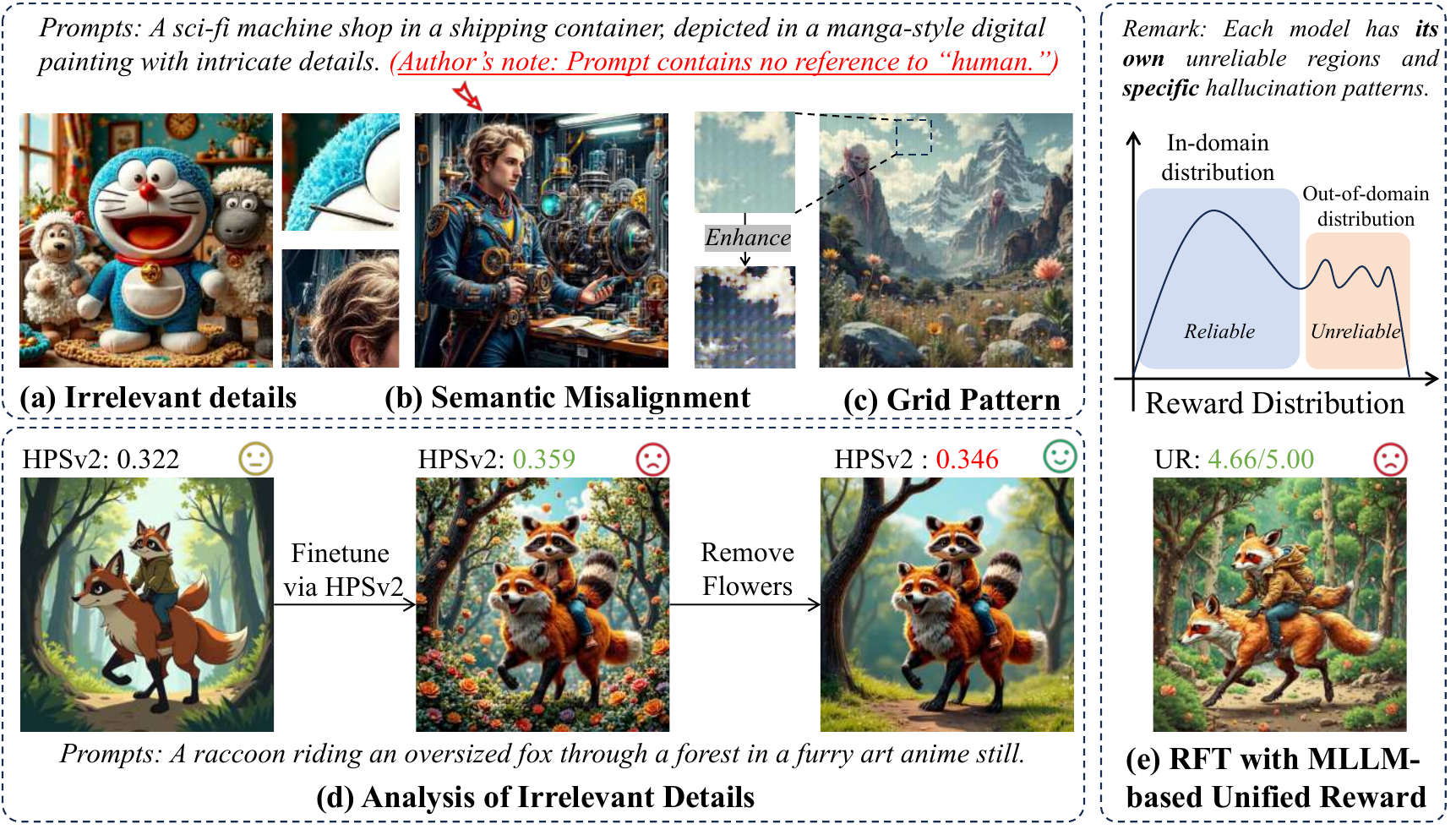}
\caption{\textbf{Visual hallucinations in flow-based RFT.} Fig. (d,e) show unreliable rewards on diverse/OoD samples encourage irrelevant details, and (e) shows that limited exploration can trap the model in reward-specific local optima.}
\label{fig:intro_teaser}
\end{figure}

\noindent \textbf{Contributions.} In this work, we aim to reduce visual hallucinations in flow-based models. To this end, we first analyze its potential causes and subsequently propose Dual-Grained Policy Optimization (DGPO), a general method for online RFT methods. Our contributions are highlighted below.

First, we investigate how exploration and exploitation contribute to visual hallucinations in flow-based RFT. Exploration affects the optimization direction. Corollary~\ref{cor:local_group_update} shows that group-relative updates draw their learning signal only from within-group differences. When SDE rollouts share the initial noise, these differences are dominated by local appearance, so the feedback concentrates on local details while global semantics receive little supervision. Exploitation determines optimization effectiveness. Corollary~\ref{thm:reinterpretation} rewrites flow-based policy gradients as advantage-weighted regression onto SDE-sampled transitions and shows that the velocity field moves only along reward-selected noise directions, which are not constrained by the ODE geometry of flow models. This degrades cross-step consistency: every SDE-based baseline raises the latent consistency error above that of the base model, and restoring it with DGPO coincides with substantially lower hallucination scores. This motivates \textit{how to reduce visual hallucinations} from RFT.

Building on this insight, we introduce \textit{Dynamic Granularity Rollout (DGR)} to balance preference signals between \textit{global semantics and local details}. We observe that the global semantics and local details in somewhat is determined by the initial and process noise, respectively. To balance global and local optimization, we dynamically steer exploration with progressive and initial noise, and introduce a progress-aware clustered rollout that preserves diversity with lower computation. Together, these designs balance global and local information and mitigate detail over-optimization.

Third, we propose \textit{Consistent Policy Gradient Optimization (CPGO)} for DPO, DDPO, and GRPO to address the cross-step semantic inconsistencies introduced by trajectory imitation. Our goal is to preserve the velocity field consistency in flow models along the generation trajectory. Specifically, we enforce single-step prediction consistency between the current model $\theta$ at step $t$ and a teacher model $\theta_{\mathrm{old}}$ at step $t{-}1$. This constraint mitigates trajectory drift while preserving the optimization benefits of reinforcement fine-tuning.

Finally, we introduce a \textit{Visual Hallucination Evaluator} to assess artifacts and perceptual hallucinations, alongside semantic and trajectory consistency. Experiments show that DGPO reduces low- and high-level hallucinations by 49\% and 38\%, respectively, and improves the average OoD metric on FLUX\hspace{0pt}1.dev by 4.9\%, versus -0.7\% for the baseline. DGPO also integrates with DDPO and DPO.

\section{Related Work}

\noindent \textbf{Hallucinations in Generative Models.} Recent studies link hallucinations to optimization and sampling. For LLMs, \citet{li2025reasoninghallucinate} find that reasoning-oriented RL fine-tuning can increase hallucinations and propose factuality-aware step-wise policy optimization. In image generation, adaptive attention modulation reduces diffusion hallucinations~\cite{oorloff2025mitigating}, while CountHalluSet~\cite{fu2025counting} quantifies counting hallucinations across sampling conditions. For flow models, Iterative Flow Matching~\cite{haber2026iterativeflow} addresses unrealistic samples through path correction and gradual refinement. LucidNFT~\cite{fei2026lucidnft} combines an input-referenced evaluator with multi-reward preference optimization for faithful flow-based super-resolution. Visual hallucinations in flow-model post-training remain an open challenge, motivating our study of rollout granularity and cross-step consistency in text-to-image RFT.

\noindent \textbf{RFT in Flow-based Model.} RFT aligns generative models with human preferences~\cite{liu2024alignment} and semantic objectives~\cite{weng2025realign, tan2024sopo}. Early work introduced policy-gradient methods such as DDPO and DPOK~\cite{black2023training, fan2023dpok}, as well as direct preference optimization methods~\cite{wallace2024diffusion, yang2024using}, to fine-tune diffusion models. For flow-based models, recent methods such as DanceGRPO~\cite{xue2025dancegrpo} and Flow-GRPO~\cite{liu2025flow} convert deterministic ODE sampling into stochastic SDE rollouts, enabling tractable likelihood estimation. Building on this, concurrent studies advance RFT along several aspects~\cite{yu2025smart, wang2025coefficients, sheng2025understanding, zheng2025diffusionnft, zhou2025text, li2025mixgrpo, wang2025pref, li2025branchgrpo, fu2025dynamic, he2025tempflow}, for example, mixed ODE--SDE sampling, pairwise preference learning, and structured trajectory sampling. Related techniques also extend to AR-based generation~\cite{yuan2025argrpo, zhang2025group} and world models~\cite{ye2025reinforcement}. Despite this progress, visual hallucinations persist. We examine rollout diversity and trajectory imitation as potential contributing mechanisms and introduce methods to mitigate their effects.

\section{Preliminaries: Unified Perspective on RFT}
To establish why visual hallucinations emerge and how to systematically mitigate them, we first formulate a unified perspective on existing RFT methods for flow-based generative models.

\subsection{Unified Paradigm: Exploration and Exploitation}
RFT adapts a pretrained flow model $\theta$ using reward feedback on a prompt dataset $\mathcal{D}=\{c_i\}$. Given a scalar reward function $r(\mathbf{x}, c)$ reflecting perceptual quality or text alignment, RFT seeks updated parameters $\theta'$ that maximize the expected reward for generated images $\mathbf{x}$ conditioned on $c$. Existing methods alternate between two coupled stages. \textit{Exploration} generates sample trajectories under condition $c \sim \mathcal{D}$ and evaluates their terminal quality with the reward model $r(\mathbf{x}_0, c)$. \textit{Exploitation} updates $\theta$ by optimizing a method-specific objective $\mathcal{J}(\theta; \mathcal{G}, r(\mathcal{G}))$ constructed from the rollouts and reward signals.

\noindent \textbf{Exploration.}
Given a flow model parameterized by $\theta$ and a condition $c \sim \mathcal{D}$, the exploration stage generates a group $\mathcal{G}$ consisting of $K$ rollout trajectories of length $T$:
\begin{equation*}
\mathcal{G} = \{\mathbf{x}^k_{0:T}\}_{k=1}^K = \pi_\theta(c, K), \qquad
\mathbf{x}^k_{0:T}=(\mathbf{x}^k_0,\mathbf{x}^k_1,\ldots,\mathbf{x}^k_T),
\end{equation*}
where $\mathbf{x}_t^k$ denotes the $k$-th sample state at discrete timestep $t \in \{0, \dots, T\}$, $K$ is the number of sampled trajectories, $T$ is the total number of sampling steps, and $\pi_\theta$ denotes an ODE- or SDE-based sampling process, e.g., DDIM~\cite{song2020denoising} or DDPM~\cite{ho2020denoising}.

To illustrate the sampling mechanics in representative flow-based RFT approaches, we introduce the SDE formulation adopted by DanceGRPO~\cite{xue2025dancegrpo}:
\begin{equation}
\mathrm{d}\mathbf{x}_t = \Big(\mathbf{v}_t - \tfrac{1}{2}\varepsilon_t^2 \nabla\log p_t(\mathbf{x}_t)\Big)\mathrm{d}t + \varepsilon_t\,\mathrm{d}\mathbf{w}_t,
\label{eq:sde}
\end{equation}
where $\mathbf{v}_t = v_\theta(\mathbf{x}_t, t)$ is the velocity field predicted by the flow network at timestep $t$, $\varepsilon_t$ is the noise schedule governing stochasticity, and $\mathbf{w}_t$ represents standard Brownian motion. For the linear interpolation with Gaussian noise, we estimate the score as $\nabla\log p_t(\mathbf{x}_t) \approx -\big(\mathbf{x}_t-(1-t)\cdot \hat{\mathbf{x}}_0\big)/t^2$, where $\hat{\mathbf{x}}_0 = \mathbf{x}_t - t\cdot \mathbf{v}_t$ denotes the clean sample recovered via a single-step first-order ODE prediction.

\noindent \textbf{Exploitation.}
Given the collection of $K$ rollout trajectories $\mathcal{G}$ and their terminal rewards $\{r(\mathbf{x}_0^k, c)\}_{k=1}^K$, the exploitation stage updates the model parameters $\theta$ by maximizing an objective function $\mathcal{J}(\theta; \mathcal{G}, r(\mathcal{G}))$. We review representative instantiations below.

\subsection{Representative Methods}

\noindent \textbf{DPO.}
Online Direct Preference Optimization (DPO)~\cite{calandriello2024human, kim2025vip, lanchantin2025bridging, wallace2024diffusion} is a preference-based alignment framework. In the exploration stage, it draws pairs of candidates per prompt using the rollout process defined above, with $K=2$. In the exploitation stage, denoting $\mathbf{x}^1 \succ \mathbf{x}^2$ as the preferred and dispreferred samples according to reward $r$, DPO maximizes the implicit preference margin while constraining policy drift relative to a frozen reference model $\theta_{\mathrm{ref}}$:
\begin{equation*}
\mathcal{J}_{\mathrm{DPO}}(\theta, \mathcal{G})
= \mathbb{E}_{c\sim\mathcal{D},\,\mathbf{x}^1\succ\mathbf{x}^2,\,t\sim\mathcal{U}(0,T)}
\big[\log \sigma(m_t)\big],
\end{equation*}
with the per-timestep preference margin
\begin{equation*}
m_t=\beta \log
         \frac{p_{\theta}(\mathbf{x}^{1}_{t-1}\mid \mathbf{x}^{1}_t, c)}{p_{\theta_{\mathrm{ref}}}(\mathbf{x}^{1}_{t-1}\mid \mathbf{x}^{1}_t, c)}
         - \beta \log \frac{p_{\theta}(\mathbf{x}^{2}_{t-1}\mid \mathbf{x}^{2}_t, c)}{p_{\theta_{\mathrm{ref}}}(\mathbf{x}^{2}_{t-1}\mid \mathbf{x}^{2}_t, c)},
\end{equation*}
where $\sigma(z)=\frac{1}{1+e^{-z}}$ is the sigmoid activation function and $\beta>0$ is a temperature hyperparameter that balances reward maximization against KL regularization to the reference policy.

\noindent \textbf{DDPO.}
As a foundational policy-gradient reinforcement learning method for generative diffusion models, DDPO~\cite{black2023training, fan2023dpok} treats the reverse sampling chain as a multi-step Markov decision process and optimizes the model parameters directly with policy gradient on $K$ rollouts:
\begin{equation*}
\nabla_{\theta} \mathcal{J}_{\mathrm{DDPO}}(\theta, \mathcal{G})
= \mathbb{E}_{c\sim\mathcal{D},\; \mathbf{x}^{1:K}\sim\mathcal{G},\; t\sim\mathcal{U}(0,T),\; k\sim\mathcal{U}(1,K)}
\nabla_{\theta} \left[r(\mathbf{x}^k_0, c) \cdot
    \log p_{\theta}(\mathbf{x}^{k}_{t-1} \mid \mathbf{x}^{k}_t, c)
  \right],
\end{equation*}
where the transition log-likelihood gradient $\nabla_\theta \log p_\theta(\mathbf{x}_{t-1}^k \mid \mathbf{x}_t^k, c)$ at every intermediate step $t$ is directly weighted by the final scalar reward $r(\mathbf{x}_0^k, c)$ of the completed image. Because the intermediate transition dynamics are parameterized as Gaussian conditionals $p_\theta(\mathbf{x}_{t-1} \mid \mathbf{x}_t, c) = \mathcal{N}\big(\mathbf{x}_{t-1}; \mu_\theta(\mathbf{x}_t, t, c), \eta^2 \Delta t\,\mathbf{I}\big)$, the likelihood gradient $\nabla_\theta \log p_\theta$ corresponds to regressing the network's predicted mean $\mu_\theta$ toward the realization of the noisy state $\mathbf{x}_{t-1}$, directly scaling the gradient step by the scalar reward magnitude.

\noindent \textbf{GRPO.}
Group Relative Policy Optimization (GRPO)~\cite{guo2025deepseek, xue2025dancegrpo} avoids training an auxiliary value network by estimating advantages $\mathcal{A}^k$ relative to the empirical mean and variance of a group of $K$ parallel rollouts. It optimizes a clipped surrogate policy gradient objective:
\begin{equation*}
\mathcal{J}_{\mathrm{GRPO}}(\theta, \mathcal{G})
= \mathbb{E}_{c\sim\mathcal{D},\, \mathbf{x}^{1:K}\sim\mathcal{G},\, t\sim\mathcal{U}(0,T),\, k\sim\mathcal{U}(1,K)}
\Big[
\min\!\big( \rho_{\theta}^{k,t}\,\mathcal{A}^{k},\;
\mathrm{clip}(\rho_{\theta}^{k,t},\,1-\varepsilon,\,1+\varepsilon)\,\mathcal{A}^{k} \big)
\Big],
\end{equation*}
where the group-normalized advantage $\mathcal{A}^{k}$ and the importance weight ratio $\rho_{\theta}^{k,t}$ are computed as follows:
\begin{equation}
\mathcal{A}^{k}=\frac{r(\mathbf{x}^{k}_0,c)-\mu(r(\mathbf{x}^{1:K},c))}{\sigma(r(\mathbf{x}^{1:K},c))},
\;
\rho_{\theta}^{k,t}=\frac{p_{\theta}(\mathbf{x}^{k}_{t-1}\mid\mathbf{x}^{k}_t,c)}{p_{\theta_{\mathrm{old}}}(\mathbf{x}^{k}_{t-1}\mid\mathbf{x}^{k}_t,c)}.
\label{eq:grpo}
\end{equation}
in which $\mu(r(\mathbf{x}^{1:K},c))$ and $\sigma(r(\mathbf{x}^{1:K},c))$ denote the empirical mean and standard deviation of the reward values evaluated across the group $\mathbf{x}^{1:K}$, $\theta_{\mathrm{old}}$ is the checkpoint of the policy prior to the current update step, and $\varepsilon > 0$ denotes the PPO clipping bound that prevents destructively large policy updates. By normalizing rewards across the peer group, $\mathcal{A}^k$ acts as a zero-mean baselined advantage satisfying $\sum_k \mathcal{A}^k = 0$. Above-average trajectories have $\mathcal{A}^k > 0$, while below-average ones have $\mathcal{A}^k < 0$. In the unclipped regime, the gradient pulls the policy mean $\mu_\theta(\mathbf{x}_t, t, c)$ toward the positive-advantage SDE trajectories and pushes it away from negative-advantage ones.

\noindent \textbf{Summary.} These methods share an SDE-based exploration via rollouts but differ in their exploitation strategies. Intuitively, samples generated during exploration serve as the informational basis for exploitation, thereby shaping the model's optimization behavior. This view motivates our investigation into the causes of visual hallucinations.


\section{Motivation: Why do Hallucinations Arise?}
\label{sec:motivation}

We study visual hallucinations from two complementary perspectives. \textit{Optimization Direction} concerns limited rollout diversity, which steers optimization toward fine-grained details at the expense of global semantics. \textit{Optimization Effectiveness} concerns imitation of SDE-sampled trajectories, which conflicts with the ODE-oriented velocity field of flow models and can further propagate SDE-induced artifacts.

\begin{figure}[t]
\centering
\includegraphics[width=0.98\linewidth]{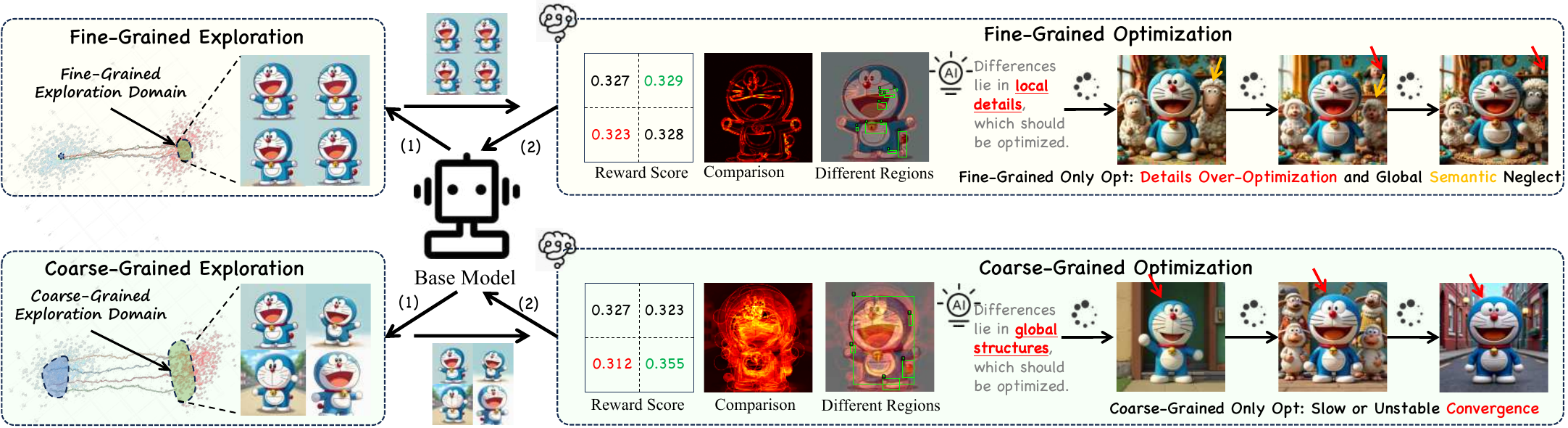}
\caption{\textbf{Relationship between rollout exploration and policy optimization.} Coarse/fine-grained groups from rollout steer the model toward coarse/fine-grained optimization. See App.~\ref{app:initial_noise_analysis} and App.~\ref{app:progressive_perception} for analysis of rollout granularity and coarse perception.}
\label{fig:motivation}
\end{figure}

\subsection{Exploration: Limited Exploration Domain}
\label{sec:exploration}

RFT methods optimize generative models by exploiting the variation among rollout samples and the corresponding discrepancies in their rewards. Consequently, the distribution and diversity of rollout samples fundamentally determine the trajectory of policy updates.

Ideally, an effective exploration strategy should supply candidate samples that exhibit both global semantic diversity and fine-grained variation, enabling the reward model to guide both coarse layout alignment and fine detail refinement. In practice, however, methods such as DanceGRPO~\cite{xue2025dancegrpo, li2025mixgrpo, li2025branchgrpo} use SDE-based exploration with shared initial noise across candidates in a group, varying only the step-wise stochastic noise. As shown in Fig.~\ref{fig:motivation}, sharing initial noise largely confines sample variations to local appearance, yielding near-identical global layouts within each group. Consequently, \textbf{the exploration space is heavily biased toward fine-grained perturbations}, while global semantic exploration is largely suppressed.

For a fixed prompt $c$, consider a rollout group $\mathcal{G}=\{\mathbf{x}_{0:T}^k\}_{k=1}^K$, with advantages $\mathcal{A}^k$ from Eq.~\eqref{eq:grpo}. Define the trajectory likelihood gradient as $\boldsymbol{\psi}_\theta(\mathbf{x}_{0:T},c)=\frac{1}{T}\sum_{t=1}^{T}\nabla_\theta\log p_\theta(\mathbf{x}_{t-1}\mid\mathbf{x}_t,c)$. At the start of an update, $\theta=\theta_{\mathrm{old}}$, all ratios $\rho_\theta^{k,t}$ equal one, and clipping is inactive. The contribution of this group to the GRPO gradient is therefore $\mathbf{g}_{\mathcal{G}}=\frac{1}{K}\sum_{k=1}^{K}\mathcal{A}^k\boldsymbol{\psi}_\theta(\mathbf{x}_{0:T}^k,c)$, with all likelihood gradients below evaluated at $\theta=\theta_{\mathrm{old}}$.

\begin{corollary}[Group-Relative Updates Depend Only on Within-Group Differences]
  \label{cor:local_group_update}
  Assume $\sigma(r(\mathbf{x}^{1:K},c))>0$, so the group-normalized advantages satisfy $\sum_{k=1}^K\mathcal{A}^k=0$. For any reference trajectory $\bar{\mathbf{x}}_{0:T}$,
  \begin{equation}
  \mathbf{g}_{\mathcal{G}}=\frac{1}{K}\sum_{k=1}^{K}\mathcal{A}^k
  \big(\boldsymbol{\psi}_\theta(\mathbf{x}_{0:T}^k,c)-\boldsymbol{\psi}_\theta(\bar{\mathbf{x}}_{0:T},c)\big).
  \label{eq:group_difference_identity}
  \end{equation}
  Let $D_{\mathbf{x}_{0:T}}$ and $D^2_{\mathbf{x}_{0:T}}$ denote the first and second derivatives with respect to the stacked trajectory states, holding $\theta$ and $c$ fixed. If $\boldsymbol{\psi}_\theta$ is twice continuously differentiable on a convex set containing the reference and all sampled trajectories, with $\|D^2_{\mathbf{x}_{0:T}}\boldsymbol{\psi}_\theta\|_{\mathrm{op}}\le M$, define $\boldsymbol{\delta}^k=\mathbf{x}_{0:T}^k-\bar{\mathbf{x}}_{0:T}$. Then
  \begin{equation}
  \mathbf{g}_{\mathcal{G}}
  = \frac{1}{K}\mathbf{J}_{\bar{\mathbf{x}}}
  \sum_{k=1}^{K}\mathcal{A}^k\boldsymbol{\delta}^k + \mathbf{e}_{\mathcal{G}},
  \qquad
  \|\mathbf{e}_{\mathcal{G}}\|_2 \leq
  \frac{M}{2}\max_{k}\|\boldsymbol{\delta}^k\|_2^2,
  \label{eq:fine_grained_gradient}
  \end{equation}
  Here $\mathbf{J}_{\bar{\mathbf{x}}}=D_{\mathbf{x}_{0:T}}\boldsymbol{\psi}_\theta(\bar{\mathbf{x}}_{0:T},c)$ is the Jacobian, $M$ bounds the operator norm of the second derivative, and $\mathbf{e}_{\mathcal{G}}$ is the Taylor remainder.
\end{corollary}

\noindent \textbf{Consequence for Shared-Noise Groups.}
At $\theta=\theta_{\mathrm{old}}$, Eq.~\eqref{eq:group_difference_identity} shows that centered advantages cancel any likelihood-gradient component shared by all trajectories. Eq.~\eqref{eq:fine_grained_gradient} expresses the remaining update through reward-weighted trajectory differences, with a remainder quadratic in their magnitude. Thus, when within-group variation is dominated by local details, the learning signal primarily reflects those details rather than global semantic alternatives. The proof is provided in App.~\ref{app:proof_local_group_update}.

\noindent \textbf{Empirical Support.}
App.~\ref{app:initial_noise_analysis} reports lower within-group CLIP-feature diversity and reward contrast for shared initial noise than for independent initial noise, supporting the premise of limited exploration. Tab.~\ref{tab:ablation} further shows weaker CLIP and Unified Reward scores under fine-grained DanceGRPO than under its coarse-grained counterpart and the base model. The visual examples in Fig.~\ref{fig:motivation} and the hallucination evaluation in Sec.~\ref{sec:experiments} provide complementary evidence of detail over-optimization. Together, these observations support the interpretation that detail-dominated exploration can weaken global semantic alignment.

\subsection{Exploitation: Imitation on SDE-Induced Trajectory}

To analyze the optimization effectiveness of RFT methods, we examine their gradient updates from the viewpoint of \textit{trajectory imitation}. The parameter gradients show that policy-gradient updates regress the predicted reverse mean onto the sampled SDE transition states, with the magnitude and sign of the regression governed by the advantage. Taking GRPO as a primary archetype, we formalize this reinterpretation in Corollary~\ref{thm:reinterpretation}. Proofs for GRPO, DDPO, and DPO are in App.~\ref{app:theoretic:cannyor}--\ref{app:theoretic:cannyor_dpo}.

\begin{corollary}[Reinterpretation of GRPO as Trajectory Imitation]\label{thm:reinterpretation}
For Gaussian discretized SDE transitions with covariance $\eta^2\Delta t\,\mathbf{I}$ independent of $\theta$, treat the sampled states and advantages as fixed during optimization. Away from clipping boundaries, each unclipped transition contributes an advantage-weighted regression onto $\mathbf{x}_{t-1}$, while clipped contributions vanish:
\begin{equation}
\nabla_{\theta}\mathcal{J}_{\mathrm{GRPO}}
=\mathbb{E}\!\left[-\mathbb{I}(\neg u)\,\omega(t)\,\rho_{\theta}^{k,t}\,\mathcal{A}^{k}\,
\nabla_{\theta}\|\mu_{\theta}(\mathbf{x}_{t},t,c)-\mathbf{x}_{t-1}\|_{2}^{2}\right].
\label{eq:grpo_trajectory_regression}
\end{equation}
Here $u=(u_1\wedge v_1)\vee(u_2\wedge\neg v_1)$ denotes the clipping condition, and $\mathbb{I}(\cdot)$ is the indicator function, equal to 1 when its argument is true and 0 otherwise. Thus, $\mathbb{I}(\neg u)$ retains only transitions with an active gradient. $\mu_\theta$ is the mean of $p_\theta(\mathbf{x}_{t-1}\mid\mathbf{x}_t,c)$, $\omega(t)=1/(2\eta^2\Delta t)$, and $\rho_{\theta}^{k,t}=1$ at $\theta=\theta_{\mathrm{old}}$. The indicator variables are
\begin{equation*}
u_1=\mathbb{I}(\rho_{\theta}^{k,t}\le 1-\varepsilon),\quad
u_2=\mathbb{I}(\rho_{\theta}^{k,t}\ge 1+\varepsilon),\quad
v_1=\mathbb{I}(\mathcal{A}^{k}<0).
\end{equation*}
\end{corollary}

\noindent \textbf{Reinterpretation of Policy Gradient Exploitation.}
Within the clipping bounds, Corollary~\ref{thm:reinterpretation} rewrites the GRPO gradient as an advantage-weighted regression of the predicted reverse mean $\mu_{\theta}(\mathbf{x}_{t}, t, c)$ onto the sampled SDE state $\mathbf{x}_{t-1}$, as illustrated by the blue region of Fig.~\ref{fig:method_main}: trajectories with $\mathcal{A}^k > 0$ are imitated and those with $\mathcal{A}^k < 0$ are repelled. To understand how this regression changes the model, we write the sampled transition as $\mathbf{x}_{t-1}=\mu_{\theta_{\mathrm{old}}}(\mathbf{x}_t,t,c)+\eta\sqrt{\Delta t}\,\boldsymbol{\xi}_t$, where $\boldsymbol{\xi}_t\sim\mathcal{N}(\mathbf{0},\mathbf{I})$. At the start of an update, $\theta=\theta_{\mathrm{old}}$, so the sampled state differs from the predicted mean only by the injected noise. Substituting this expression into Eq.~\eqref{eq:grpo_trajectory_regression}, with $\rho_\theta^{k,t}=1$ and clipping inactive, gives
\begin{equation}
\nabla_\theta\mathcal{J}_{\mathrm{GRPO}}
=\frac{1}{\eta\sqrt{\Delta t}}\,
\mathbb{E}\!\left[\mathcal{A}^k\,\Big(\frac{\partial\mu_\theta(\mathbf{x}^k_t,t,c)}{\partial\theta}\Big)^{\!\top}\boldsymbol{\xi}^k_t\right].
\label{eq:reward_noise}
\end{equation}
The update depends on how reward advantages correlate with injected noise, as derived in App.~\ref{app:theoretic:cannyor}. Reward-favored perturbations may therefore alter the pretrained flow geometry.

\begin{wrapfigure}{r}{0.52\linewidth}
    \centering
    \vspace{-12pt}
    \includegraphics[width=\linewidth]{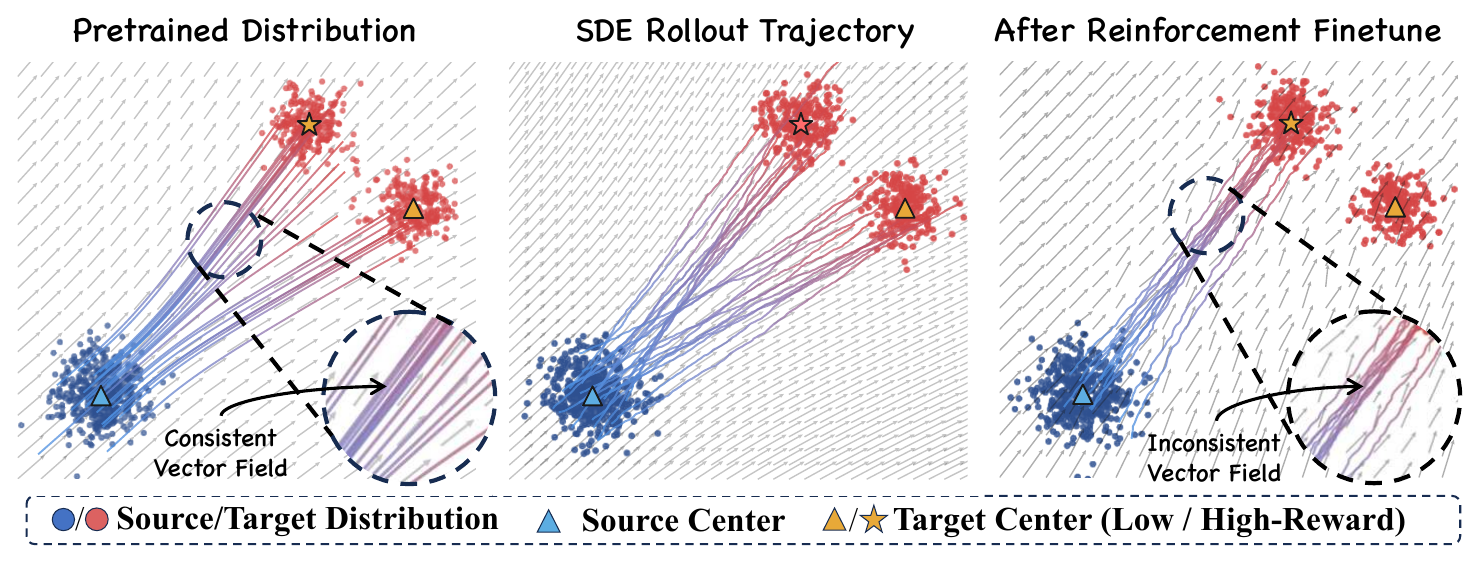}
    \caption{\textbf{Trajectory imitation toy example.} Pre-training establishes near-straight ODE paths, whereas SDE-based RFT regresses onto stochastic transitions, bending the paths and reducing velocity consistency.}
    \label{fig:motivation_flow}
    \vspace{-8pt}
\end{wrapfigure}
\noindent \textbf{Discussion on the Conflict with Flow Pre-training Geometry.}
Flow-matching pre-training learns a velocity field with near-straight ODE trajectories, supporting few-step sampling and single-step endpoint prediction. Corollary~\ref{thm:reinterpretation} expresses SDE-based RFT updates as advantage-weighted regression onto sampled stochastic transitions. These updates can bend the ODE trajectories and weaken cross-step prediction consistency, as illustrated by the toy example in Fig.~\ref{fig:motivation_flow}. This mismatch motivates a consistency constraint that limits drift from the pretrained flow geometry during reward optimization.

\section{Method: How to Reduce Hallucinations?}
\label{sec:method}

As illustrated in Fig.~\ref{fig:method_main}, DGPO combines Dynamic Granularity Rollout to broaden exploration with Consistent Policy Gradient Optimization to mitigate artifacts from trajectory imitation. Together, these components balance global semantics and local details while preserving cross-step consistency, following the training procedure in Algorithm~\ref{alg:DGPO_app} of App.~\ref{app:overall}.

\begin{figure}[t]
\centering
\includegraphics[width=0.98\linewidth]{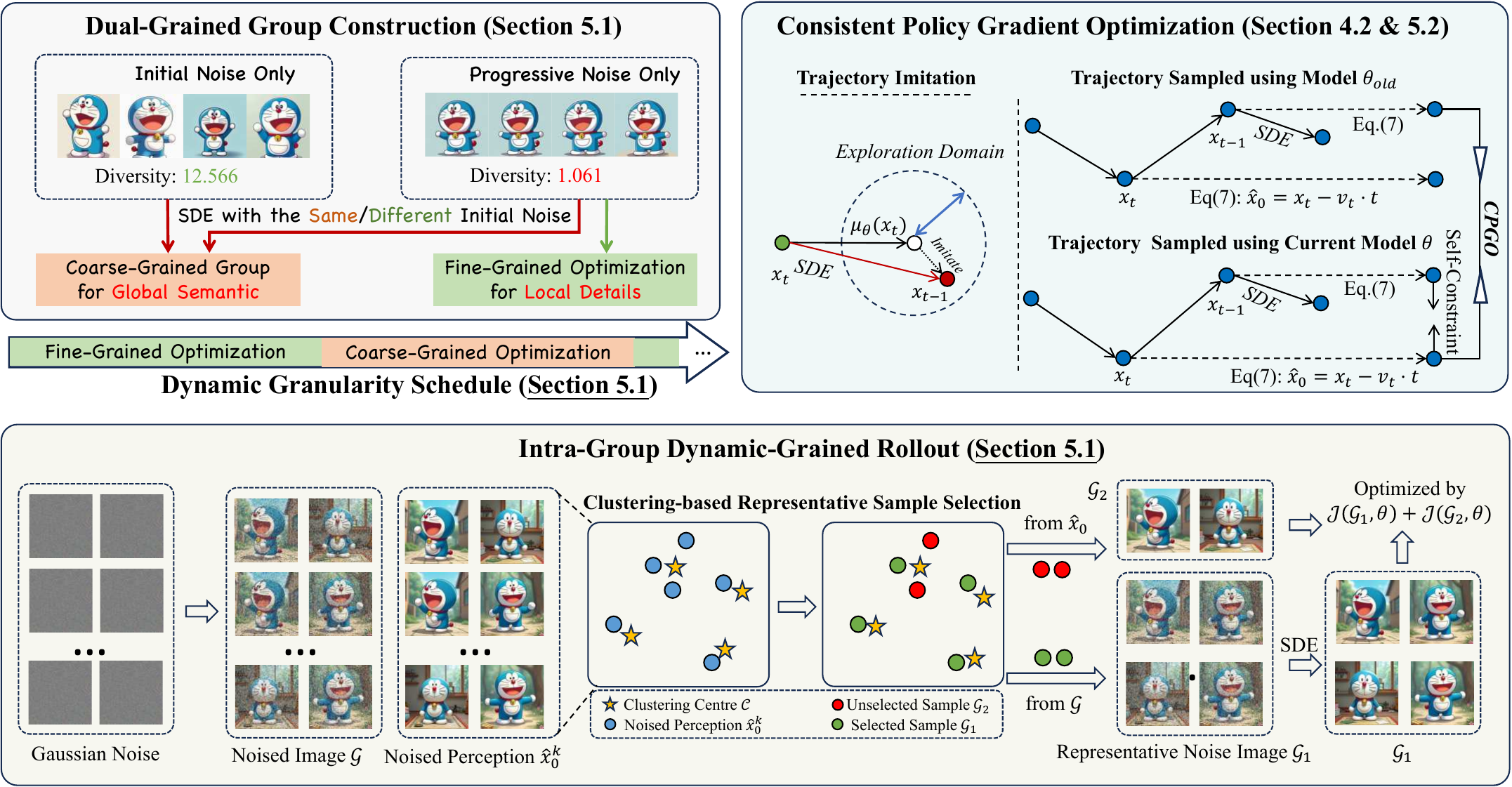}
\caption{\textbf{Overview of DGPO.} It consists of Dynamic Granularity Rollout and Consistent Policy Gradient Optimization.}
\label{fig:method_main}
\end{figure}

\subsection{Exploration: Dynamic Granularity Rollout}

A central question is how RFT can exploit both global structure and local details. As illustrated in Fig.~\ref{fig:motivation}, the initial noise largely determines global semantics, whereas progressive process noise mainly refines local details, consistent with prior findings~\cite{bai2024zigzag,Zhou_2025_ICCV,chen2024find}. Since RFT updates are tightly coupled with rollout, we introduce \emph{Dynamic Granularity Rollout (DGR)}, which separately controls initialization and process noise to steer optimization toward global- or local-focused updates.

\subsubsection{Inter-Group Dynamic-Granularity Rollout}
Building on this observation, we introduce Inter-Group DGR to balance global and local rollouts. Specifically, we dynamically construct groups that differ primarily in either global semantics or local details. Then, we alternately perform such fine/coarse-grained optimization between them during training, as shown in Fig.~\ref{fig:motivation}.

\noindent \textbf{Fine- and Coarse-Grained Group.} To expose the model to both semantic-level variations and detail-level refinements, we construct two complementary rollout groups by regulating initial noise versus progressive process noise, as illustrated in Fig.~\ref{fig:method_main}. Specifically, the initial noise sampling is defined as:
\begin{equation*}
\{\mathbf{x}_T^k\}_{k=1}^K
=\{\mathbf{x}_T^k\mid\mathbf{x}_T^k\stackrel{\text{i.i.d.}}{\sim}\mathcal{N}(\mathbf{0},\mathbf{I})\}\ \text{(Coarse-grained)},
\quad
\{\mathbf{x}_T^k\}_{k=1}^K
=\{\mathbf{x}_T^k\mid\mathbf{x}_T^k\leftarrow(\mathbf{x}_T\sim\mathcal{N}(\mathbf{0},\mathbf{I}))\}\ \text{(Fine-grained)}.
\end{equation*}
where $\text{i.i.d.}$ denotes independent and identically distributed sampling from the standard Gaussian prior $\mathcal{N}(\mathbf{0},\mathbf{I})$. In the coarse-grained group, independent noise vectors $\mathbf{x}_T^k$ induce wide coverage across distinct semantic modes of the data distribution, fostering global composition exploration. Conversely, in the fine-grained group, all $K$ candidates share the identical starting point $\mathbf{x}_T$ and branch out solely via subsequent stochastic transitions, isolating local structural refinements.

\noindent \textbf{Dynamic Granularity Schedule.}
By interleaving the fine- and coarse-grained rollout groups defined above, we dynamically modulate the optimization granularity across training iterations. Specifically, for a total budget of $N$ optimization iterations, we partition training into $N/j$ cyclical periods, where each period contains $j = j_g + j_c$ iterations. Within each cycle, the model executes $j_g$ fine-grained update steps followed by $j_c$ coarse-grained update steps. The coarse ratio $j_c / j \in [0, 1]$ serves as a global control parameter that balances macroscopic semantic exploration with microscopic texture refinement throughout the training trajectory. This scheduling mechanism is illustrated in Fig.~\ref{fig:method_main}.

\subsubsection{Intra-Group Dynamic-Grained Rollout}
\label{sec:intra_group_rollout}
Inter-Group DGR balances global semantic exploration and local detail refinement across training iterations. Within each group, however, candidates can share similar semantics, making full denoising of every candidate computationally redundant. We therefore introduce Intra-Group DGR to allocate further denoising to representative candidates while retaining feedback from the entire group. As illustrated in Fig.~\ref{fig:method_main}, we estimate coarse outputs at an intermediate step, cluster their visual features, and continue denoising the selected representatives. The remaining candidates retain their coarse predictions, and both subsets contribute to optimization with separately normalized rewards. The following components detail coarse progress perception, representative selection and refinement, and separate group optimization; Algorithm~\ref{alg:DGPO_app} in App.~\ref{app:overall} summarizes the procedure.

\noindent \textbf{Coarse Progress Perception.}
For each prompt, we initialize $K$ candidates according to the granularity schedule and sample from $T$ to $t_{\mathrm{s}}$ using the SDE in Eq.~\eqref{eq:sde}. Here $s_t\in(0,1]$ denotes the noise level (flow time) of sampling step $t$, with $s_T=1$ and $s_t$ decreasing as $t\to0$. We cache the sampled states and velocities for subsequent optimization. At $t_{\mathrm{s}}$, we estimate each candidate's clean latent with a single Euler prediction:
\begin{equation*}
\hat{\mathbf{x}}_{0}^k = \mathbf{x}_{t_{\mathrm{s}}}^k
- s_{t_{\mathrm{s}}}\mathbf{v}_{t_{\mathrm{s}}}^k,\qquad
\mathbf{v}_{t_{\mathrm{s}}}^k=v_\theta(\mathbf{x}_{t_{\mathrm{s}}}^k,t_{\mathrm{s}}).
\end{equation*}
This reuses the velocity computed during sampling and requires no additional flow-model evaluation. The coarse predictions provide early estimates of image semantics and quality, as quantified in Sec.~\ref{sec:rollout_analysis}.

\noindent \textbf{Cluster-Based Selection \& Fine-Grained Refinement.}
We decode the coarse latents into images, extract features with the reward model's visual encoder, and apply $k$-means with $K_1<K$ clusters. Within each cluster, we select the candidate nearest to its centroid in Euclidean feature distance. These $K_1$ representatives form $\mathcal{G}_1$, while the remaining candidates form $\mathcal{G}_2$, as shown in Fig.~\ref{fig:method_main}. We continue SDE sampling only for $\mathcal{G}_1$, starting from the cached states at $t_{\mathrm{s}}$ and recording the remaining states and velocities down to $t=0$. Thus, $\mathcal{G}_1$ contains complete trajectories, whereas $\mathcal{G}_2$ retains the early trajectory segments and coarse predictions. This procedure, detailed in App.~\ref{app:clustering_selection}, avoids fully denoising every candidate while preserving representatives of different semantic clusters.

\noindent \textbf{Separate Group Optimization.} We score the final images from $\mathcal{G}_1$ and the decoded coarse predictions from $\mathcal{G}_2$ using the same reward model. Following Eq.~\eqref{eq:grpo}, we normalize rewards separately within each group and compute a GRPO objective over its sampled transitions. The two objectives are then combined:
\begin{equation*}
\mathcal{J}_{\mathrm{GRPO}}(\theta,\mathcal{G})
=\mathcal{J}_{\mathrm{GRPO}}(\theta,\mathcal{G}_1)
+\mathcal{J}_{\mathrm{GRPO}}(\theta,\mathcal{G}_2).
\end{equation*}
Separate normalization prevents differences in reward scale between coarse and fully denoised images from distorting the relative advantages. Both groups provide feedback for the early sampling steps, while $\mathcal{G}_1$ additionally supervises the remaining fine-grained steps~\cite{XieDyMO}. Finally, we combine this reward objective with the CPGO constraint introduced below, reusing the cached states and velocities for optimization.

\subsection{Exploitation: Consistent Policy Gradient Optimization}
\label{sec:cpgo}

While Dynamic Granularity Rollout remedies the limited exploration domain and mitigates rollout computation, Corollary~\ref{thm:reinterpretation} shows that the policy gradient exploitation phase still regresses the velocity field onto reward-selected SDE realizations, which degrades cross-step consistency. To rectify this deficiency, we propose Consistent Policy Gradient Optimization (CPGO). CPGO adds an ODE-based one-step consistency regularizer during RL fine-tuning that limits drift away from the near-straight velocity field established during pre-training.

\noindent \textbf{Single-Step Consistency Proxy.}
For any continuous-time flow model parameterized by $\theta$, pre-training enables an estimate of the clean sample from any intermediate state in a single first-order Euler step:
\begin{equation}
\mathbf{x}_{t\rightarrow 0}^\theta = \mathbf{x}_t - s_t\cdot \mathbf{v}_\theta(\mathbf{x}_t, t),
\label{eq:single_step_pred}
\end{equation}
where $\mathbf{v}_\theta(\mathbf{x}_t, t)$ denotes the instantaneous velocity vector field predicted by the flow network at sampling step $t$, whose noise level is $s_t\in(0,1]$. As discussed in Sec.~\ref{sec:motivation}, SDE-based post-training can conflict with the ODE-oriented geometry established during pre-training, weakening cross-step prediction consistency. We therefore regularize these single-step predictions to preserve consistency during reward optimization.

CPGO augments policy-gradient reward optimization with a penalty on deviations between the current policy's single-step estimate $\mathbf{x}_{t\rightarrow 0}^{\theta}$ and two complementary reference anchors:
\begin{enumerate}
    \item \textbf{Teacher Policy Consistency}: Aligning with the prediction at the identical timestep $t$ produced by an earlier, more consistent teacher checkpoint $\theta_{\mathrm{old}}$, denoted as $\mathbf{x}_{t\rightarrow 0}^{\theta_{\mathrm{old}}} = \mathbf{x}_t - s_t\cdot \mathbf{v}_{\theta_{\mathrm{old}}}(\mathbf{x}_t, t)$. This limits policy drift and anchors predictions to the pre-trained generative prior.
    \item \textbf{Temporal Self-Consistency}: Aligning with the adjacent-step prediction produced by the current model along the trajectory at timestep $t-1$, denoted as $\mathbf{x}_{(t-1)\rightarrow 0}^{\theta} = \mathbf{x}_{t-1} - s_{t-1}\cdot \mathbf{v}_\theta(\mathbf{x}_{t-1}, t-1)$. This encourages local temporal smoothness of the endpoint prediction along the trajectory.
\end{enumerate}
Formally, this composite consistency objective can be expressed as:
\begin{equation}
\min_{\theta}\; \ell\!\left(\mathbf{x}_{t\rightarrow 0}^{\theta}, \mathbf{x}_{t\rightarrow 0}^{\theta_{\mathrm{old}}}\right)
+ \ell\!\left(\mathbf{x}_{t\rightarrow 0}^{\theta}, \mathbf{x}_{(t-1)\rightarrow 0}^{\theta}\right),
\label{eq:cpgo_dual}
\end{equation}
where $\ell(\mathbf{a}, \mathbf{b}) = \|\mathbf{a} - \mathbf{b}\|_2^2$ denotes the squared $\ell_2$ distance metric in latent space.

\noindent \textbf{A Unified Consistency Target.}
Although Eq.~\eqref{eq:cpgo_dual} promotes teacher alignment and temporal consistency, it requires balancing two separate objectives. We avoid this trade-off by approximately aligning the current prediction at step $t$ with the teacher's prediction at step $t-1$:
\begin{equation}
\min_{\theta}\; \ell\!\left(\mathbf{x}_{t\rightarrow 0}^{\theta}, \mathbf{x}_{(t-1)\rightarrow 0}^{\theta_{\mathrm{old}}}\right).
\label{eq:cpgo_unified}
\end{equation}
This shifted target uses the teacher's adjacent-step prediction to encourage teacher alignment and cross-step consistency within a single objective.

\noindent \textbf{Full DGPO Optimization Objective.}
Finally, we integrate the unified CPGO consistency regularization into the group relative policy optimization framework:
\begin{equation}
\mathcal{J}_{\mathrm{DGPO}}(\theta, \mathcal{G})
= \mathcal{J}_{\mathrm{GRPO}}(\theta, \mathcal{G})
- \omega \cdot \mathbb{E}_{c\sim\mathcal{D},\; \mathbf{x}\in\mathcal{G},\; t\sim\mathcal{U}(\{t:\,s_t\le\tau\})}
  \left[\big\| \mathbf{x}_{t\rightarrow 0}^{\theta} - \mathbf{x}_{(t-1)\rightarrow 0}^{\theta_{\mathrm{old}}} \big\|_2^2\right],
\label{eq:dgpo_final}
\end{equation}
where $\omega > 0$ is a scalar hyperparameter balancing reward exploitation against consistency preservation, and $\tau$ is the noise-level threshold. CPGO is applied only to steps with $s_t\le\tau$, excluding the highly noisy initial stage $s_t>\tau$, where single-step predictions are less reliable. CPGO constructs its consistency targets from states and velocities cached during rollout, avoiding additional trajectory sampling.

By regularizing policy gradients with the ODE-based cross-step consistency term in Eq.~\eqref{eq:dgpo_final}, DGPO limits how far reward-selected SDE transitions pull the velocity field from its pre-trained geometry.

\section{Visual Hallucination Evaluator}
\label{sec:vh_evaluator}

Existing benchmarks evaluate human preference~\cite{ma2025hpsv3}, personalized generation~\cite{peng2024dreambench++, liu2025f}, and compositional fidelity~\cite{huang2025t2i, ghosh2023geneval}, but do not directly measure the over-optimization artifacts introduced by RFT. In particular, higher preference scores can coexist with excessive sharpening, irrelevant details, and grid patterns. To this end, we introduce VH-Evaluator, which combines low-level image statistics with high-level MLLM assessment. We also report latent trajectory consistency to examine changes in the flow geometry.

\noindent \textbf{Low-Level Evaluation.} We measure four image properties: sharpness, high-frequency content, ringing, and noise. Specifically, we use Laplacian variance~\cite{Bansal2016} to measure edge sharpness, high-pass filter responses~\cite{abdel2003analysis} to measure high-frequency energy, intensity variation around detected edges to characterize ringing, and residuals in smooth regions to estimate noise. These metrics capture changes in local image structure, but cannot determine whether a detail is natural or relevant to the prompt. For trajectory consistency, we compute the mean squared deviation of intermediate latents from the straight line connecting initial noise and the generated latent. Lower values indicate straighter trajectories. Formulas and implementation details are provided in App.~\ref{app:vh_evaluator}.

\begin{tcolorbox}[
  enhanced,
  breakable,
  colback=blue!2!white,
  colframe=blue!30!black,
  colbacktitle=blue!7!white,
  coltitle=black,
  coltext=black,
  title={VH-Evaluator Prompt\hfill{\normalfont\footnotesize Key excerpts}},
  fonttitle=\small\bfseries,
  fontupper=\small,
  boxrule=0.5pt,
  arc=1mm,
  left=3mm,
  right=3mm,
  top=2mm,
  bottom=2mm,
  before skip=8pt,
  after skip=8pt
]
\textbf{Task.} Assess over-optimization hallucinations in the image for ``\texttt{<prompt>}''.

\smallskip
\textbf{Scoring.} Score each axis from \textbf{0 for no artifact to 5 for extreme artifacts}.
\begin{itemize}
  \setlength{\itemsep}{2pt}
  \setlength{\parskip}{0pt}
  \setlength{\parsep}{0pt}
  \item \texttt{detail\_sharpness\_score}: harsh edges, unnatural sharpening, and high-frequency noise.
  \item \texttt{irrelevant\_details\_score}: prompt-unrelated details, over-decoration, and clutter.
  \item \texttt{grid\_pattern\_score}: artificial grid patterns in the generated image.
\end{itemize}
\smallskip
\textbf{Output.} Return valid JSON with the three scores, their mean, confidence, and a brief diagnostic explanation.
\end{tcolorbox}

\noindent \textbf{High-Level Evaluation.} To distinguish over-optimization from reasonable image detail, we use Qwen2.5-VL-72B~\cite{bai2025qwen25vl} to evaluate each image together with its text prompt. We score unnatural sharpening, irrelevant details, and grid artifacts. These include harsh edges and high-frequency noise, prompt-unrelated elements and excessive decoration, and repetitive artificial textures, respectively. Each score ranges from 0 for no artifact to 5 for extreme artifacts, and we report their mean as the overall hallucination score. It returns the three scores, their mean, confidence in $[0,1]$, and a brief explanation in JSON format. The box above summarizes the key instructions. The full prompt and reference examples are provided in App.~\ref{app:vh_high_level}.

\FloatBarrier
\begin{table}[!htb]
\centering
\caption{Comparison with SoTA RFT methods. \textcolor{gray!50}{Gray text} marks in-domain training metrics, and gray rows are ours. Gains are in \textcolor{blue}{blue} and drops in \textcolor{red}{red}; best results within each block are in \textbf{bold}. Relative changes are against Flux Dev. Avg.\ is the mean of the OoD metrics, excluding in-domain training rewards: HPS throughout, and also CLIP in the CLIP+HPS block. HPS: Human Preference Score v2.1; IR: ImageReward; PS: PickScore; Aes.: Aesthetic Predictor v2.5; CLIP: CLIP Score; UR-S: Unified Reward-Semantic; UR: Unified Reward.}
\label{tab:main_results}
\small
\setlength{\tabcolsep}{1.4pt}
\renewcommand{\arraystretch}{1.08}
\resizebox{\linewidth}{!}{%
\begin{tabular}{@{}lccccccccc@{}}
\toprule
\multirow{2}{*}{Method} & \multirow{2}{*}{Reward} & \multicolumn{3}{c}{Human Preference} & \multicolumn{1}{c}{Aesthetic} & \multicolumn{2}{c}{Semantic} & \multicolumn{2}{c}{Comprehensive}  \\
\cmidrule(lr){3-5} \cmidrule(lr){6-6} \cmidrule(lr){7-8} \cmidrule(lr){9-10}
 & & HPS & IR & PS & Aes. & CLIP & UR-S & UR & Avg. \\
\midrule
Flux Dev & - & \textcolor{gray!50}{0.312} & 1.089 & 0.226 & 5.837 & 0.388 & 3.365 & 3.520 & 2.404 \\
\midrule
\multicolumn{10}{c}{\textit{DDPO-Based Methods}} \\
\cmidrule(lr){1-10}
DDPO & \textit{HPS} & \textcolor{gray!50}{0.313}$_{\textcolor{blue!60}{+0.3\%}}$ & 1.129$_{\textcolor{blue}{+3.7\%}}$ & 0.227$_{\textcolor{blue}{+0.4\%}}$ & 5.896$_{\textcolor{blue}{+1.0\%}}$ & 0.391$_{\textcolor{blue}{+0.8\%}}$ & 3.387$_{\textcolor{blue}{+0.7\%}}$ & 3.568$_{\textcolor{blue}{+1.4\%}}$ & 2.433$_{\textcolor{blue}{+1.2\%}}$ \\
\rowcolor{gray!20}+ DGPO & \textit{HPS} & \textcolor{gray!50}{\textbf{0.324}}$_{\textcolor{blue!60}{+3.8\%}}$ & \textbf{1.197}$_{\textcolor{blue}{+9.9\%}}$ & \textbf{0.228}$_{\textcolor{blue}{+0.9\%}}$ & \textbf{5.998}$_{\textcolor{blue}{+2.8\%}}$ & \textbf{0.396}$_{\textcolor{blue}{+2.1\%}}$ & \textbf{3.476}$_{\textcolor{blue}{+3.3\%}}$ & \textbf{3.593}$_{\textcolor{blue}{+2.1\%}}$ & \textbf{2.481}$_{\textcolor{blue}{+3.2\%}}$ \\
\midrule
\multicolumn{10}{c}{\textit{DPO-Based Methods}} \\
\cmidrule(lr){1-10}
Diffusion-DPO & \textit{HPS} & \textcolor{gray!50}{0.318}$_{\textcolor{blue!60}{+1.9\%}}$ & 1.177$_{\textcolor{blue}{+8.1\%}}$ & 0.230$_{\textcolor{blue}{+1.8\%}}$ & 5.967$_{\textcolor{blue}{+2.2\%}}$ & 0.393$_{\textcolor{blue}{+1.3\%}}$ & 3.424$_{\textcolor{blue}{+1.8\%}}$ & 3.528$_{\textcolor{blue}{+0.2\%}}$ & 2.453$_{\textcolor{blue}{+2.0\%}}$ \\
D3PO & \textit{HPS} & \textcolor{gray!50}{\textbf{0.338}}$_{\textcolor{blue!60}{+8.3\%}}$ & 1.185$_{\textcolor{blue}{+8.8\%}}$ & 0.227$_{\textcolor{blue}{+0.4\%}}$ & 5.849$_{\textcolor{blue}{+0.2\%}}$ & 0.379$_{\textcolor{red}{-2.3\%}}$ & 3.396$_{\textcolor{blue}{+0.9\%}}$ & 3.507$_{\textcolor{red}{-0.4\%}}$ & 2.424$_{\textcolor{blue}{+0.8\%}}$ \\
\rowcolor{gray!20}+ DGPO & \textit{HPS} & \textcolor{gray!50}{0.334}$_{\textcolor{blue!60}{+7.1\%}}$ & \textbf{1.249}$_{\textcolor{blue}{+14.7\%}}$ & \textbf{0.232}$_{\textcolor{blue}{+2.7\%}}$ & \textbf{6.077}$_{\textcolor{blue}{+4.1\%}}$ & \textbf{0.394}$_{\textcolor{blue}{+1.5\%}}$ & \textbf{3.499}$_{\textcolor{blue}{+4.0\%}}$ & \textbf{3.570}$_{\textcolor{blue}{+1.4\%}}$ & \textbf{2.504}$_{\textcolor{blue}{+4.1\%}}$ \\
\midrule
\multicolumn{10}{c}{\textit{GRPO-Based Methods}} \\
\cmidrule(lr){1-10}
DanceGRPO & \textit{CLIP+HPS} & \textcolor{gray!50}{\textbf{0.336}}$_{\textcolor{blue!60}{+7.7\%}}$ & 1.124$_{\textcolor{blue}{+3.2\%}}$ & 0.229$_{\textcolor{blue}{+1.3\%}}$ & 5.746$_{\textcolor{red}{-1.6\%}}$ & \textcolor{gray!50}{\textbf{0.407}}$_{\textcolor{blue!60}{+4.9\%}}$ & 3.333$_{\textcolor{red}{-1.0\%}}$ & 3.491$_{\textcolor{red}{-0.8\%}}$ & 2.785$_{\textcolor{red}{-0.8\%}}$ \\
\rowcolor{gray!20}+ DGPO & \textit{CLIP+HPS} & \textcolor{gray!50}{0.329}$_{\textcolor{blue!60}{+5.4\%}}$ & \textbf{1.349}$_{\textcolor{blue}{+23.9\%}}$ & \textbf{0.235}$_{\textcolor{blue}{+4.0\%}}$ & \textbf{6.057}$_{\textcolor{blue}{+3.8\%}}$ & \textcolor{gray!50}{0.401}$_{\textcolor{blue!60}{+3.4\%}}$ & \textbf{3.493}$_{\textcolor{blue}{+3.8\%}}$ & \textbf{3.612}$_{\textcolor{blue}{+2.6\%}}$ & \textbf{2.949}$_{\textcolor{blue}{+5.1\%}}$ \\
\hdashline[1pt/2pt]
\noalign{\vskip 1pt}
MixGRPO & \textit{HPS} & \textcolor{gray!50}{\textbf{0.361}}$_{\textcolor{blue!60}{+15.7\%}}$ & 1.201$_{\textcolor{blue}{+10.3\%}}$ & 0.222$_{\textcolor{red}{-1.8\%}}$ & 5.833$_{\textcolor{red}{-0.1\%}}$ & 0.346$_{\textcolor{red}{-10.8\%}}$ & 3.293$_{\textcolor{red}{-2.1\%}}$ & 3.401$_{\textcolor{red}{-3.4\%}}$ & 2.383$_{\textcolor{red}{-0.9\%}}$ \\
\rowcolor{gray!20}+ DGPO & \textit{HPS} & \textcolor{gray!50}{0.354}$_{\textcolor{blue!60}{+13.5\%}}$ & \textbf{1.323}$_{\textcolor{blue}{+21.5\%}}$ & \textbf{0.228}$_{\textcolor{blue}{+0.9\%}}$ & \textbf{6.052}$_{\textcolor{blue}{+3.7\%}}$ & \textbf{0.374}$_{\textcolor{red}{-3.6\%}}$ & \textbf{3.328}$_{\textcolor{red}{-1.1\%}}$ & \textbf{3.432}$_{\textcolor{red}{-2.5\%}}$ & \textbf{2.456}$_{\textcolor{blue}{+2.2\%}}$ \\
\hdashline[1pt/2pt]
\noalign{\vskip 1pt}
FlowGRPO & \textit{HPS} & \textcolor{gray!50}{0.326}$_{\textcolor{blue!60}{+4.5\%}}$ & 1.135$_{\textcolor{blue}{+4.2\%}}$ & 0.226$_{\textcolor{gray}{-}}$ & 5.926$_{\textcolor{blue}{+1.5\%}}$ & 0.375$_{\textcolor{red}{-3.4\%}}$ & 3.320$_{\textcolor{red}{-1.3\%}}$ & 3.476$_{\textcolor{red}{-1.3\%}}$ & 2.410$_{\textcolor{blue}{+0.2\%}}$ \\
DanceGRPO & \textit{HPS} & \textcolor{gray!50}{\textbf{0.353}}$_{\textcolor{blue!60}{+13.1\%}}$ & 1.155$_{\textcolor{blue}{+6.1\%}}$ & 0.226$_{\textcolor{gray}{-}}$ & 5.897$_{\textcolor{blue}{+1.0\%}}$ & 0.361$_{\textcolor{red}{-7.0\%}}$ & 3.300$_{\textcolor{red}{-1.9\%}}$ & 3.379$_{\textcolor{red}{-4.0\%}}$ & 2.386$_{\textcolor{red}{-0.7\%}}$ \\
PrefGRPO & \textit{HPS} & \textcolor{gray!50}{0.346}$_{\textcolor{blue!60}{+10.9\%}}$ & 1.172$_{\textcolor{blue}{+7.6\%}}$ & 0.227$_{\textcolor{blue}{+0.4\%}}$ & 5.953$_{\textcolor{blue}{+2.0\%}}$ & 0.378$_{\textcolor{red}{-2.6\%}}$ & 3.338$_{\textcolor{red}{-0.8\%}}$ & 3.486$_{\textcolor{red}{-1.0\%}}$ & 2.426$_{\textcolor{blue}{+0.9\%}}$ \\
\rowcolor{gray!20}DanceGRPO w/ DGPO & \textit{HPS} & \textcolor{gray!50}{0.348}$_{\textcolor{blue!60}{+11.5\%}}$ & \textbf{1.295}$_{\textcolor{blue}{+18.9\%}}$ & \textbf{0.230}$_{\textcolor{blue}{+1.8\%}}$ & \textbf{6.197}$_{\textcolor{blue}{+6.2\%}}$ & \textbf{0.384}$_{\textcolor{red}{-1.0\%}}$ & \textbf{3.408}$_{\textcolor{blue}{+1.3\%}}$ & \textbf{3.622}$_{\textcolor{blue}{+2.9\%}}$ & \textbf{2.523}$_{\textcolor{blue}{+4.9\%}}$ \\
\bottomrule
\end{tabular}}%
\end{table}

\begingroup
\setlength{\textfloatsep}{10pt plus 2pt minus 2pt}
\setlength{\floatsep}{8pt plus 2pt minus 2pt}
\setlength{\intextsep}{8pt plus 2pt minus 2pt}
\captionsetup{skip=4pt}

\section{Experiments}
\label{sec:experiments}

\subsection{Experimental Setup}

\noindent \textbf{Datasets \& Models.} We use HPDv2~\cite{wu2023human} prompts for online training and evaluation under the HPS-v2 benchmark setup, reserving the last 400 prompts from the 103.7k training set for validation. Our base text-to-image model is FLUX.1~dev~\cite{flux2024}. For the main-text comparisons, we use HPS-v2.1~\cite{wu2023human} as the training reward, including for FlowGRPO~\cite{liu2025flow}, except for the explicitly labeled CLIP+HPS setting in Tab.~\ref{tab:main_results}. The additional FlowGRPO setting described in App.~\ref{app:add_exp_settings} uses PickScore~\cite{kirstain2023pick}. For offline preference optimization, we generate 20k image pairs with FLUX.1~dev and assign preference labels using HPS-v2.1 scores.

\noindent \textbf{Metrics.} We report human preference via ImageReward~\cite{xu2023imagereward} and PickScore~\cite{kirstain2023pick}, aesthetic quality~\cite{discus0434_aesthetic_2024}, and semantic alignment via CLIP and Unified Reward~\cite{wang2025unified}, emphasizing OoD metrics over the in-domain training reward. For visual hallucinations, we follow the evaluation protocol in Sec.~\ref{sec:vh_evaluator}.

\noindent \textbf{Implementation.} Our method is compatible with standard policy-gradient baselines. We follow official hyperparameters for DanceGRPO~\cite{xue2025dancegrpo}, MixGRPO~\cite{li2025mixgrpo}, FlowGRPO~\cite{liu2025flow}, and PrefGRPO~\cite{wang2025pref}. Diffusion-DPO~\cite{wallace2024diffusion}, D3PO~\cite{yang2024using}, and DDPO~\cite{black2023training} are reproduced via the DanceGRPO codebase. We use 16 sampling steps, group size 12, Inter-Group DGR period 40 with coarse ratio 0.25, Intra-Group DGR at step 12 with dual groups of size 6, CPGO weight $10^{-6}$, and $\tau{=}0.6$. See App.~\ref{app:add_exp_settings}. We evaluate rollout diversity and coarse prediction reliability using 50 HPDv2 prompts. FLUX.1~dev generates images with 16 steps, guidance 3.5, and resolution $720\times720$. Quantitative results are averaged over 50 independent repetitions per prompt.

\vspace{-4pt}
\subsection{Main Results}

\noindent \textbf{Comparison with SoTA RFT Methods.}
Tab.~\ref{tab:main_results} evaluates DGPO without Intra-Group DGR. Existing methods often sacrifice OoD performance for in-domain gains. For example, MixGRPO improves HPSv2.1 by 15.7\% but degrades CLIP by 10.8\%. In contrast, DGPO improves ImageReward by 18.9\% and Aesthetic by 6.2\%, balancing in-domain and OoD performance.

\setlength{\intextsep}{1pt}
\begin{wraptable}{r}{0.54\linewidth}
\centering
\setlength{\parskip}{0pt}
\captionsetup{position=top,aboveskip=0pt,belowskip=0pt}
\caption{\textbf{Evaluation on visual hallucination.} Gray rows are ours. LV: Laplacian Variance; HF: High-Frequency Energy; Nse.: Noise; Sh.: Sharpness; Irr.: Irrelevant Details; Gr.: Grid; LC: Latent Consistency.}
\label{tab:second_results}
\resizebox{\linewidth}{!}{%
\small
\setlength{\tabcolsep}{2.4pt}
\renewcommand{\arraystretch}{0.92}
\begin{tabular}{@{}lccccccccc@{}}
\toprule
\multirow{2}{*}{Method} &
\multicolumn{4}{c}{Low $\downarrow$} &
\multicolumn{4}{c}{High $\downarrow$} &
Cons. $\downarrow$ \\
\cmidrule(lr){2-5}\cmidrule(lr){6-9}
& LV & HF & Edge & Nse. & Sh. & Irr. & Gr. & Avg. & LC \\
\midrule
Flux Dev & 977 & 26 & 61 & 0.27 & 1.00 & 0.87 & 0.12 & 0.66 & 0.30 \\
\midrule
DDPO & 3895 & 68 & 71 & 0.78 & 2.48 & 2.52 & 0.61 & 1.87 & 0.32 \\
\rowcolor{gray!20}w/ DGPO & \textbf{1028} & \textbf{40} & \textbf{66} & \textbf{0.62} & \textbf{1.07} & \textbf{0.92} & \textbf{0.12} & \textbf{0.70} & \textbf{0.30} \\
\midrule
OnlineDPO & 2752 & 44 & 62 & 0.61 & 2.43 & 2.20 & 0.58 & 1.74 & 0.32 \\
\rowcolor{gray!20}w/ DGPO & \textbf{1131} & \textbf{37} & \textbf{59} & \textbf{0.59} & \textbf{1.50} & \textbf{1.34} & \textbf{0.29} & \textbf{1.04} & \textbf{0.30} \\
\midrule
MixGRPO & 3900 & 82 & 66 & 2.65 & 3.12 & 2.73 & 0.79 & 2.21 & 0.35 \\
\rowcolor{gray!20}w/ DGPO & \textbf{1309} & \textbf{45} & \textbf{53} & \textbf{1.17} & \textbf{2.43} & \textbf{2.08} & \textbf{0.54} & \textbf{1.68} & \textbf{0.31} \\
\midrule
DanceGRPO & 5348 & 91 & 68 & 2.10 & 2.76 & 2.69 & 0.68 & 2.04 & 0.33 \\
w/ LoRA & 6707 & 99 & 69 & 2.00 & 2.71 & 2.43 & 0.66 & 1.93 & 0.33 \\
w/ KL & 3699 & 79 & 68 & 2.22 & 2.37 & 2.03 & 0.55 & 1.65 & 0.31 \\
w/ Early Stop & 3604 & 77 & 65 & 2.29 & 3.16 & 3.17 & 0.83 & 2.39 & 0.32 \\
w/ PrefGRPO & 2634 & 71 & 62 & 2.24 & 2.59 & 2.20 & 0.73 & 1.84 & 0.33 \\
\rowcolor{gray!20}w/ DGPO & \textbf{1421} & \textbf{45} & \textbf{57} & \textbf{0.93} & \textbf{1.81} & \textbf{1.59} & \textbf{0.37} & \textbf{1.26} & \textbf{0.30} \\
\bottomrule
\end{tabular}%
}
\end{wraptable}
\setlength{\intextsep}{8pt plus 2pt minus 2pt}

\noindent \textbf{Comparison with SoTA Pretrained Model.} In Tab.~\ref{tab:sota_results}, our method reaches an average of 2.523, surpassing efficient models with inference under 100s and approaching Qwen-Image's 2.576. It is also $9\times$ faster, requiring 14s rather than 128s. These results show that CPGO preserves consistency and quality.

\noindent \textbf{Evaluation on Vision Hallucination.} Tab.~\ref{tab:second_results} shows that RFT methods often introduce visual hallucinations, while DGPO mitigates them. DanceGRPO~\cite{xue2025dancegrpo} and MixGRPO~\cite{li2025mixgrpo} raise artifact scores from 0.66 to 2.04 and 2.21. DGPO reduces them by 38\% to 1.26 and by 24\% to 1.68. The user study in Fig.~\ref{fig:user_study} further supports this observation.

\noindent \textbf{Comparison with Methods for Reward Hacking.} Reward hacking is closely related to visual hallucination. In Tab.~\ref{tab:second_results}, LoRA reduces artifacts by only 5.4\%, Early Stopping increases them by 17\%, KL divergence cuts 19\% but suppresses preference scores, and PrefGRPO achieves 9.8\%. Our method achieves a substantially larger reduction of 38\%.

\begin{table}[!htb]
\centering
\caption{\textbf{Comparison with SoTA pretrained models.} Best results are in \textbf{bold}, second-best results are \underline{underlined}, and gray rows are ours. Avg.\ is the mean of the six metrics excluding HPS.}
\label{tab:sota_results}
\small
\resizebox{\linewidth}{!}{%
\begin{tabular}{lcccccccccc}
\toprule
\multirow{2}{*}{Method} & \multirow{2}{*}{Reward} & \multicolumn{3}{c}{Human Preference} & \multicolumn{1}{c}{Aesthetic} & \multicolumn{2}{c}{Semantic} & \multicolumn{2}{c}{Comprehensive} & \multirow{2}{*}{Time / s} \\
\cmidrule(lr){3-5} \cmidrule(lr){6-6} \cmidrule(lr){7-8} \cmidrule(lr){9-10}
 & & {HPS} & IR & PS & Aes. & CLIP & UR-S & UR & Avg. & \\
\cmidrule(lr){1-11}
SDXL~\cite{podell2023sdxl} & - & 0.292 & 0.945 & 0.226 & 5.735 & \underline{0.418} & 3.275 & 3.425 & 2.337 & \textbf{9} \\
Hunyuan-DiT~\cite{li2024hunyuandit} & - & 0.302 & 1.094 & 0.225 & 5.616 & 0.411 & 3.296 & 3.458 & 2.350 & 23 \\
Kolors~\cite{kolors} & - & 0.312 & 0.974 & 0.225 & 6.000 & 0.386 & 3.296 & 3.376 & 2.376 & 31 \\
SD-3.5-M~\cite{esser2024scaling} & - & 0.302 & 1.130 & 0.227 & 5.659 & 0.411 & 3.350 & 3.580 & 2.393 & \underline{14} \\
SD-3.5-L~\cite{esser2024scaling} & - & 0.303 & 1.143 & 0.228 & 5.853 & 0.410 & 3.330 & 3.627 & 2.432 & 17 \\
HiDream-Full~\cite{hidreami1technicalreport} & - & 0.325 & \underline{1.385} & \underline{0.231} & 5.742 & 0.412 & 3.407 & \underline{3.714} & 2.482 & 75 \\
Qwen-Image~\cite{wu2025qwenimagetechnicalreport} & - & 0.324 & \textbf{1.427} & \textbf{0.233} & 6.048 & \textbf{0.421} & \textbf{3.436} & \textbf{3.891} & \textbf{2.576} & 128 \\
\midrule
Flux Dev~\cite{flux2024} & - & 0.312 & 1.089 & 0.226 & 5.837 & 0.388 & 3.365 & 3.520 & 2.404 & 36 \\
\rowcolor{gray!20}w/ Ours & \textit{HPS} & \textbf{0.348} & 1.295 & 0.230 & \underline{6.197} & 0.384 & \underline{3.408} & {3.622} & 2.523 & 36 \\
\rowcolor{gray!20}Ours w/ 20-step sampling & \textit{HPS} & \underline{0.346} & 1.311 & \underline{0.231} & \textbf{6.229} & 0.384 & {3.407}& {3.596} & \underline{2.526} & \underline{14} \\
\bottomrule
\end{tabular}%
}
\end{table}

\subsection{Analysis of Rollout Granularity and Coarse Perception}
\label{sec:rollout_analysis}

\noindent \textbf{Rollout Granularity.} As shown in Fig.~\ref{supp:fig:coarse_fine_group}, independent initial noise varies global layouts, whereas shared initial noise mainly yields local appearance changes. This supports combining coarse semantic exploration with fine detail refinement, consistent with the diversity and reward-contrast analysis in App.~\ref{app:initial_noise_analysis}.

\begin{figure}[!t]
  \centering
  \includegraphics[width=\textwidth]{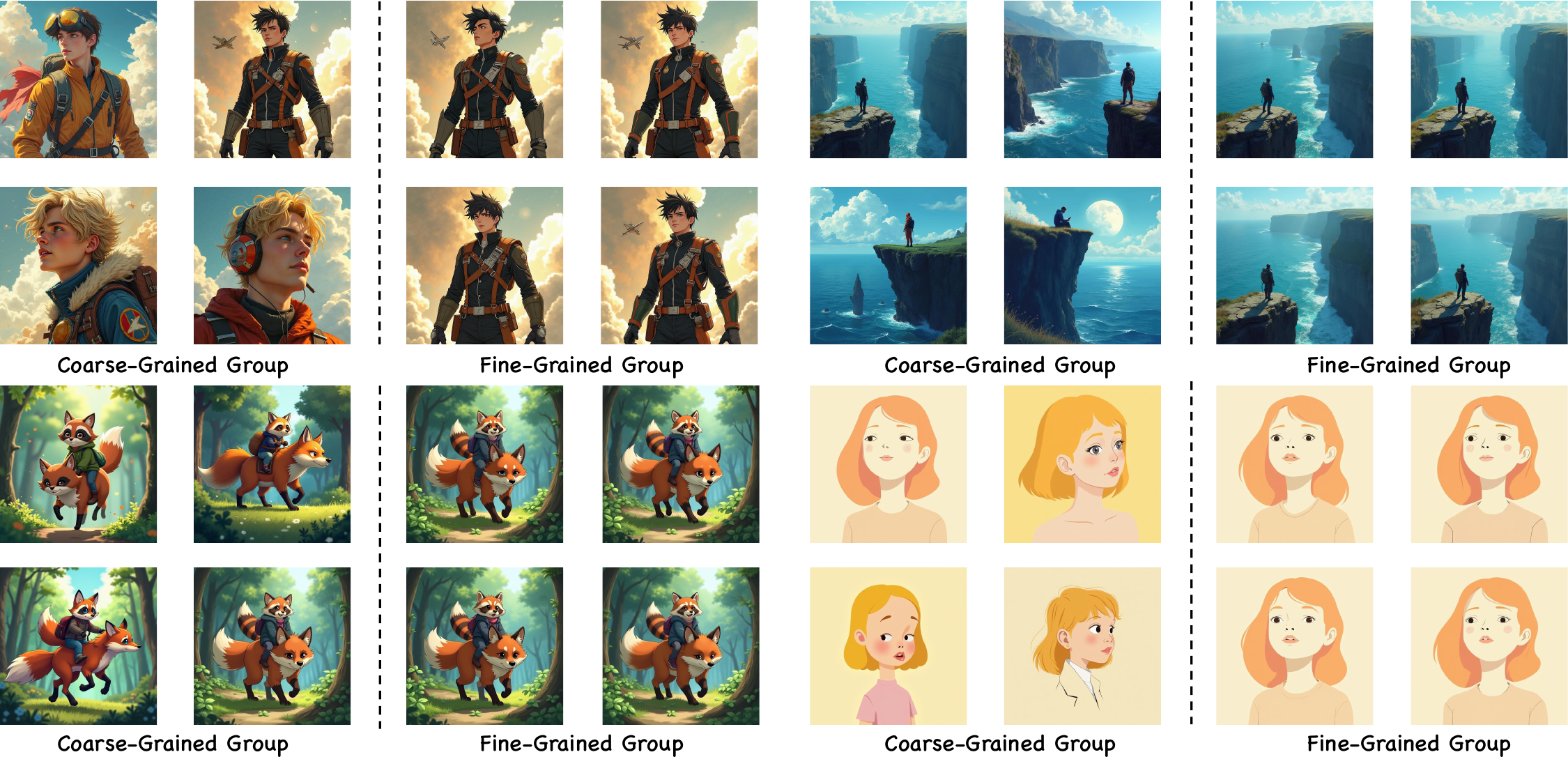}
  \caption{\textbf{Complementary rollout granularities.} Independent initial noise yields greater global semantic variation in coarse-grained groups, whereas shared initial noise isolates local appearance variations in fine-grained groups.}
  \label{supp:fig:coarse_fine_group}
\end{figure}

\noindent \textbf{Coarse Prediction Reliability.} Fig.~\ref{supp:app:Coarse_Progress_Perception} reports ImageReward ranking correlation and CLIP feature similarity against fully denoised images. Coarse predictions outperform noisy intermediate images on both metrics, supporting early candidate selection before full denoising. We use $t_{\mathrm{s}}=12$ for accuracy and $t_{\mathrm{s}}=8$ for speed, following the analysis in App.~\ref{app:progressive_perception}.

\subsection{Ablation Study}

\noindent \textbf{Granularity Trade-off.} As shown in Fig.~\ref{fig:intro}, fine-grained optimization favors local details but hurts global semantics, while coarse-grained does the opposite, revealing a clear trade-off in Tab.~\ref{tab:ablation}. Specifically, fine-grained yields a high in-domain HPS of 0.353 but a low CLIP of 0.361, whereas coarse-grained flips this with an HPS of 0.326 and a CLIP of 0.375. DGR resolves this trade-off, attaining a competitive HPS of 0.349 and a CLIP of 0.376 by balancing local details and global semantics.

\mbox{}\par
\vspace{-\baselineskip}
\vspace{-\parskip}
\setlength{\intextsep}{1pt}
\begin{wraptable}{r}{0.68\linewidth}
\centering
\setlength{\parskip}{0pt}
\captionsetup{position=top,aboveskip=0pt,belowskip=0pt}
\caption{Ablation study of DGPO components. \textcolor{gray!50}{Gray text} marks in-domain HPS, and gray rows are the full DGPO variants. Avg.\ is the mean of OoD metrics excluding HPS. IR: ImageReward; PS: PickScore; Aes.: Aesthetic Predictor v2.5; UR-S: Unified Reward-Semantic; UR: Unified Reward.}
\label{tab:ablation}
\resizebox{\linewidth}{!}{%
\footnotesize
\setlength{\tabcolsep}{3.4pt}
\renewcommand{\arraystretch}{0.96}
\begin{tabular}{@{}lllcccccccc@{}}
\toprule
Method & Granularity & Time / s & HPS & IR & PS & Aes. & CLIP & UR-S & UR & Avg. \\
\midrule
Flux Dev & - & - & \textcolor{gray!50}{0.312} & 1.09 & 0.226 & 5.84 & 0.388 & 3.37 & 3.52 & 2.41 \\
\midrule
DanceGRPO & \textit{Fine} & 665 & \textcolor{gray!50}{\textbf{0.353}} & 1.16 & 0.226 & 5.90 & 0.361 & 3.30 & 3.38 & 2.39 \\
- & \textit{Coarse} & 665 & \textcolor{gray!50}{0.326} & 1.14 & 0.226 & 5.93 & 0.375 & 3.32 & 3.48 & 2.41 \\
+ Inter DGR & \textit{Dynamic} & 665 & \textcolor{gray!50}{0.349} & 1.19 & 0.230 & 6.14 & 0.376 & 3.35 & 3.50 & 2.46 \\
\rowcolor{gray!20}+ CPGO & \textit{Dynamic} & 665 & \textcolor{gray!50}{0.348} & 1.30 & 0.230 & 6.20 & \textbf{0.384} & 3.41 & 3.62 & 2.52 \\
+ Intra DGR & \textit{Dynamic} & \textbf{541} & \textcolor{gray!50}{0.348} & 1.28 & 0.231 & 6.16 & 0.373 & 3.39 & 3.53 & 2.49 \\
\rowcolor{gray!20}+ Intra DGR & \textit{Dynamic} & \textbf{658} & \textcolor{gray!50}{\textbf{0.353}} & \textbf{1.32} & \textbf{0.232} & \textbf{6.22} & 0.382 & \textbf{3.42} & \textbf{3.64} & \textbf{2.54} \\
\hdashline[1pt/2pt]
\noalign{\vskip 1pt}
OnlineDPO & \textit{Coarse} & 373 & \textcolor{gray!50}{0.338} & 1.19 & 0.227 & 5.85 & 0.379 & 3.40 & 3.51 & 2.43 \\
- & \textit{Fine} & 373 & \textcolor{gray!50}{\textbf{0.340}} & 1.15 & 0.228 & 5.92 & 0.374 & 3.38 & 3.44 & 2.42 \\
+ Inter DGR & \textit{Dynamic}  & 373 & \textcolor{gray!50}{0.322} & 1.19 & 0.228 & 6.00 & 0.392 & 3.42 & 3.52 & 2.46 \\
\rowcolor{gray!20}+ CPGO & \textit{Dynamic}  & 373 & \textcolor{gray!50}{0.334} & \textbf{1.25} & \textbf{0.232} & \textbf{6.08} & \textbf{0.394} & \textbf{3.50} & \textbf{3.57} & \textbf{2.50} \\
\hdashline[1pt/2pt]
\noalign{\vskip 1pt}
DDPO & \textit{Fine} & 665 & \textcolor{gray!50}{0.313} & 1.13 & 0.227 & 5.90 & 0.391 & 3.39 & 3.57 & 2.43 \\
- & \textit{Coarse}  & 665 & \textcolor{gray!50}{0.309} & 1.11 & 0.226 & 5.91 & 0.389 & 3.37 & 3.57 & 2.43 \\
+ Inter DGR & \textit{Dynamic}  & 665 & \textcolor{gray!50}{0.316} & 1.15 & \textbf{0.228} & 5.95 & 0.392 & 3.41 & 3.59 & 2.45 \\
\rowcolor{gray!20}+ CPGO & \textit{Dynamic}  & 665 & \textcolor{gray!50}{0.324} & 1.20 & \textbf{0.228} & 6.00 & 0.396 & 3.48 & 3.59 & 2.48 \\
+  Intra DGR & \textit{Dynamic}  & \textbf{541} & \textcolor{gray!50}{0.322} & 1.18 & \textbf{0.228} & 5.95 & 0.401 & 3.50 & 3.57 & 2.47 \\
\rowcolor{gray!20}+  Intra DGR & \textit{Dynamic}  & \textbf{658} & \textcolor{gray!50}{\textbf{0.326}} & \textbf{1.22} & \textbf{0.228} & \textbf{6.03} & \textbf{0.412} & \textbf{3.57} & \textbf{3.61} & \textbf{2.51} \\
\bottomrule
\end{tabular}%
}
\end{wraptable}
\setlength{\intextsep}{8pt plus 2pt minus 2pt}

\noindent \textbf{Effectiveness of CPGO.} In Tab.~\ref{tab:ablation}, CPGO raises the average OoD score to 2.523 and ImageReward by 9.2\%, which we attribute to constrained interpolation in smoother distributions. Fig.~\ref{fig:exp_step_baseline} further shows that CPGO sustains quality under few-step sampling where DanceGRPO degrades. Notably, 20-step sampling reaches 2.526, slightly above the full-step score of 2.523 in Tab.~\ref{tab:main_results}, indicating improved vector-field consistency across steps.

\noindent \textbf{Effectiveness of Intra-Group DGR.} Intra-group DGR uses intermediate perception to preserve diversity and accelerate fine-tuning. Fig.~\ref{fig:reward_curves} shows an 18.6\% reduction in DanceGRPO training time, from 665 to 541 s. Tab.~\ref{tab:ablation} reports a similar average OoD score of 2.52 versus 2.49. Unlike best-of-$N$, which increases compute and often yields similar top/bottom-$N$ samples as diversity declines during RFT, Intra-group DGR maintains diversity with lower overhead.

\WFclear
\begin{figure}[!tb]
    \centering
    \edef\comparisonplotheight{\the\dimexpr\textwidth*155/1000\relax}
    \begin{minipage}[t]{0.23\linewidth}
        \centering
        \vspace{0pt}
        \includegraphics[height=\comparisonplotheight]{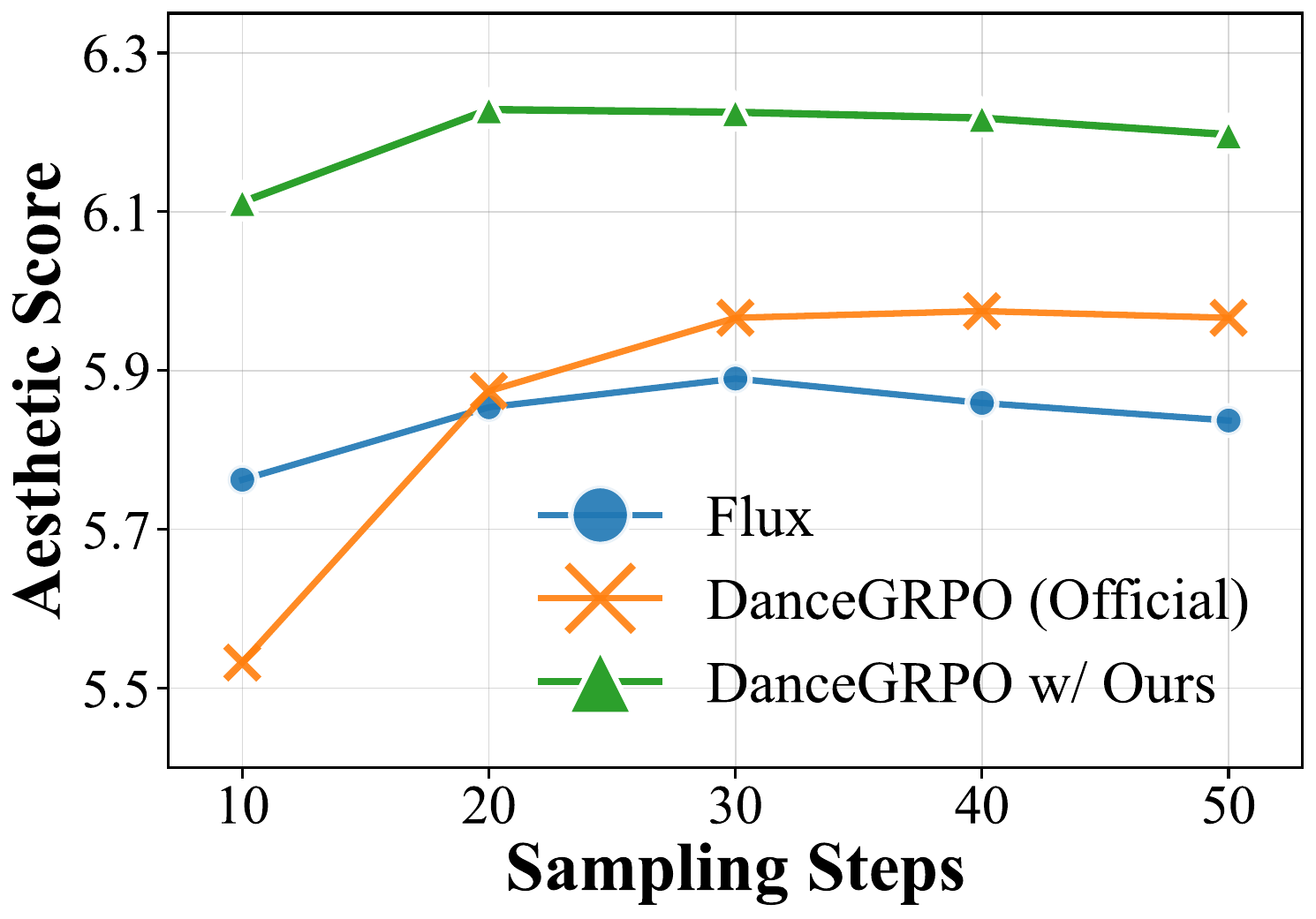}
        \caption{\textbf{Performance across sampling steps.} CPGO preserves generation quality under few-step sampling.}
        \label{fig:exp_step_baseline}
    \end{minipage}
    \hfill
    \begin{minipage}[t]{0.47\linewidth}
        \centering
        \vspace{0pt}
        \includegraphics[height=\comparisonplotheight]{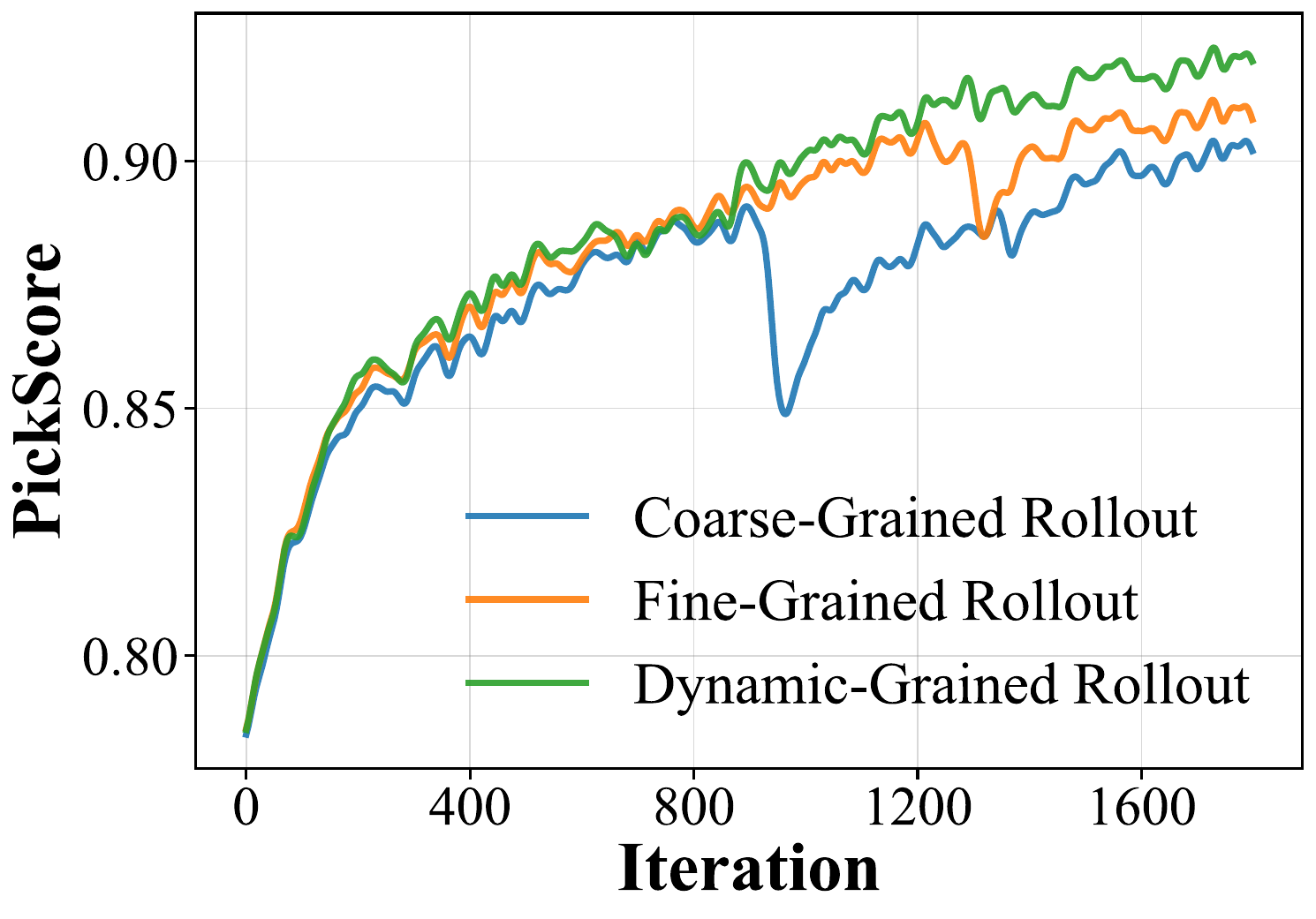}\hfill
        \includegraphics[height=\comparisonplotheight]{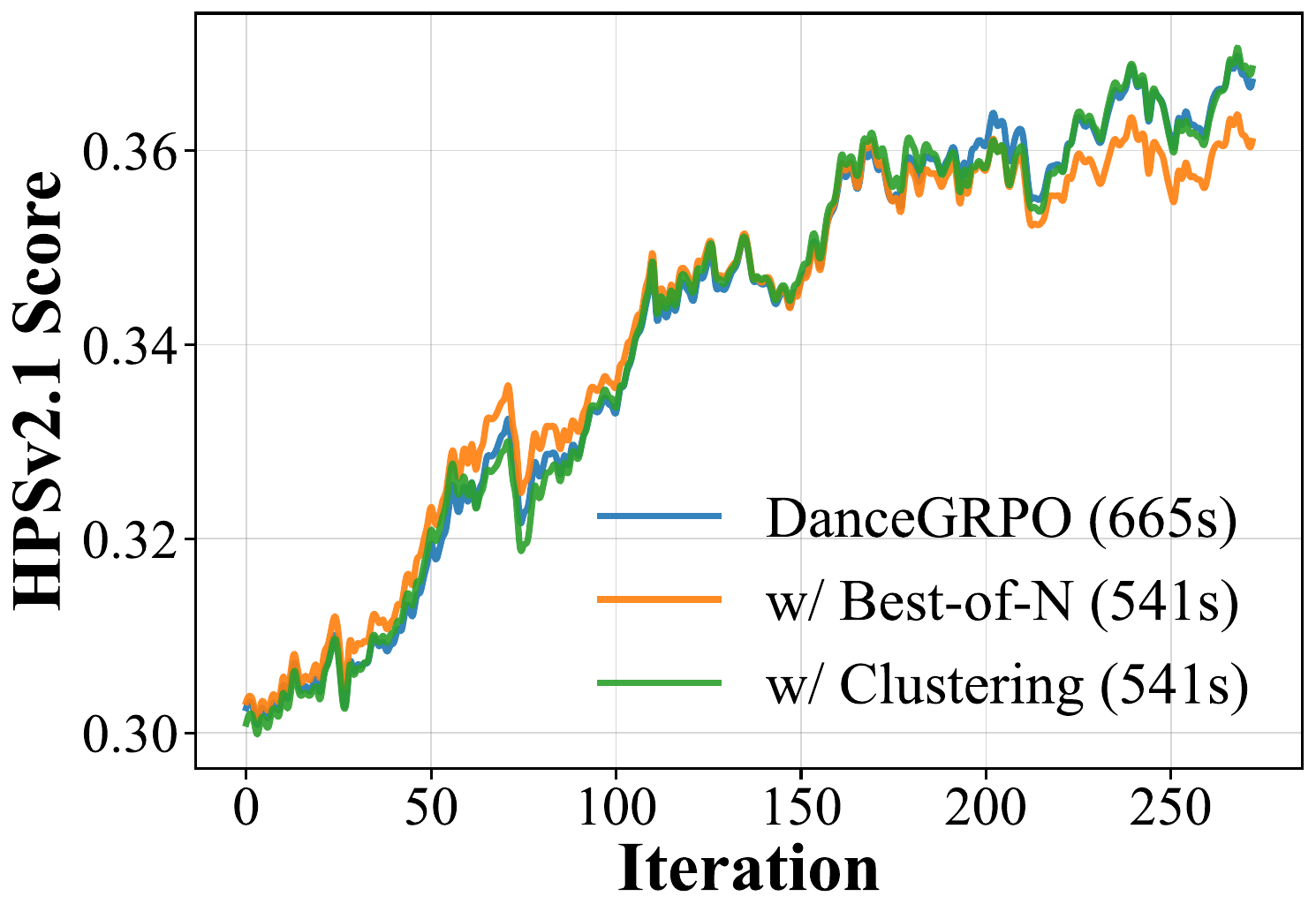}
        \caption{\textbf{Reward curves and training efficiency.} Left: Dynamic Granularity Rollout outperforms static exploration regimes. Right: Intra-Group DGR yields speedup without penalty.}
        \label{fig:reward_curves}
    \end{minipage}
    \hfill
    \begin{minipage}[t]{0.26\linewidth}
        \centering
        \vspace{0pt}
        \includegraphics[height=\comparisonplotheight]{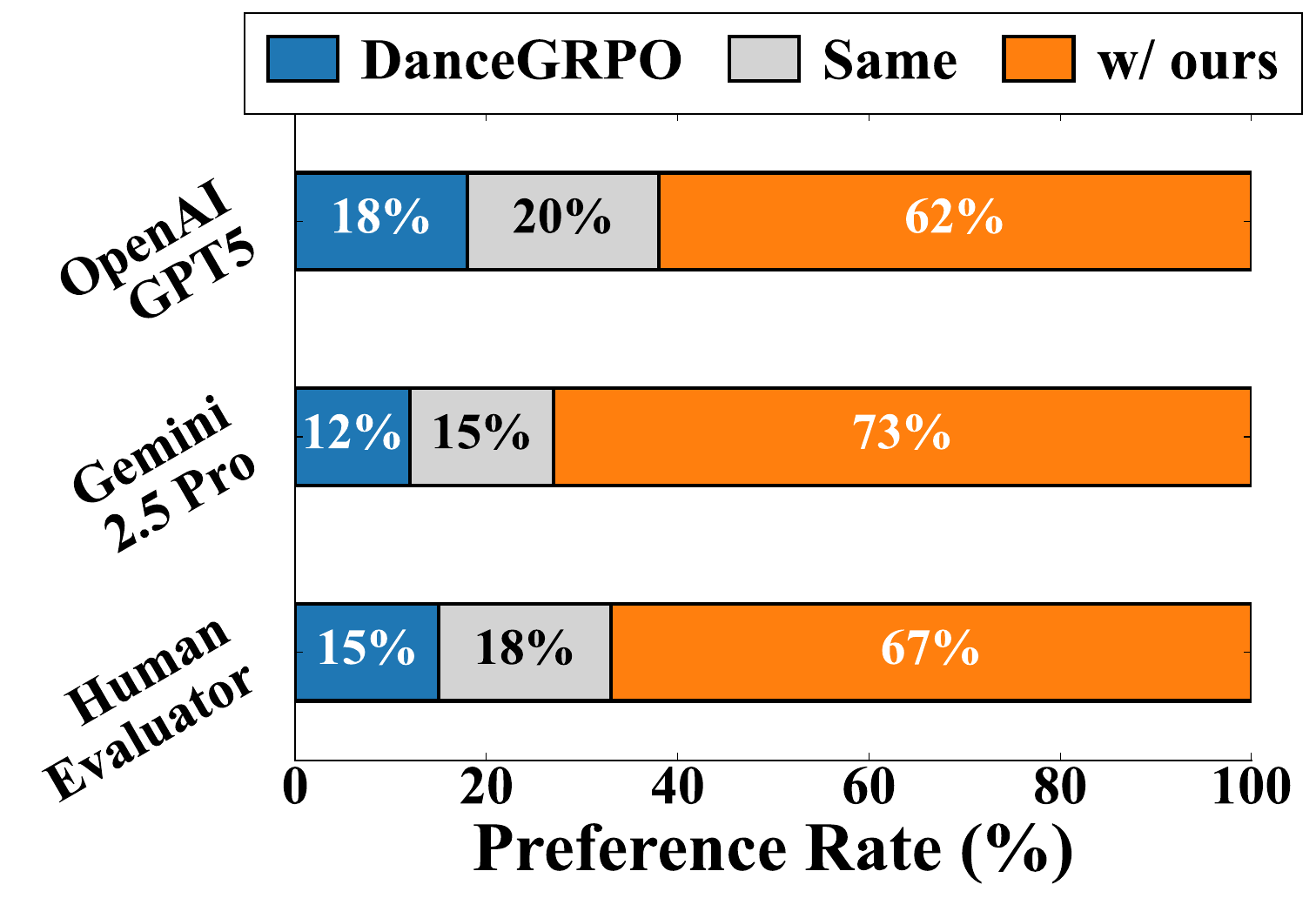}
        \caption{\textbf{Results on LLM and user evaluation.} Win rates on visual quality and hallucination reduction.}
        \label{fig:user_study}
    \end{minipage}
    \vspace{-4pt}
\end{figure}

\noindent \textbf{More Results \& Discussion.} Additional results cover optimization-granularity comparisons in App.~\ref{app:com_fine_coarse_dynamic}, Coarse Progress Perception in App.~\ref{app:progressive_perception}, noise and hyperparameter analyses in App.~\ref{app:initial_noise_analysis} and App.~\ref{app:param_sensitivity_analysis}, and visualizations in App.~\ref{app:visualization}.
\par
\WFclear
\endgroup

\section{Conclusion}
This work investigates why reinforcement fine-tuning exacerbates visual hallucinations in flow-based models, tracing them to limited exploration that concentrates feedback on local details and to trajectory imitation that degrades the cross-step consistency of the vector field. We propose DGPO, a general framework comprising Dynamic Granularity Rollout to balance fine- and coarse-grained preference signals, and Consistent Policy Gradient Optimization to preserve single-step prediction fidelity along the ODE trajectory. We further introduce a Visual Hallucination Evaluator for assessment. Extensive experiments demonstrate that DGPO reduces low-level artifacts by 49\% and high-level hallucinations by 38\%.

\noindent{\textbf{Limitations.}} Although this work provides a preliminary investigation into visual hallucinations during post-training and proposes an initial mitigation strategy, we recognize that addressing this challenge requires not only a robust RFT method but also more powerful and scalable reward models. We anticipate future developments of stronger open-source reward models by the community.

\subsubsection*{Broader Impact Statement}
This work studies visual hallucinations in flow-based reinforcement fine-tuning for text-to-image generative models. Reducing over-optimization artifacts and semantic drift can improve the utility of synthetic media for creative tools, design, and simulation. DGPO changes the fine-tuning procedure and uses standard reward models. More faithful outputs may still be used to create misleading or biased content. VH-Evaluator assesses visual artifacts and prompt relevance; these scores are distinct from judgments of content authenticity or misuse.

\bibliography{main}

\begin{thebibliography}{77}
\providecommand{\natexlab}[1]{#1}
\providecommand{\url}[1]{\texttt{#1}}
\expandafter\ifx\csname urlstyle\endcsname\relax
  \providecommand{\doi}[1]{doi: #1}\else
  \providecommand{\doi}{doi: \begingroup \urlstyle{rm}\Url}\fi

\bibitem[Abdel-Qader et~al.(2003)Abdel-Qader, Abudayyeh, and
  Kelly]{abdel2003analysis}
Ikhlas Abdel-Qader, Osama Abudayyeh, and Michael~E. Kelly.
\newblock Analysis of edge-detection techniques for crack identification in
  bridges.
\newblock \emph{Journal of computing in civil engineering}, 17\penalty0
  (4):\penalty0 255--263, 2003.
\newblock URL \url{https://doi.org/10.1061/(asce)0887-3801(2003)17:4(255)}.

\bibitem[Aithal et~al.(2024)Aithal, Maini, Lipton, and
  Kolter]{aithal2024understanding}
Sumukh~K Aithal, Pratyush Maini, Zachary Lipton, and J~Zico Kolter.
\newblock Understanding hallucinations in diffusion models through mode
  interpolation.
\newblock \emph{Advances in Neural Information Processing Systems},
  37:\penalty0 134614--134644, 2024.
\newblock URL \url{https://doi.org/10.52202/079017-4278}.

\bibitem[Bai et~al.(2025{\natexlab{a}})Bai, Shao, Zhou, Qi, Xu, Xiong, and
  Xie]{bai2024zigzag}
Lichen Bai, Shitong Shao, Zikai Zhou, Zipeng Qi, Zhiqiang Xu, Haoyi Xiong, and
  Zeke Xie.
\newblock Zigzag diffusion sampling: Diffusion models can self-improve via
  self-reflection.
\newblock \emph{The Thirteenth International Conference on Learning
  Representations}, 2025{\natexlab{a}}.
\newblock URL \url{https://arxiv.org/abs/2412.10891}.

\bibitem[Bai et~al.(2025{\natexlab{b}})Bai, Chen, Liu, Wang, Ge, Song, Dang,
  Wang, Wang, Tang, et~al.]{bai2025qwen25vl}
Shuai Bai, Keqin Chen, Xuejing Liu, Jialin Wang, Wenbin Ge, Sibo Song, Kai
  Dang, Peng Wang, Shijie Wang, Jun Tang, et~al.
\newblock Qwen2.5-vl technical report.
\newblock \emph{arXiv preprint arXiv:2502.13923}, 2025{\natexlab{b}}.
\newblock URL \url{https://arxiv.org/abs/2502.13923}.

\bibitem[Bansal et~al.(2016)Bansal, Raj, and Choudhury]{Bansal2016}
Raghav Bansal, Gaurav Raj, and Tanupriya Choudhury.
\newblock Blur image detection using laplacian operator and open-cv.
\newblock In \emph{2016 International Conference System Modeling and
  Advancement in Research Trends (SMART)}, pp.\  63--67, 2016.
\newblock URL \url{https://doi.org/10.1109/SYSMART.2016.7894491}.

\bibitem[Black et~al.(2023)Black, Janner, Du, Kostrikov, and
  Levine]{black2023training}
Kevin Black, Michael Janner, Yilun Du, Ilya Kostrikov, and Sergey Levine.
\newblock Training diffusion models with reinforcement learning.
\newblock \emph{arXiv preprint arXiv:2305.13301}, 2023.
\newblock URL \url{https://arxiv.org/abs/2305.13301}.

\bibitem[{Black Forest Labs}(2024)]{flux2024}
{Black Forest Labs}.
\newblock Flux.
\newblock \url{https://github.com/black-forest-labs/flux}, 2024.

\bibitem[Cai et~al.(2025)Cai, Chen, Chen, Li, Long, Pan, Qiu, Zhang, Gao, Xu,
  et~al.]{hidreami1technicalreport}
Qi~Cai, Jingwen Chen, Yang Chen, Yehao Li, Fuchen Long, Yingwei Pan, Zhaofan
  Qiu, Yiheng Zhang, Fengbin Gao, Peihan Xu, et~al.
\newblock Hidream-i1: A high-efficient image generative foundation model with
  sparse diffusion transformer.
\newblock \emph{arXiv preprint arXiv:2505.22705}, 2025.
\newblock URL \url{https://arxiv.org/abs/2505.22705}.

\bibitem[Calandriello et~al.(2024)Calandriello, Guo, Munos, Rowland, Tang,
  {\'A}vila~Pires, Richemond, Le~Lan, Valko, Liu,
  et~al.]{calandriello2024human}
Daniele Calandriello, Daniel Guo, R{\'e}mi Munos, Mark Rowland, Yunhao Tang,
  Bernardo {\'A}vila~Pires, Pierre~Harvey Richemond, Charline Le~Lan, Michal
  Valko, Tianqi Liu, et~al.
\newblock Human alignment of large language models through online preference
  optimisation.
\newblock \emph{arXiv preprint arXiv:2403.08635}, 2024.
\newblock URL \url{https://arxiv.org/abs/2403.08635}.

\bibitem[Chen et~al.(2024)Chen, Yang, Yang, Chen, He, Wang, and
  Li]{chen2024find}
Changgu Chen, Libing Yang, Xiaoyan Yang, Lianggangxu Chen, Gaoqi He, Changbo
  Wang, and Yang Li.
\newblock Find: Fine-tuning initial noise distribution with policy optimization
  for diffusion models.
\newblock In \emph{Proceedings of the 32nd ACM International Conference on
  Multimedia}, pp.\  6735--6744, 2024.
\newblock URL \url{https://doi.org/10.1145/3664647.3681047}.

\bibitem[{DanceGRPO GitHub Issues}(2025{\natexlab{a}})]{dance_grpo_issue1}
{DanceGRPO GitHub Issues}.
\newblock Some results on different reward models ({DanceGRPO} issue \#36).
\newblock \url{https://github.com/XueZeyue/DanceGRPO/issues/36},
  2025{\natexlab{a}}.
\newblock Accessed: 2025-07-18.

\bibitem[{DanceGRPO GitHub Issues}(2025{\natexlab{b}})]{dance_grpo_issue2}
{DanceGRPO GitHub Issues}.
\newblock Questions about the difference with {DDPO} ({DanceGRPO} issue \#72).
\newblock \url{https://github.com/XueZeyue/DanceGRPO/issues/72},
  2025{\natexlab{b}}.
\newblock Accessed: 2025-09-16.

\bibitem[Deng et~al.(2024)Deng, Wang, Wei, Hou, and Grundmann]{deng2024prdp}
Fei Deng, Qifei Wang, Wei Wei, Tingbo Hou, and Matthias Grundmann.
\newblock Prdp: Proximal reward difference prediction for large-scale reward
  finetuning of diffusion models.
\newblock In \emph{Proceedings of the IEEE/CVF Conference on Computer Vision
  and Pattern Recognition}, pp.\  7423--7433, 2024.
\newblock URL \url{https://doi.org/10.1109/cvpr52733.2024.00709}.

\bibitem[Discus0434(2024)]{discus0434_aesthetic_2024}
Discus0434.
\newblock {Aesthetic-Predictor-v2-5}: Siglip-based aesthetic score predictor.
\newblock
  \url{https://github.com/discus0434/aesthetic-predictor-v2-5?tab=readme-ov-file},
  2024.
\newblock Accessed: May 27, 2024.

\bibitem[Esser et~al.(2024)Esser, Kulal, Blattmann, Entezari, M{\"u}ller,
  Saini, Levi, Lorenz, Sauer, Boesel, et~al.]{esser2024scaling}
Patrick Esser, Sumith Kulal, Andreas Blattmann, Rahim Entezari, Jonas
  M{\"u}ller, Harry Saini, Yam Levi, Dominik Lorenz, Axel Sauer, Frederic
  Boesel, et~al.
\newblock Scaling rectified flow transformers for high-resolution image
  synthesis.
\newblock \emph{arXiv preprint arXiv:2403.03206}, 2024.
\newblock URL \url{https://arxiv.org/abs/2403.03206}.

\bibitem[Fan et~al.(2023)Fan, Watkins, Du, Liu, Ryu, Boutilier, Abbeel,
  Ghavamzadeh, Lee, and Lee]{fan2023dpok}
Ying Fan, Olivia Watkins, Yuqing Du, Hao Liu, Moonkyung Ryu, Craig Boutilier,
  Pieter Abbeel, Mohammad Ghavamzadeh, Kangwook Lee, and Kimin Lee.
\newblock Dpok: Reinforcement learning for fine-tuning text-to-image diffusion
  models.
\newblock \emph{Advances in Neural Information Processing Systems},
  36:\penalty0 79858--79885, 2023.
\newblock URL \url{https://doi.org/10.52202/075280-3497}.

\bibitem[Fei et~al.(2026)Fei, Ye, Chen, Xing, Lai, and Zhu]{fei2026lucidnft}
Song Fei, Tian Ye, Sixiang Chen, Zhaohu Xing, Jianyu Lai, and Lei Zhu.
\newblock {LucidNFT}: {LR}-anchored multi-reward preference optimization for
  flow-based real-world super-resolution, 2026.
\newblock URL \url{https://arxiv.org/abs/2603.05947}.

\bibitem[Fu et~al.(2025{\natexlab{a}})Fu, Zhou, Chen, Jing, Nguyen, Liu, Zeng,
  Ma, Zhang, and Wu]{fu2025counting}
Shuai Fu, Jian Zhou, Qi~Chen, Huang Jing, Huy~Anh Nguyen, Xiaohan Liu, Zhixiong
  Zeng, Lin Ma, Quanshi Zhang, and Qi~Wu.
\newblock Counting hallucinations in diffusion models.
\newblock \emph{arXiv preprint arXiv:2510.13080}, 2025{\natexlab{a}}.
\newblock URL \url{https://arxiv.org/abs/2510.13080}.

\bibitem[Fu et~al.(2025{\natexlab{b}})Fu, Ma, Guo, Dong, Yang, Sin, Zhou, He,
  Fu, Zhou, Huang, and Li]{fu2025dynamic}
Xiaolong Fu, Lichen Ma, Zipeng Guo, ShiPing Dong, Lan Yang, Tan~Lit Sin,
  Gaojing Zhou, Yu~He, Jingling Fu, Shizhe Zhou, Junshi Huang, and Jason Li.
\newblock Dynamic-treerpo: Breaking the independent trajectory bottleneck with
  structured sampling.
\newblock \emph{arXiv preprint arXiv:2509.23352}, 2025{\natexlab{b}}.
\newblock URL \url{https://arxiv.org/abs/2509.23352}.

\bibitem[Ghosh et~al.(2023)Ghosh, Hajishirzi, and Schmidt]{ghosh2023geneval}
Dhruba Ghosh, Hannaneh Hajishirzi, and Ludwig Schmidt.
\newblock Geneval: An object-focused framework for evaluating text-to-image
  alignment.
\newblock \emph{Advances in Neural Information Processing Systems},
  36:\penalty0 52132--52152, 2023.
\newblock URL \url{https://doi.org/10.52202/075280-2270}.

\bibitem[Guo et~al.(2025)Guo, Yang, Zhang, Song, Wang, Zhu, Xu, Zhang, Ma, Bi,
  et~al.]{guo2025deepseek}
Daya Guo, Dejian Yang, Haowei Zhang, Junxiao Song, Peiyi Wang, Qihao Zhu,
  Runxin Xu, Ruoyu Zhang, Shirong Ma, Xiao Bi, et~al.
\newblock Deepseek-r1 incentivizes reasoning in llms through reinforcement
  learning.
\newblock \emph{Nature}, 645\penalty0 (8081):\penalty0 633--638, 2025.
\newblock URL \url{https://doi.org/10.1038/s41586-025-09422-z}.

\bibitem[Haber et~al.(2026)Haber, Ahamed, Rahim~Siddiqui, Zakariaei, and
  Eliasof]{haber2026iterativeflow}
Eldad Haber, Shadab Ahamed, Md.~Shahriar Rahim~Siddiqui, Niloufar Zakariaei,
  and Moshe Eliasof.
\newblock Iterative flow matching: Path correction and gradual refinement for
  enhanced generative modeling.
\newblock \emph{SIAM Journal on Scientific Computing}, 48\penalty0
  (4):\penalty0 C814--C831, 2026.
\newblock URL \url{https://doi.org/10.1137/25M1736633}.

\bibitem[He et~al.(2025)He, Fu, Zhao, Li, Yang, Yin, Rao, and
  Zhang]{he2025tempflow}
Xiaoxuan He, Siming Fu, Yuke Zhao, Wanli Li, Jian Yang, Dacheng Yin, Fengyun
  Rao, and Bo~Zhang.
\newblock Tempflow-grpo: When timing matters for grpo in flow models.
\newblock \emph{arXiv preprint arXiv:2508.04324}, 2025.
\newblock URL \url{https://arxiv.org/abs/2508.04324}.

\bibitem[Ho et~al.(2020)Ho, Jain, and Abbeel]{ho2020denoising}
Jonathan Ho, Ajay Jain, and Pieter Abbeel.
\newblock Denoising diffusion probabilistic models.
\newblock In \emph{Advances in Neural Information Processing Systems},
  volume~33, pp.\  6840--6851, 2020.
\newblock URL \url{https://arxiv.org/abs/2006.11239}.

\bibitem[Huang et~al.(2025{\natexlab{a}})Huang, Duan, Sun, Xie, Li, and
  Liu]{huang2025t2i}
Kaiyi Huang, Chengqi Duan, Kaiyue Sun, Enze Xie, Zhenguo Li, and Xihui Liu.
\newblock T2i-compbench++: An enhanced and comprehensive benchmark for
  compositional text-to-image generation.
\newblock \emph{IEEE Transactions on Pattern Analysis and Machine
  Intelligence}, 47\penalty0 (5):\penalty0 3563--3579, 2025{\natexlab{a}}.
\newblock URL \url{https://doi.org/10.1109/tpami.2025.3531907}.

\bibitem[Huang et~al.(2025{\natexlab{b}})Huang, Yu, Ma, Zhong, Feng, Wang,
  Chen, Peng, Feng, Qin, et~al.]{huang2025survey}
Lei Huang, Weijiang Yu, Weitao Ma, Weihong Zhong, Zhangyin Feng, Haotian Wang,
  Qianglong Chen, Weihua Peng, Xiaocheng Feng, Bing Qin, et~al.
\newblock A survey on hallucination in large language models: Principles,
  taxonomy, challenges, and open questions.
\newblock \emph{ACM Transactions on Information Systems}, 43\penalty0
  (2):\penalty0 1--55, 2025{\natexlab{b}}.
\newblock URL \url{https://doi.org/10.1145/3703155}.

\bibitem[Kim et~al.(2025)Kim, Seo, Kim, Park, Park, and Yu]{kim2025vip}
Jisoo Kim, Wooseok Seo, Junwan Kim, Seungho Park, Sooyeon Park, and Youngjae
  Yu.
\newblock V.i.p.: Iterative online preference distillation for efficient video
  diffusion models.
\newblock In \emph{Proceedings of the IEEE/CVF International Conference on
  Computer Vision}, pp.\  17235--17245, 2025.
\newblock URL \url{https://doi.org/10.1109/iccv51701.2025.01601}.

\bibitem[Kim et~al.(2024)Kim, Jin, Diethe, Figini, Tregidgo, Mullokandov,
  Teare, and Alexander]{kim2024tackling}
Seunghoi Kim, Chen Jin, Tom Diethe, Matteo Figini, Henry~FJ Tregidgo, Asher
  Mullokandov, Philip Teare, and Daniel~C Alexander.
\newblock Tackling structural hallucination in image translation with local
  diffusion.
\newblock In \emph{European Conference on Computer Vision}, pp.\  87--103.
  Springer, 2024.
\newblock URL \url{https://doi.org/10.1007/978-3-031-73004-7_6}.

\bibitem[Kirstain et~al.(2023)Kirstain, Polyak, Singer, Matiana, Penna, and
  Levy]{kirstain2023pick}
Yuval Kirstain, Adam Polyak, Uriel Singer, Shahbuland Matiana, Joe Penna, and
  Omer Levy.
\newblock Pick-a-pic: An open dataset of user preferences for text-to-image
  generation.
\newblock \emph{Advances in Neural Information Processing Systems},
  36:\penalty0 36652--36663, 2023.
\newblock URL \url{https://doi.org/10.52202/075280-1594}.

\bibitem[{Kolors Team}(2024)]{kolors}
{Kolors Team}.
\newblock Kolors: Effective training of diffusion model for photorealistic
  text-to-image synthesis.
\newblock Kuaishou Technology Technical Report, 2024.
\newblock URL
  \url{https://github.com/Kwai-Kolors/Kolors/blob/master/imgs/Kolors_paper.pdf}.

\bibitem[Kong et~al.(2024)Kong, Tian, Zhang, Min, Dai, Zhou, Xiong, Li, Wu,
  Zhang, et~al.]{kong2024hunyuanvideo}
Weijie Kong, Qi~Tian, Zijian Zhang, Rox Min, Zuozhuo Dai, Jin Zhou, Jiangfeng
  Xiong, Xin Li, Bo~Wu, Jianwei Zhang, et~al.
\newblock {HunyuanVideo}: A systematic framework for large video generative
  models, 2024.
\newblock URL \url{https://arxiv.org/abs/2412.03603}.

\bibitem[Lanchantin et~al.(2025)Lanchantin, Chen, Lan, Li, Saha, Wang, Xu, Yu,
  Yuan, Weston, et~al.]{lanchantin2025bridging}
Jack Lanchantin, Angelica Chen, Janice Lan, Xian Li, Swarnadeep Saha, Tianlu
  Wang, Jing Xu, Ping Yu, Weizhe Yuan, Jason~E Weston, et~al.
\newblock Bridging offline and online reinforcement learning for llms.
\newblock \emph{arXiv preprint arXiv:2506.21495}, 2025.
\newblock URL \url{https://arxiv.org/abs/2506.21495}.

\bibitem[Li \& Ng(2025)Li and Ng]{li2025reasoninghallucinate}
Junyi Li and Hwee~Tou Ng.
\newblock Reasoning models hallucinate more: Factuality-aware reinforcement
  learning for large reasoning models.
\newblock In \emph{Advances in Neural Information Processing Systems},
  volume~38, pp.\  167712--167733. Curran Associates, Inc., 2025.
\newblock URL \url{https://doi.org/10.52202/085713-5052}.

\bibitem[Li et~al.(2025{\natexlab{a}})Li, Cui, Huang, Kong, Cheng, Zeng, Ma,
  Fan, Yang, Zhong, and Bo]{li2025mixgrpo}
Junzhe Li, Yutao Cui, Tao Huang, Weijie Kong, Yiming Cheng, Chuxuan Zeng,
  Yinping Ma, Chun Fan, Miles Yang, Zhao Zhong, and Liefeng Bo.
\newblock Mixgrpo: Unlocking flow-based grpo efficiency with mixed ode-sde.
\newblock \emph{arXiv preprint arXiv:2507.21802}, 2025{\natexlab{a}}.
\newblock URL \url{https://arxiv.org/abs/2507.21802}.

\bibitem[Li et~al.(2025{\natexlab{b}})Li, Wang, Zhu, Zhao, Lu, She, and
  Zhang]{li2025branchgrpo}
Yuming Li, Yikai Wang, Yuying Zhu, Zhongyu Zhao, Ming Lu, Qi~She, and Shanghang
  Zhang.
\newblock Branchgrpo: Stable and efficient grpo with structured branching in
  diffusion models.
\newblock \emph{arXiv preprint arXiv:2509.06040}, 2025{\natexlab{b}}.
\newblock URL \url{https://arxiv.org/abs/2509.06040}.

\bibitem[Li et~al.(2024)Li, Zhang, Lin, Xiong, Long, Deng, Zhang, Liu, Huang,
  Xiao, et~al.]{li2024hunyuandit}
Zhimin Li, Jianwei Zhang, Qin Lin, Jiangfeng Xiong, Yanxin Long, Xinchi Deng,
  Yingfang Zhang, Xingchao Liu, Minbin Huang, Zedong Xiao, et~al.
\newblock Hunyuan-dit: A powerful multi-resolution diffusion transformer with
  fine-grained chinese understanding, 2024.
\newblock URL \url{https://arxiv.org/abs/2405.08748}.

\bibitem[Lipman et~al.(2022)Lipman, Chen, Ben-Hamu, Nickel, and
  Le]{lipman2022flow}
Yaron Lipman, Ricky~TQ Chen, Heli Ben-Hamu, Maximilian Nickel, and Matt Le.
\newblock Flow matching for generative modeling.
\newblock \emph{arXiv preprint arXiv:2210.02747}, 2022.
\newblock URL \url{https://arxiv.org/abs/2210.02747}.

\bibitem[Liu et~al.(2024)Liu, Shao, Li, Bai, Xu, Xiong, Kwok, Helal, and
  Xie]{liu2024alignment}
Buhua Liu, Shitong Shao, Bao Li, Lichen Bai, Zhiqiang Xu, Haoyi Xiong, James
  Kwok, Sumi Helal, and Zeke Xie.
\newblock Alignment of diffusion models: Fundamentals, challenges, and future.
\newblock \emph{arXiv preprint arXiv:2409.07253}, 2024.
\newblock URL \url{https://arxiv.org/abs/2409.07253}.

\bibitem[Liu et~al.(2025{\natexlab{a}})Liu, Liu, Liang, Li, Liu, Wang, Wan,
  Zhang, and Ouyang]{liu2025flow}
Jie Liu, Gongye Liu, Jiajun Liang, Yangguang Li, Jiaheng Liu, Xintao Wang,
  Pengfei Wan, Di~Zhang, and Wanli Ouyang.
\newblock Flow-grpo: Training flow matching models via online rl.
\newblock \emph{arXiv preprint arXiv:2505.05470}, 2025{\natexlab{a}}.
\newblock URL \url{https://arxiv.org/abs/2505.05470}.

\bibitem[Liu et~al.(2025{\natexlab{b}})Liu, Duan, Hu, Yang, Cai, Ye, Liu,
  Zhang, and Zhai]{liu2025f}
Lu~Liu, Huiyu Duan, Qiang Hu, Liu Yang, Chunlei Cai, Tianxiao Ye, Huayu Liu,
  Xiaoyun Zhang, and Guangtao Zhai.
\newblock F-bench: Rethinking human preference evaluation metrics for
  benchmarking face generation, customization, and restoration.
\newblock In \emph{Proceedings of the IEEE/CVF International Conference on
  Computer Vision}, pp.\  10982--10994, 2025{\natexlab{b}}.
\newblock URL \url{https://doi.org/10.1109/iccv51701.2025.01022}.

\bibitem[Liu et~al.(2022)Liu, Gong, and Liu]{liu2022flow}
Xingchao Liu, Chengyue Gong, and Qiang Liu.
\newblock Flow straight and fast: Learning to generate and transfer data with
  rectified flow.
\newblock \emph{arXiv preprint arXiv:2209.03003}, 2022.
\newblock URL \url{https://arxiv.org/abs/2209.03003}.

\bibitem[Ma et~al.(2025)Ma, Wu, Sun, and Li]{ma2025hpsv3}
Yuhang Ma, Xiaoshi Wu, Keqiang Sun, and Hongsheng Li.
\newblock Hpsv3: Towards wide-spectrum human preference score.
\newblock In \emph{Proceedings of the IEEE/CVF International Conference on
  Computer Vision}, pp.\  15086--15095, 2025.
\newblock URL \url{https://doi.org/10.1109/iccv51701.2025.01400}.

\bibitem[Oorloff et~al.(2025)Oorloff, Yacoob, and
  Shrivastava]{oorloff2025mitigating}
Trevine Oorloff, Yaser Yacoob, and Abhinav Shrivastava.
\newblock Mitigating hallucinations in diffusion models through adaptive
  attention modulation.
\newblock \emph{arXiv preprint arXiv:2502.16872}, 2025.
\newblock URL \url{https://arxiv.org/abs/2502.16872}.

\bibitem[Peng et~al.(2024)Peng, Cui, Tang, Qi, Dong, Bai, Han, Ge, Zhang, and
  Xia]{peng2024dreambench++}
Yuang Peng, Yuxin Cui, Haomiao Tang, Zekun Qi, Runpei Dong, Jing Bai, Chunrui
  Han, Zheng Ge, Xiangyu Zhang, and Shu-Tao Xia.
\newblock Dreambench++: A human-aligned benchmark for personalized image
  generation.
\newblock \emph{arXiv preprint arXiv:2406.16855}, 2024.
\newblock URL \url{https://arxiv.org/abs/2406.16855}.

\bibitem[Podell et~al.(2023)Podell, English, Lacey, Blattmann, Dockhorn,
  M{\"u}ller, Penna, and Rombach]{podell2023sdxl}
Dustin Podell, Zion English, Kyle Lacey, Andreas Blattmann, Tim Dockhorn, Jonas
  M{\"u}ller, Joe Penna, and Robin Rombach.
\newblock Sdxl: Improving latent diffusion models for high-resolution image
  synthesis.
\newblock \emph{arXiv preprint arXiv:2307.01952}, 2023.
\newblock URL \url{https://arxiv.org/abs/2307.01952}.

\bibitem[Radford et~al.(2021)Radford, Kim, Hallacy, Ramesh, Goh, Agarwal,
  Sastry, Askell, Mishkin, Clark, et~al.]{radford2021learning}
Alec Radford, Jong~Wook Kim, Chris Hallacy, Aditya Ramesh, Gabriel Goh,
  Sandhini Agarwal, Girish Sastry, Amanda Askell, Pamela Mishkin, Jack Clark,
  et~al.
\newblock Learning transferable visual models from natural language
  supervision.
\newblock In \emph{International conference on machine learning}, pp.\
  8748--8763. PmLR, 2021.
\newblock URL \url{https://arxiv.org/abs/2103.00020}.

\bibitem[Rakhshanfar \& Amer(2016)Rakhshanfar and Amer]{Rakhshanfar2016}
Meisam Rakhshanfar and Maria~A. Amer.
\newblock Estimation of gaussian, poissonian–gaussian, and processed visual
  noise and its level function.
\newblock \emph{IEEE Transactions on Image Processing}, 25\penalty0
  (9):\penalty0 4172--4185, 2016.
\newblock URL \url{https://doi.org/10.1109/TIP.2016.2588320}.

\bibitem[Sheng et~al.(2025)Sheng, Zhao, Chen, Yao, and
  Tang]{sheng2025understanding}
Jiayuan Sheng, Hanyang Zhao, Haoxian Chen, David~D. Yao, and Wenpin Tang.
\newblock Understanding sampler stochasticity in training diffusion models for
  rlhf.
\newblock \emph{arXiv preprint arXiv:2510.10767}, 2025.
\newblock URL \url{https://arxiv.org/abs/2510.10767}.

\bibitem[Shin et~al.(2005)Shin, Park, Yang, and Jung]{Shin2005}
Dong-Hyuk Shin, Rae-Hong Park, Seungjoon Yang, and Jae-Han Jung.
\newblock Block-based noise estimation using adaptive gaussian filtering.
\newblock \emph{IEEE Transactions on Consumer Electronics}, 51\penalty0
  (1):\penalty0 218--226, 2005.
\newblock URL \url{https://doi.org/10.1109/TCE.2005.1405723}.

\bibitem[Song et~al.(2021{\natexlab{a}})Song, Meng, and
  Ermon]{song2020denoising}
Jiaming Song, Chenlin Meng, and Stefano Ermon.
\newblock Denoising diffusion implicit models.
\newblock In \emph{International Conference on Learning Representations},
  2021{\natexlab{a}}.
\newblock URL \url{https://arxiv.org/abs/2010.02502}.

\bibitem[Song et~al.(2021{\natexlab{b}})Song, Sohl-Dickstein, Kingma, Kumar,
  Ermon, and Poole]{song2021score}
Yang Song, Jascha Sohl-Dickstein, Diederik~P. Kingma, Abhishek Kumar, Stefano
  Ermon, and Ben Poole.
\newblock Score-based generative modeling through stochastic differential
  equations.
\newblock In \emph{International Conference on Learning Representations},
  2021{\natexlab{b}}.
\newblock URL \url{https://arxiv.org/abs/2011.13456}.

\bibitem[Su et~al.(2026)Su, Wei, Cen, Wang, Chen, Yuan, and
  Chu]{su2026generationenhancesunderstandingunified}
Zihan Su, Hongyang Wei, Kangrui Cen, Yong Wang, Guanhua Chen, Chun Yuan, and
  Xiangxiang Chu.
\newblock Generation enhances understanding in unified multimodal models via
  multi-representation generation, 2026.
\newblock URL \url{https://arxiv.org/abs/2601.21406}.

\bibitem[Tan et~al.(2024)Tan, Wang, Geng, and Zhou]{tan2024sopo}
Xiaofeng Tan, Hongsong Wang, Xin Geng, and Pan Zhou.
\newblock Sopo: Text-to-motion generation using semi-online preference
  optimization.
\newblock \emph{arXiv preprint arXiv:2412.05095}, 2024.
\newblock URL \url{https://arxiv.org/abs/2412.05095}.

\bibitem[Tong et~al.(2025)Tong, Guo, Zhang, Shan, Wei, Xing, Li, and
  Heng]{tong2025delving}
Chengzhuo Tong, Ziyu Guo, Renrui Zhang, Wenyu Shan, Xinyu Wei, Zhenghao Xing,
  Hongsheng Li, and Pheng-Ann Heng.
\newblock Delving into rl for image generation with cot: A study on dpo vs.
  grpo.
\newblock \emph{arXiv preprint arXiv:2505.17017}, 2025.
\newblock URL \url{https://arxiv.org/abs/2505.17017}.

\bibitem[Wallace et~al.(2024)Wallace, Dang, Rafailov, Zhou, Lou, Purushwalkam,
  Ermon, Xiong, Joty, and Naik]{wallace2024diffusion}
Bram Wallace, Meihua Dang, Rafael Rafailov, Linqi Zhou, Aaron Lou, Senthil
  Purushwalkam, Stefano Ermon, Caiming Xiong, Shafiq Joty, and Nikhil Naik.
\newblock Diffusion model alignment using direct preference optimization.
\newblock In \emph{Proceedings of the IEEE/CVF Conference on Computer Vision
  and Pattern Recognition}, pp.\  8228--8238, 2024.
\newblock URL \url{https://doi.org/10.1109/cvpr52733.2024.00786}.

\bibitem[Wang \& Yu(2025)Wang and Yu]{wang2025coefficients}
Feng Wang and Zihao Yu.
\newblock Coefficients-preserving sampling for reinforcement learning with flow
  matching.
\newblock \emph{arXiv preprint arXiv:2509.05952}, 2025.
\newblock URL \url{https://arxiv.org/abs/2509.05952}.

\bibitem[Wang et~al.(2025{\natexlab{a}})Wang, Li, Zang, Zhou, Bu, Wang, Lu,
  Jin, and Wang]{wang2025pref}
Yibin Wang, Zhimin Li, Yuhang Zang, Yujie Zhou, Jiazi Bu, Chunyu Wang, Qinglin
  Lu, Cheng Jin, and Jiaqi Wang.
\newblock Pref-grpo: Pairwise preference reward-based grpo for stable
  text-to-image reinforcement learning.
\newblock \emph{arXiv preprint arXiv:2508.20751}, 2025{\natexlab{a}}.
\newblock URL \url{https://arxiv.org/abs/2508.20751}.

\bibitem[Wang et~al.(2025{\natexlab{b}})Wang, Zang, Li, Jin, and
  Wang]{wang2025unified}
Yibin Wang, Yuhang Zang, Hao Li, Cheng Jin, and Jiaqi Wang.
\newblock Unified reward model for multimodal understanding and generation.
\newblock \emph{arXiv preprint arXiv:2503.05236}, 2025{\natexlab{b}}.
\newblock URL \url{https://arxiv.org/abs/2503.05236}.

\bibitem[Wei et~al.(2024)Wei, Zhou, Sun, and Zhang]{Wei_2024_CVPR}
Min Wei, Jingkai Zhou, Junyao Sun, and Xuesong Zhang.
\newblock Adversarial score distillation: When score distillation meets gan.
\newblock In \emph{Proceedings of the IEEE/CVF Conference on Computer Vision
  and Pattern Recognition (CVPR)}, pp.\  8131--8141, June 2024.
\newblock URL \url{https://doi.org/10.1109/cvpr52733.2024.00777}.

\bibitem[Weng et~al.(2025)Weng, Tan, Wang, Xie, Zhou, and
  Wang]{weng2025realign}
Wanjiang Weng, Xiaofeng Tan, Junbo Wang, Guo-Sen Xie, Pan Zhou, and Hongsong
  Wang.
\newblock Realign: Text-to-motion generation via step-aware reward-guided
  alignment.
\newblock \emph{arXiv preprint arXiv:2511.19217}, 2025.
\newblock URL \url{https://arxiv.org/abs/2511.19217}.

\bibitem[Wu et~al.(2025{\natexlab{a}})Wu, Li, Zhou, Lin, Gao, Yan, Yin, Bai,
  Xu, Chen, et~al.]{wu2025qwenimagetechnicalreport}
Chenfei Wu, Jiahao Li, Jingren Zhou, Junyang Lin, Kaiyuan Gao, Kun Yan,
  Sheng-ming Yin, Shuai Bai, Xiao Xu, Yilei Chen, et~al.
\newblock Qwen-image technical report, 2025{\natexlab{a}}.
\newblock URL \url{https://arxiv.org/abs/2508.02324}.

\bibitem[Wu et~al.(2023)Wu, Hao, Sun, Chen, Zhu, Zhao, and Li]{wu2023human}
Xiaoshi Wu, Yiming Hao, Keqiang Sun, Yixiong Chen, Feng Zhu, Rui Zhao, and
  Hongsheng Li.
\newblock Human preference score v2: A solid benchmark for evaluating human
  preferences of text-to-image synthesis.
\newblock \emph{arXiv preprint arXiv:2306.09341}, 2023.
\newblock URL \url{https://arxiv.org/abs/2306.09341}.

\bibitem[Wu et~al.(2025{\natexlab{b}})Wu, Ma, Ding, Li, Wang, Chen, Su, Zhang,
  Huang, Zhang, et~al.]{wu2025takes}
Yihong Wu, Liheng Ma, Lei Ding, Muzhi Li, Xinyu Wang, Kejia Chen, Zhan Su,
  Zhanguang Zhang, Chenyang Huang, Yingxue Zhang, et~al.
\newblock It takes two: Your grpo is secretly dpo.
\newblock \emph{arXiv preprint arXiv:2510.00977}, 2025{\natexlab{b}}.
\newblock URL \url{https://arxiv.org/abs/2510.00977}.

\bibitem[Xie \& Gong(2025)Xie and Gong]{XieDyMO}
Xin Xie and Dong Gong.
\newblock Dymo: Training-free diffusion model alignment with dynamic
  multi-objective scheduling.
\newblock In \emph{2025 IEEE/CVF Conference on Computer Vision and Pattern
  Recognition (CVPR)}, pp.\  13220--13230, 2025.
\newblock URL \url{https://doi.org/10.1109/CVPR52734.2025.01234}.

\bibitem[Xu et~al.(2023)Xu, Liu, Wu, Tong, Li, Ding, Tang, and
  Dong]{xu2023imagereward}
Jiazheng Xu, Xiao Liu, Yuchen Wu, Yuxuan Tong, Qinkai Li, Ming Ding, Jie Tang,
  and Yuxiao Dong.
\newblock Imagereward: Learning and evaluating human preferences for
  text-to-image generation.
\newblock \emph{Advances in Neural Information Processing Systems},
  36:\penalty0 15903--15935, 2023.
\newblock URL \url{https://doi.org/10.52202/075280-0700}.

\bibitem[Xue et~al.(2025)Xue, Wu, Gao, Kong, Zhu, Chen, Liu, Liu, Guo, Huang,
  et~al.]{xue2025dancegrpo}
Zeyue Xue, Jie Wu, Yu~Gao, Fangyuan Kong, Lingting Zhu, Mengzhao Chen, Zhiheng
  Liu, Wei Liu, Qiushan Guo, Weilin Huang, et~al.
\newblock Dancegrpo: Unleashing grpo on visual generation.
\newblock \emph{arXiv preprint arXiv:2505.07818}, 2025.
\newblock URL \url{https://arxiv.org/abs/2505.07818}.

\bibitem[Yang et~al.(2024)Yang, Tao, Lyu, Ge, Chen, Shen, Zhu, and
  Li]{yang2024using}
Kai Yang, Jian Tao, Jiafei Lyu, Chunjiang Ge, Jiaxin Chen, Weihan Shen,
  Xiaolong Zhu, and Xiu Li.
\newblock Using human feedback to fine-tune diffusion models without any reward
  model.
\newblock In \emph{Proceedings of the IEEE/CVF Conference on Computer Vision
  and Pattern Recognition}, pp.\  8941--8951, 2024.
\newblock URL \url{https://doi.org/10.1109/cvpr52733.2024.00854}.

\bibitem[Ye et~al.(2025)Ye, He, Yang, and Bian]{ye2025reinforcement}
Yang Ye, Tianyu He, Shuo Yang, and Jiang Bian.
\newblock Reinforcement learning with inverse rewards for world model
  post-training.
\newblock \emph{arXiv preprint arXiv:2509.23958}, 2025.
\newblock URL \url{https://arxiv.org/abs/2509.23958}.

\bibitem[Yu et~al.(2025)Yu, Liu, and Cui]{yu2025smart}
Benjamin Yu, Jackie Liu, and Justin Cui.
\newblock Smart-grpo: Smartly sampling noise for efficient rl of flow-matching
  models.
\newblock \emph{arXiv preprint arXiv:2510.02654}, 2025.
\newblock URL \url{https://arxiv.org/abs/2510.02654}.

\bibitem[Yuan et~al.(2025)Yuan, Liu, Yue, Zhang, Zuo, Wang, Zhang, and
  Zhou]{yuan2025argrpo}
Shihao Yuan, Yahui Liu, Yang Yue, Jingyuan Zhang, Wangmeng Zuo, Qi~Wang,
  Fuzheng Zhang, and Guorui Zhou.
\newblock Ar-grpo: Training autoregressive image generation models via
  reinforcement learning.
\newblock \emph{arXiv preprint arXiv:2508.06924}, 2025.
\newblock URL \url{https://arxiv.org/abs/2508.06924}.

\bibitem[Zhang \& Gienger(2024)Zhang and Gienger]{zhang2024affordance}
Fan Zhang and Michael Gienger.
\newblock Affordance-based robot manipulation with flow matching.
\newblock \emph{arXiv preprint arXiv:2409.01083}, 2024.
\newblock URL \url{https://arxiv.org/abs/2409.01083}.

\bibitem[Zhang et~al.(2025)Zhang, Yu, Ma, Zhang, Pan, Yao, Xiao, Huang, and
  Zhao]{zhang2025group}
Guohui Zhang, Hu~Yu, Xiaoxiao Ma, JingHao Zhang, Yaning Pan, Mingde Yao, Jie
  Xiao, Linjiang Huang, and Feng Zhao.
\newblock Group critical-token policy optimization for autoregressive image
  generation.
\newblock \emph{arXiv preprint arXiv:2509.22485}, 2025.
\newblock URL \url{https://arxiv.org/abs/2509.22485}.

\bibitem[Zhang et~al.(2024)Zhang, Tzeng, Du, and Kislyuk]{zhang2024large}
Yinan Zhang, Eric Tzeng, Yilun Du, and Dmitry Kislyuk.
\newblock Large-scale reinforcement learning for diffusion models.
\newblock In \emph{European Conference on Computer Vision}, pp.\  1--17.
  Springer, 2024.
\newblock URL \url{https://doi.org/10.1007/978-3-031-73036-8_1}.

\bibitem[Zheng et~al.(2025)Zheng, Chen, Ye, Wang, Zhang, Jiang, Su, Ermon, Zhu,
  and Liu]{zheng2025diffusionnft}
Kaiwen Zheng, Huayu Chen, Haotian Ye, Haoxiang Wang, Qinsheng Zhang, Kai Jiang,
  Hang Su, Stefano Ermon, Jun Zhu, and Ming-Yu Liu.
\newblock Diffusionnft: Online diffusion reinforcement with forward process.
\newblock \emph{arXiv preprint arXiv:2509.16117}, 2025.
\newblock URL \url{https://arxiv.org/abs/2509.16117}.

\bibitem[Zhou et~al.(2025{\natexlab{a}})Zhou, Ling, Bu, Wang, Zang, Wang, Niu,
  and Zhai]{zhou2025text}
Yujie Zhou, Pengyang Ling, Jiazi Bu, Yibin Wang, Yuhang Zang, Jiaqi Wang,
  Li~Niu, and Guangtao Zhai.
\newblock Fine-grained grpo for precise preference alignment in flow models.
\newblock \emph{arXiv preprint arXiv:2510.01982}, 2025{\natexlab{a}}.
\newblock URL \url{https://arxiv.org/abs/2510.01982}.

\bibitem[Zhou et~al.(2025{\natexlab{b}})Zhou, Shao, Bai, Zhang, Xu, Han, and
  Xie]{Zhou_2025_ICCV}
Zikai Zhou, Shitong Shao, Lichen Bai, Shufei Zhang, Zhiqiang Xu, Bo~Han, and
  Zeke Xie.
\newblock Golden noise for diffusion models: A learning framework.
\newblock In \emph{Proceedings of the IEEE/CVF International Conference on
  Computer Vision (ICCV)}, pp.\  17688--17697, October 2025{\natexlab{b}}.
\newblock URL \url{https://doi.org/10.1109/iccv51701.2025.01643}.

\bibitem[Ziou \& Tabbone(1998)Ziou and Tabbone]{ziou1998edge}
Djemel Ziou and Salvatore Tabbone.
\newblock Edge detection techniques-an overview.
\newblock \emph{Pattern Recognition and Image Analysis: Advances in
  Mathematical Theory and Applications}, 8\penalty0 (4):\penalty0 537--559,
  1998.
\newblock URL
  \url{https://www.researchgate.net/publication/2336485_Edge_Detection_Techniques_-_An_Overview}.

\end{thebibliography}
\bibliographystyle{tmlr}

\clearpage
\appendix
\BeforeBeginEnvironment{equation}{\normalsize}
\BeforeBeginEnvironment{equation*}{\normalsize}
\setcounter{figure}{0}
\setcounter{table}{0}
\renewcommand{\thefigure}{A\arabic{figure}}
\renewcommand{\thetable}{A\arabic{table}}
\renewcommand{\theHfigure}{appendix.\arabic{figure}}
\renewcommand{\theHtable}{appendix.\arabic{table}}
\addtocontents{toc}{\protect\AppendixTocStart}

\noindent{\LARGE\bf\sffamily Dual-Grained Policy Optimization for Hallucination Reduction in Flow-Based RL\par}
\vskip 0.15in
\noindent{\large\bf\sffamily Supplementary Material\par}
\vspace{0.6em}

This supplementary document provides additional quantitative and qualitative results, implementation details, theoretical justifications, and visualizations for DGPO. It is organized as follows. App.~\ref{app:add_exp_results} reports additional experimental results, including extended quantitative comparisons, parameter sensitivity studies, and experimental settings. App.~\ref{app:method_details} provides further method details and discussion of our dynamic-granularity rollout and consistency regularization. App.~\ref{app:vh_evaluator} introduces the proposed visual hallucination evaluator, including low-level metrics, high-level MLLM-based assessment, and evaluation prompts, while App.~\ref{app:theoretic} provides proofs for group-relative updates and trajectory imitation. Finally, App.~\ref{app:visualization} provides visualizations.
\vspace{0.8em}

\printSupplementToc

\section{Additional Experiments Results}
\label{app:add_exp_results}

This section presents comprehensive additional experimental results, including quantitative analyses, parameter sensitivity studies, and detailed experimental settings that support the main findings in the paper.

\subsection{Additional Quantitative Results}
\label{app:Quantitative}
This section summarizes the quantitative findings and sensitivity analyses presented in the following subsections. Below we provide quantitative and qualitative results for key analyses discussed throughout this supplementary document. App.~\ref{app:com_fine_coarse_dynamic} presents the qualitative comparison in Fig.~\ref{supp:app:coarse_opt} of fine-, coarse-, and dynamic optimization strategies, demonstrating how our approach balances detail refinement and semantic diversity. App.~\ref{app:progressive_perception} reports quantitative metrics for Coarse Progress Perception, including cosine similarity across intermediate steps and ranking correlation with ImageReward, as shown in Fig.~\ref{supp:app:Coarse_Progress_Perception} and Fig.~\ref{supp:app:vis_coarse_progress_perception}. App.~\ref{app:initial_noise_analysis} quantifies the distinct roles of initial noise and progressive noise in generating diverse samples, with results presented in Fig.~\ref{supp:image_variance_comparison}. App.~\ref{app:param_sensitivity_analysis} provides comprehensive parameter sensitivity analysis in Tab.~\ref{tab:Sensitivity}, with detailed hyperparameter analysis illustrated in Fig.~\ref{supp:vis_cpgo}. All results are supplemented by qualitative visualizations including Fig.~\ref{supp:vh_large} and Fig.~\ref{supp:app:blurred_img} in App.~\ref{app:visualization}.

\subsection{Additional Experiments Settings}
\label{app:add_exp_settings}

\noindent \textbf{Datasets \& Models.} For online settings, we use prompts from the  HPDv2 \cite{wu2023human} datasets for training and testing, following the HPS-v2 benchmark configuration. Considering that the existing post-training methods typically require fewer than 10$k$ data prompts, we reserve the last 400 prompts from the training set with 103.7k prompts for validation. We employ FLUX.1 DEV~\cite{flux2024} as our text-to-image model. For the additional FlowGRPO~\cite{liu2025flow} setting described here, we use PickScore~\cite{kirstain2023pick} as the training reward; this is distinct from the HPS-v2.1-trained FlowGRPO comparison in the main text (Tab.~\ref{tab:main_results}). The remaining methods use HPS-v2.1~\cite{wu2023human}, unless another reward is explicitly specified. For offline settings, we generate 20k image pairs using FLUX due to the insufficient quality of the HPDv2 dataset, with sample pairs determined by HPS-v2.1 scores.

\noindent \textbf{Metrics.} Following prior work, we evaluate our method using multiple metrics focused on diverse aspects, including human preference evaluators ImageReward~\cite{xu2023imagereward} and PickScore~\cite{kirstain2023pick}, an aesthetic model Aesthetic Predictor v2.5~\cite{discus0434_aesthetic_2024}, a semantic metric CLIP Score~\cite{radford2021learning}, and a comprehensive reward model Unified Reward~\cite{wang2025unified}. Particularly, the reward model used for training, such as HPS-v2.1, serves as the \textit{in-domain metric}, while others are treated as \textit{OoD metrics}. Here, we suggest paying more attention to the OoD metrics, as they reflect the model's generalizability.

We evaluate visual hallucinations across three key aspects: detail over-optimization, semantic consistency, and sampling trajectory consistency. Over-optimization is assessed using our VH-Evaluator, which combines low-level metrics with MLLM-based evaluation. Semantic consistency is measured by CLIP and Unified Reward scores, and trajectory consistency is determined by the straightness of the latent trajectory. Detailed descriptions of these metrics can be found in App.~\ref{app:vh_evaluator}.

\begin{table}[t]
\centering
\caption{Hyperparameter configurations for DanceGRPO and MixGRPO. MixGRPO-specific parameters are marked with $\dagger$.}
\label{tab:hyperparameters}
\small
\setlength{\tabcolsep}{10pt}
\begin{tabular}{lcc}
\toprule
\textbf{Hyperparameter} & \textbf{DanceGRPO} & \textbf{MixGRPO} \\
\midrule
Learning Rate & $1e{-5}$ & $1e{-5}$ \\
Batch Size & 2 & 2 \\
Grad Accum Steps & 12 & 3 \\
Num Generations & 12 & 12 \\
Sampling Steps & 16 & 25 \\
Resolution & 720 $\times$ 720 & 720 $\times$ 720 \\
SDE Noise Weight $\eta$ & 0.3 & 0.7 \\
Timestep Fraction & 0.6 & 0.6 \\
Clip Range & $1e{-4}$ & $1e{-4}$ \\
Max Grad Norm & 0.01 & 1.0 \\
Weight Decay & $1e{-4}$ & $1e{-4}$ \\
Adv Clip Max & 5.0 & 5.0 \\
Shift & 3 & 3 \\
\midrule
Sample Strategy$^\dagger$ & SDE & SDE-ODE \\
DPM Solver$^\dagger$ & --- & Yes \\
\bottomrule
\end{tabular}
\end{table}

\noindent \textbf{Implementation.} All experiments are conducted on 8 GPUs. Unless otherwise specified, we follow the official baseline hyperparameters listed in Tab.~\ref{tab:hyperparameters}. For DanceGRPO, we adopt the standard configuration with $\eta = 0.3$, timestep fraction = 0.6, and clip range = $10^{-4}$. For MixGRPO, we use the single HPS-v2.1 reward with $\eta = 0.7$ and progressive sampling with overlap. DPO and DDPO use DanceGRPO's base settings. Since DPO optimizes via preference pairs rather than groups, we increase its batch size to 16 to maintain equivalent sample counts per iteration. For LoRA-based fine-tuning experiments, we set the LoRA rank to 128 and the scaling factor $\alpha$ to 256. Our method introduces additional hyperparameters: the CPGO weight $\omega$, the inter-group DGR period, the coarse-grained ratio, and the intermediate perception step $t_{\mathrm{s}}$. These are selected via validation set performance. For baseline methods in comparative experiments, the KL divergence weight is set to $10^{-6}$, and we evaluate the in-domain validation metric every 20 iterations and trigger early stopping when it begins to decline. Parameter analysis for DGPO is provided in App.~\ref{app:param_sensitivity_analysis}. During evaluation, we adopt the same inference configuration as DanceGRPO. The pretrained models in Tab.~\ref{tab:sota_results} use their respective official inference settings.

\subsection{Comparison of Fine-, Coarse-, and Dynamic Optimization Strategies}
\label{app:com_fine_coarse_dynamic}

Fig.~\ref{supp:app:coarse_opt} presents a qualitative comparison of optimization strategies with different granularities. Fine-grained optimization tends to focus on local details, which can lead to over-optimization artifacts such as excessive sharpening and detail accumulation. Coarse-grained optimization, by contrast, maintains greater semantic diversity but may sacrifice fine-grained detail refinement. Our dynamic-grained optimization strategy balances both objectives by adaptively mixing coarse- and fine-grained rollouts, achieving improved visual quality with fewer hallucination artifacts.

\begin{figure}[th]
  \centering
  \includegraphics[width=0.92\linewidth]{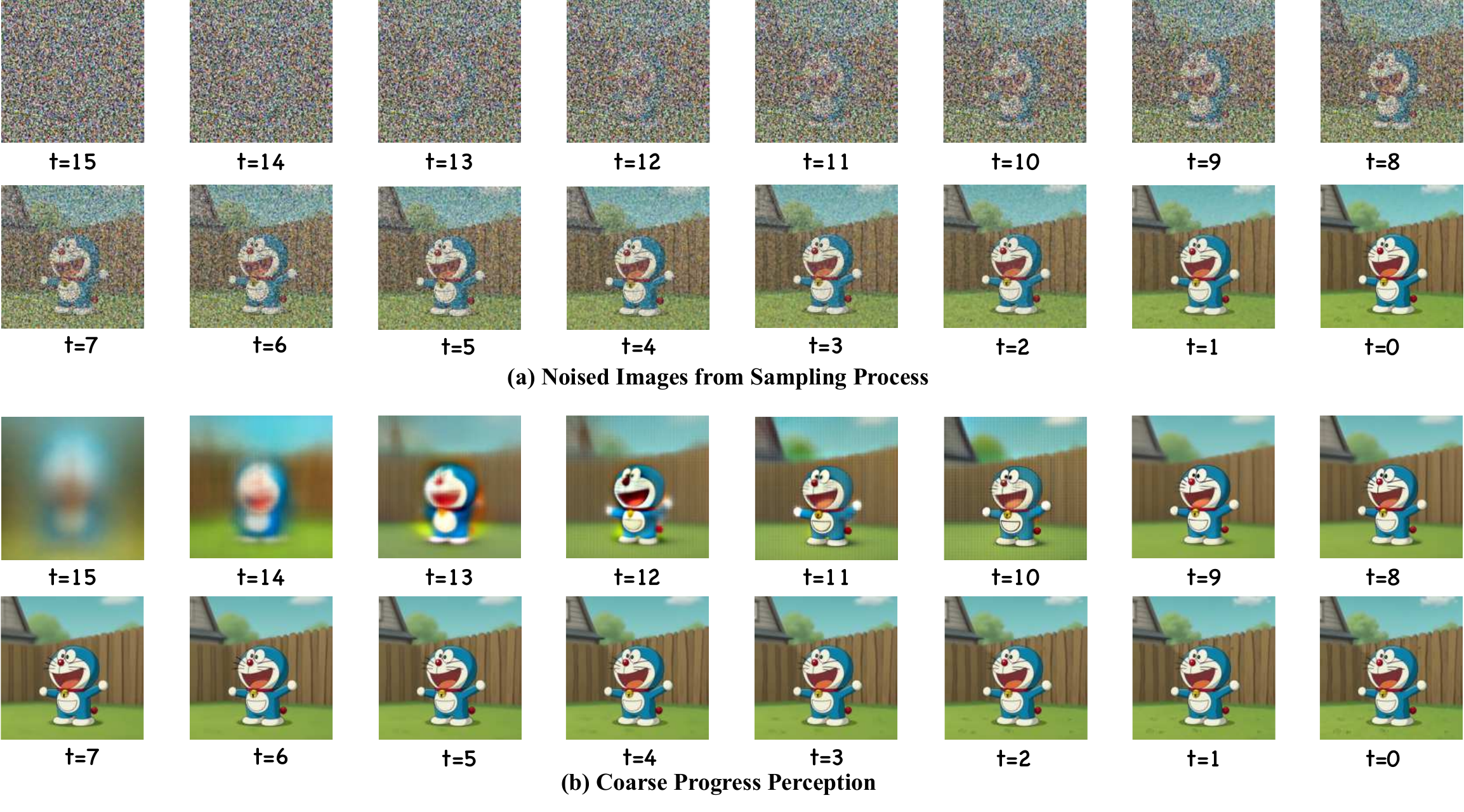}
  \caption{\textbf{Visualization of noised images and their coarse perceptions.} Coarse progress perception facilitates recognition of image semantics and quality by the reward model at earlier steps with smaller $t$.}
    \label{supp:app:vis_coarse_progress_perception}
\end{figure}

\subsection{Coarse Progress Perception Results}
\label{app:progressive_perception}

\noindent \textbf{Experimental Settings.} To assess the accuracy of the coarse progressive perception results relative to full-trajectory denoising, we conduct two analyses. We first select 50 prompts from HPDv2~\cite{wu2023human}. For each prompt, using FLUX.1~dev~\cite{flux2024} as the base model, we generate 12 images with 16 inference steps, a guidance scale of 3.5, and a resolution of $720\times720$. First, for each of the 16 intermediate steps, we store both the noisy images and their noise-aware perceptions, and compute cosine similarity using CLIP ViT-L/14 features~\cite{radford2021learning}. Second, we investigate the correlation between the Coarse Progress Perception Results and rewards derived from full-trajectory denoising. Specifically, we adopt ImageReward~\cite{xu2023imagereward} as the reward model to rank the 12 images generated for each prompt; we then rank the same set based on the Coarse Progress Perception Results and compare the two rankings to quantify their correlation. We evaluate 50 prompts with 50 independent repetitions per prompt and report the mean.

\noindent \textbf{Results.} Fig.~\ref{supp:app:Coarse_Progress_Perception} reports ImageReward ranking correlation on the left and CLIP feature similarity on the right, comparing noisy images and their coarse predictions with fully denoised images. The results indicate that coarse predictions can anticipate image quality and semantics from intermediate trajectories. We recommend $t_{\mathrm{s}}{=}12$ for more accurate training and $t_{\mathrm{s}}{=}8$ for faster training; these correspond to $t=4$ and $t=8$ on the plot's horizontal axis, respectively. We also provide qualitative examples in Fig.~\ref{supp:app:vis_coarse_progress_perception}, which illustrate the same phenomena.

\begin{figure}[t]
  \centering
  \includegraphics[width=\linewidth]{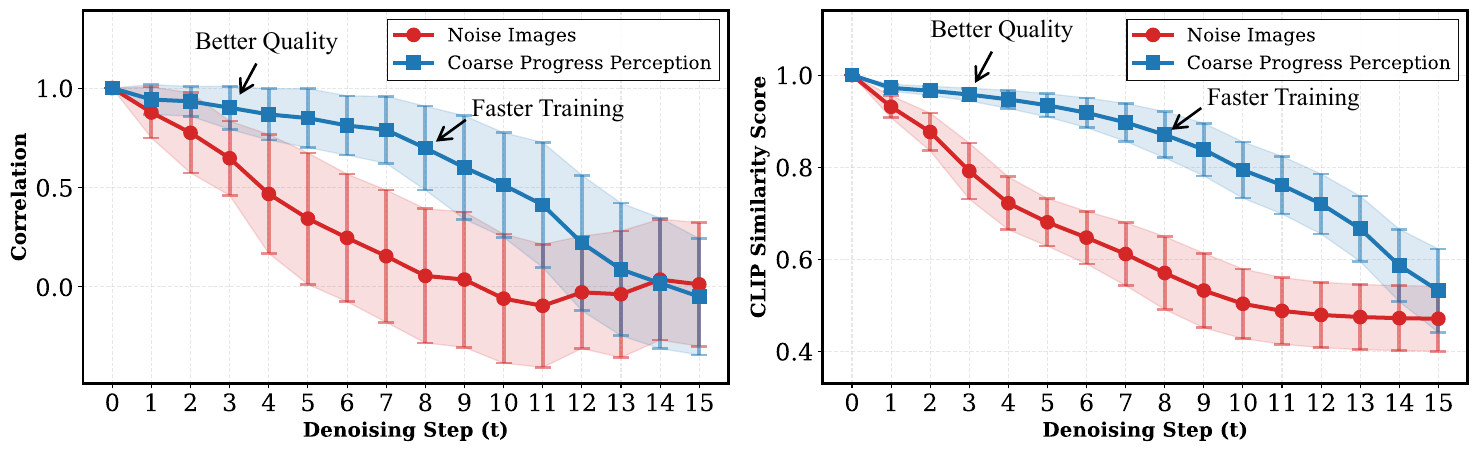}
  \caption{\textbf{Reliability of Coarse Progress Perception.} Intermediate coarse predictions are compared with fully denoised images using 12 candidates per prompt. Left: ImageReward ranking correlation. Right: CLIP feature cosine similarity.}
  \label{supp:app:Coarse_Progress_Perception}
\end{figure}

\begin{figure}[t]
  \centering
  \includegraphics[width=0.55\linewidth]{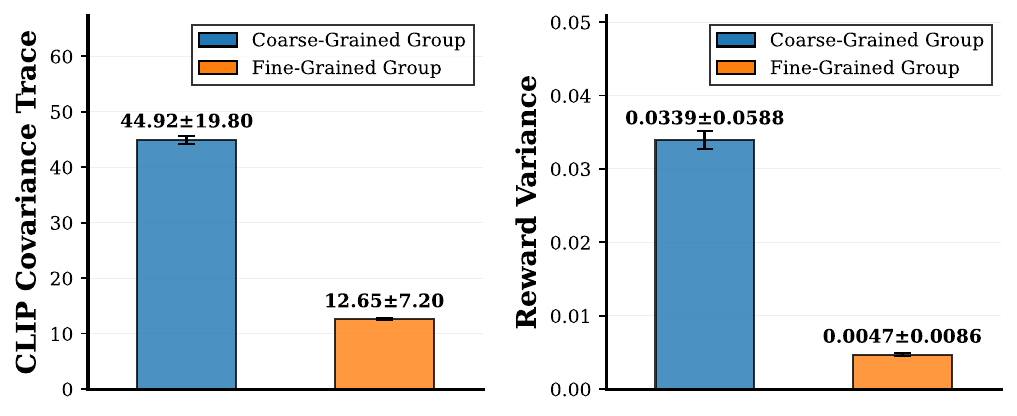}
  \caption{\textbf{Diversity and contrast of fine- and coarse-grained groups.} Left: diversity. Right: contrast. The fine-grained group isolates randomness from progressive in-trajectory noise, while the coarse-grained group includes randomness from both initial and progressive noise.}
  \label{supp:image_variance_comparison}
\end{figure}

\subsection{Comparison between Initial Noise and Progressive Noise}
\label{app:initial_noise_analysis}
\noindent \textbf{Experimental Settings.} To compare the characteristics of initial noise and progressive in-trajectory noise, we quantitatively measure contrast and diversity for both coarse-grained and fine-grained groups. Concretely, we select 50 prompts from HPDv2~\cite{wu2023human}. For each prompt, using FLUX.1~dev~\cite{flux2024} as the base model, we generate 12 images with 16 inference steps, a guidance scale of 3.5, and a resolution of $720\times720$. First, for each group, we compute diversity as the variance in the CLIP feature space, that is, the trace of the covariance matrix. Second, we examine the dispersion of reward values within each group by reporting the standard deviation of ImageReward~\cite{xu2023imagereward}. We evaluate 50 prompts with 50 independent repetitions per prompt and report the mean.

\noindent \textbf{Results.} As shown in Fig.~\ref{supp:image_variance_comparison}, we compare coarse- and fine-grained groups with respect to diversity and contrast. Diversity is defined as the trace of the covariance matrix in the CLIP feature space, and contrast is measured by the standard deviation of ImageReward within each group. Quantitatively, the coarse-grained group exhibits markedly higher diversity and contrast than the fine-grained group. This pattern indicates that randomness attributable to initial noise, present only in the coarse-grained setting, primarily shapes global semantics, whereas progressive in-trajectory noise, shared by both settings, mainly affects local details. Consistent qualitative evidence is shown in Fig.~\ref{supp:fig:coarse_fine_group}, which further corroborates the distinct roles of initial and progressive noise in the generation process.

\begin{table}[t]
\centering
\caption{Parameter sensitivity analysis. \textcolor{gray!50}{Gray text} marks in-domain HPS, and gray rows use the default period of 40 and coarse ratio of 0.25. HPS: HPSv2.1; IR: ImageReward; PS: PickScore; Aes.: Aesthetic Predictor v2.5; UR: Unified Reward; UR-S: Unified Reward-Semantic.}
\label{tab:Sensitivity}
\footnotesize
\setlength{\tabcolsep}{4.2pt}
\renewcommand{\arraystretch}{0.92}
\begin{tabular}{lcccccccccc}
\toprule
Method & Period & Ratio & HPS & IR & PS & Aes. & CLIP & UR-S & UR & Avg. \\
\midrule
Flux Dev & - & - & \textcolor{gray!50}{0.312} & 1.09 & 0.226 & 5.84 & 0.388 & 3.37 & 3.52 & 2.41 \\
DanceGRPO \cite{xue2025dancegrpo} & - & - & \textcolor{gray!50}{\textbf{0.353}} & 1.16 & 0.226 & 5.90 & 0.361 & 3.30 & 3.38 & 2.39 \\
\midrule
DGPO & 20 & 0.25 & \textcolor{gray!50}{0.351} & 1.28 & \textbf{0.231} & 6.16 & 0.378 & 3.38 & 3.57 & 2.50 \\
DGPO & 50 & 0.25 & \textcolor{gray!50}{0.346} & 1.29 & 0.229 & 6.13 & 0.387 & 3.43 & 3.60 & 2.51 \\
\rowcolor{gray!20}DGPO  & 40 & 0.25 & \textcolor{gray!50}{0.348} & \textbf{1.30} & 0.230 & \textbf{6.20} & 0.384 & 3.41 & \textbf{3.62} & \textbf{2.52} \\
\midrule
DGPO & 40 & 0.75 & \textcolor{gray!50}{0.342} & 1.25 & 0.228 & 6.04 & \textbf{0.389} & \textbf{3.46} & 3.55 & 2.49 \\
DGPO & 40 & 0.5 & \textcolor{gray!50}{0.344} & 1.27 & 0.229 & 6.12 & 0.385 & 3.43 & 3.58 & 2.50 \\
\rowcolor{gray!20}DGPO  & 40 & 0.25 & \textcolor{gray!50}{0.348} & \textbf{1.30} & 0.230 & \textbf{6.20} & 0.384 & 3.41 & \textbf{3.62} & \textbf{2.52} \\
\bottomrule
\end{tabular}
\end{table}

\begin{figure}[t]
  \centering
  \includegraphics[width=\linewidth]{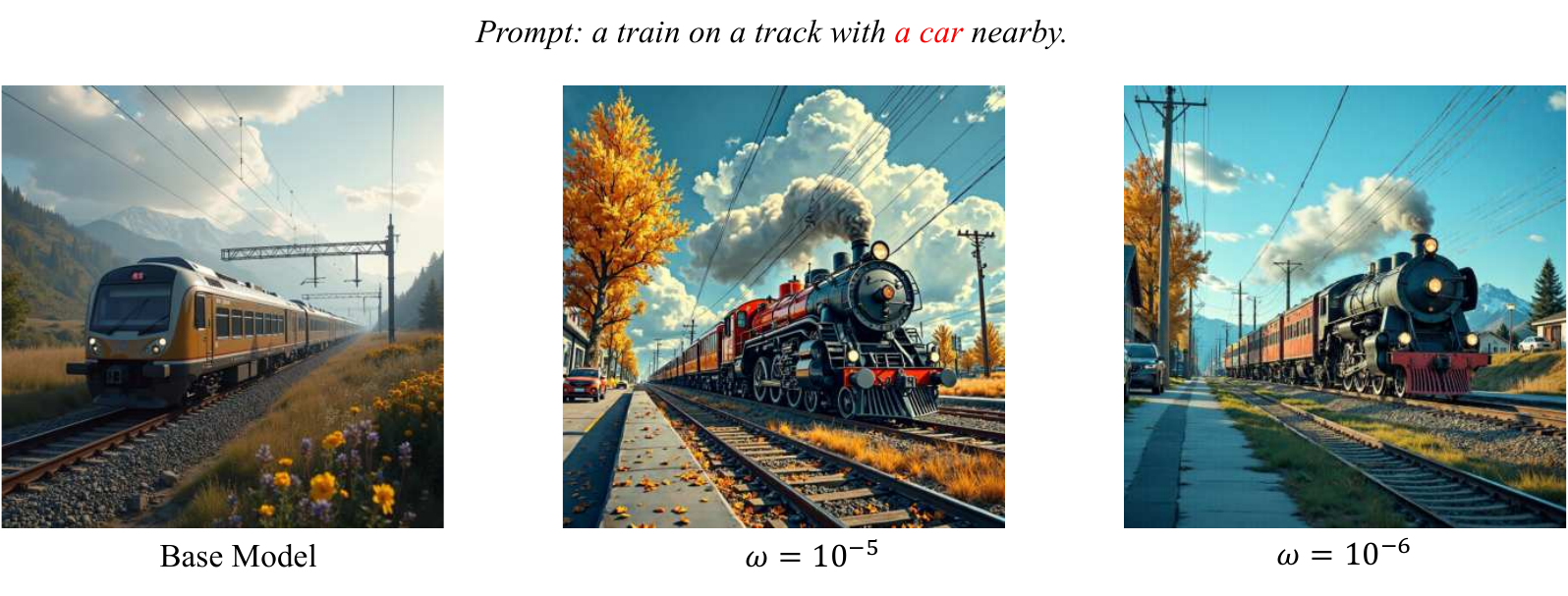}
  \caption{\textbf{Effect of the CPGO weight $\omega$ on optimization.} A larger weight, such as $10^{-5}$, can cause saturation, while moderate values of $10^{-6}$ or less strike a better balance between constraint strength and optimization flexibility.}
  \label{supp:vis_cpgo}
\end{figure}

\subsection{Parameter Selection Discussion}
\label{app:param_sensitivity_analysis}
In this section, we discuss the key hyperparameters of our DGPO, including the period of the inter-group dynamic-grained rollout, the coarse-grained ratio, the coarse perception timestep in the intra-group dynamic-grained rollout, the weight of CPGO, and the threshold used to filter unreliable timesteps. All hyperparameters are determined from validation experiments or heuristic studies such as Fig.~\ref{supp:app:Coarse_Progress_Perception}. Moreover, when applying our method to DPO and DDPO, we simply align their hyperparameter settings with those of DanceGRPO without additional tuning, highlighting the robustness and transferability of our hyperparameter choices. {{In practice, choosing appropriate hyperparameters is crucial; in what follows, we provide practical guidelines for selecting them.}}

\noindent \textbf{Coarse Perception Timestep $t_\text{s}$.}
The coarse perception timestep $t_{\mathrm{s}}$ specifies the denoising step at which we perform Coarse Progress Perception, allowing an intermediate noisy sample to serve as a proxy for the final clean image. To determine a suitable value, the study in Fig.~\ref{supp:app:Coarse_Progress_Perception} measures feature similarity and reward correlation with fully denoised samples at each candidate timestep. These analyses show that the reward signal becomes increasingly predictive as $t_{\mathrm{s}}$ moves towards later steps. Based on these results, we set $t_{\mathrm{s}}{=}12$ as our default choice and additionally consider a faster variant with $t_{\mathrm{s}}{=}8$, which provides a favorable trade-off between accuracy and efficiency. For methods using a sampling budget other than 16 steps, we recommend a similar visualization experiment to choose $t_{\mathrm{s}}$; this takes about 10 minutes of computation.

\noindent \textbf{Weight of CPGO $\omega$.}
The CPGO weight $\omega$ controls the strength of the consistency constraint that ties the current model's single-step ODE endpoint prediction to the teacher's prediction at the adjacent timestep along GRPO rollouts. Empirically, we observe that the GRPO loss is on the order of $10^{-7}$, whereas the CPGO loss is typically around $10^{-4}$. Since CPGO acts as a regularization term, we recommend searching $\omega$ in $\{10^{-5},10^{-6},10^{-7}\}$ so that the CPGO loss remains 2--4 orders of magnitude smaller than the GRPO loss.
As illustrated in Fig.~\ref{supp:vis_cpgo}, setting $\omega=10^{-5}$ can lead to saturation phenomena reminiscent of over-strong distillation effects in adversarial score distillation~\cite{Wei_2024_CVPR}, while $\omega=10^{-6}$ or smaller yields more desirable behavior, with smaller values gradually weakening the effect of the constraint. In practice, we therefore recommend choosing $\omega$ between $10^{-6}$ and $10^{-7}$. In all our experiments, we unifiedly fix $\omega=10^{-6}$ for DPO, DDPO, and GRPO.

\noindent \textbf{Period of Inter-Group DGR and Coarse-Grained Rollout Ratio.} We analyze the sensitivity of the method to two key hyperparameters: the period controlling the switching frequency between coarse-grained and fine-grained optimization, and the ratio governing the proportion of coarse-grained rollouts. To balance the effectiveness of each optimization type, we recommend selecting the period from \{20, 40, 50\} and the ratio from \{0.25, 0.5, 0.75\}. A more detailed discussion of how the coarse-grained ratio interacts with different types of reward functions is provided in App.~\ref{app:coarse_grained_optimization}. The results in Tab.~\ref{tab:Sensitivity} align with the analysis presented in App.~\ref{app:com_fine_coarse_dynamic}. Both excessively small and large periods tend to diminish performance. Importantly, the method demonstrates robustness to these hyperparameter choices, with relatively stable results across different configurations.

\section{Additional Method Details \& Discussion}

This section provides implementation details, training algorithms, clustering mechanisms, and unified comparisons across GRPO, DPO, and DDPO.
\label{app:method_details}
\subsection{Additional Preliminaries}
\label{app:preliminaries}

Flow matching~\cite{lipman2022flow} trains a time-dependent velocity field $v_\theta(\mathbf{x}_t, t)$ that transports a simple prior distribution to the data distribution via continuous-time dynamics. We summarize the training objective and inference procedure below.

\noindent \textbf{Training Objective.} The flow matching loss minimizes the discrepancy between the learned velocity field and a target conditional velocity along an interpolation path:
\begin{equation}
\mathcal{L}_{\mathrm{FM}}(\theta)
= \mathbb{E}_{t \sim \mathcal{U}(0,1),\, \mathbf{x}_0 \sim p_{\text{data}},\, \mathbf{x}_1 \sim \mathcal{N}(\mathbf{0},\sigma^2\mathbf{I})}
\Big[
  \|v_\theta(\mathbf{x}_t,t)
  -u_t(\mathbf{x}_t\mid\mathbf{x}_0,\mathbf{x}_1)\|_2^2
\Big],
\label{eq:fm_loss}
\end{equation}
where the linear interpolation and target velocity are defined by
\begin{equation}
\mathbf{x}_t = (1-t)\, \mathbf{x}_0 + t\, \mathbf{x}_1, \quad u_t(\mathbf{x}_t \mid \mathbf{x}_0, \mathbf{x}_1) = \mathbf{x}_1 - \mathbf{x}_0.
\label{eq:interpolation}
\end{equation}
This formulation directly regresses the network output $v_\theta$ to the conditional flow velocity $u_t$, avoiding path simulation and enabling efficient training.

\noindent \textbf{Inference via ODE.} Image synthesis in flow matching proceeds by integrating the deterministic probability flow ODE using the learned parameterized velocity field:
\begin{equation}
\frac{\mathrm{d}\mathbf{x}_t}{\mathrm{d}t} = v_\theta(\mathbf{x}_{t}, t), \qquad \mathbf{x}_{1} \sim \mathcal{N}(\mathbf{0}, \sigma^2 \mathbf{I}),
\label{eq:ode_inference}
\end{equation}
which integrates backward in continuous time from the noise boundary distribution at $t{=}1$ to the target data distribution at $t{=}0$. Under a numerical discretization with $T$ evaluation steps along the partition $1 = t_0 > t_1 > \cdots > t_T = 0$ with step size $\Delta t_n = t_{n+1} - t_n < 0$, a first-order Euler forward solver approximates the trajectory update as:
\begin{equation}
\mathbf{x}_{t_{n+1}} = \mathbf{x}_{t_n} + (t_{n+1} - t_n)\, v_\theta(\mathbf{x}_{t_n}, t_n).
\label{eq:single_step}
\end{equation}
When the velocity field exhibits straight optimal transport characteristics and satisfies $v_\theta(\mathbf{x}_t, t) \approx \mathbf{x}_1 - \mathbf{x}_0$, one can recover an instantaneous estimate of the clean data point $\mathbf{x}_0$ from any intermediate state $\mathbf{x}_t$ in a single ODE Euler step:
\begin{equation}
\mathbf{x}_{t\rightarrow 0}^{\theta} := \mathbf{x}_t - t\, v_\theta(\mathbf{x}_t, t),
\label{eq:pi_definition}
\end{equation}
where $t$ is the flow time. For the discrete sampler, step $t$ uses its noise level $s_t$ in place of $t$ (Sec.~\ref{sec:method}). It serves as the foundational prediction operator for our coarse progress perception and cross-step consistency regularization.

\subsection{Algorithm}
\label{app:overall}
\begin{algorithm}[t]
\caption{DGPO Training Algorithm}
\label{alg:DGPO_app}
\resizebox{\columnwidth}{!}{%
\begin{minipage}{\columnwidth}
\begin{algorithmic}[1]
\REQUIRE $v_\theta$, reward model $r$, prompts $\mathcal{D}_c$, group size $K$, step $t_\text{s}$, cluster size $K_1$
\FOR{iteration $m = 1, \ldots, M$}
    \STATE Sample $\mathcal{D}_b \sim \mathcal{D}_c$
    \FOR{prompt $\mathbf{c} \in \mathcal{D}_b$}
        \STATE \textit{$\triangleright$ Dynamic Granularity Rollout}
        \STATE Initialize $K$ noises: $\{\mathbf{x}_T^k\} \stackrel{\text{i.i.d.}}{\sim} \mathcal{N}(\mathbf{0}, \mathbf{I})$ (coarse) or shared $\mathbf{x}_T$ (fine)
        \FOR{$k = 1, \ldots, K$}
            \STATE SDE sampling $T \to t_\text{s}$: $\mathbf{x}_{t_\text{s}}^k \gets \text{SDE}(\mathbf{x}_T^k, v_\theta)$; record $\{\mathbf{v}_t^k, \mathbf{v}_t^{k,\text{old}}\}_{t \in [t_\text{s}:T]}$
            \STATE Coarse perception: $\hat{\mathbf{x}}_0^k = \mathbf{x}_{t_\text{s}}^k - s_{t_\text{s}} \mathbf{v}_{t_\text{s}}^k$
        \ENDFOR
        \STATE \textit{$\triangleright$ Clustering-based Selection \& Fine-Grained Refinement}
        \STATE Cluster $\{\mathrm{Enc}(\hat{\mathbf{x}}_0^k)\}$ via $k$-means $\to$ select $\mathcal{G}_1$, complement $\mathcal{G}_2$
        \FOR{$\mathbf{x}_{t_\text{s}}^k \in \mathcal{G}_1$}
            \STATE SDE sampling $t_\text{s} \to 0$: $\mathbf{x}_0^k \gets \text{SDE}(\mathbf{x}_{t_\text{s}}^k, v_\theta)$; record $\{\mathbf{v}_t^k, \mathbf{v}_t^{k,\text{old}}\}_{t \in [0:t_\text{s}]}$
        \ENDFOR
        \STATE \textit{$\triangleright$ Reward \& GRPO}
        \FOR{$\mathcal{G}_i \in \{\mathcal{G}_1, \mathcal{G}_2\}$}
            \STATE Evaluate $r_k=r(\mathbf{x}_0^k,\mathbf{c})$ for $\mathcal{G}_1$ and $r_k=r(\hat{\mathbf{x}}_0^k,\mathbf{c})$ for $\mathcal{G}_2$
            \STATE Compute $\bar r_i,\sigma_{r,i}$ within $\mathcal{G}_i$ and $A_k\gets(r_k-\bar r_i)/\sigma_{r,i}$
            \STATE Compute $\mathcal{J}_\text{GRPO}(\theta,\mathcal{G}_i)$ over its sampled transitions
        \ENDFOR
        \STATE \textit{$\triangleright$ CPGO \& Update}
        \STATE $\mathcal{J}_{\mathrm{GRPO}}(\theta) = \mathcal{J}_{\mathrm{GRPO}}(\theta, \mathcal{G}_1) + \mathcal{J}_{\mathrm{GRPO}}(\theta, \mathcal{G}_2)$
        \STATE Initialize accumulated loss $S\gets0$ and valid-step count $n\gets0$
        \FOR{each sampled step whose noise level satisfies $s_t\le\tau$}
            \STATE Let $\mathcal{K}_t$ contain candidates with both $\mathbf{x}_t^k$ and $\mathbf{x}_{t-1}^k$ cached
            \IF{$|\mathcal{K}_t|>0$}
                \STATE $S\gets S-\frac{1}{|\mathcal{K}_t|}\sum_{k\in\mathcal{K}_t}\left\|\mathbf{x}_{t\rightarrow0}^{\theta,k}-\mathbf{x}_{(t-1)\rightarrow0}^{\theta_{\mathrm{old}},k}\right\|_2^2$
                \STATE $n\gets n+1$
            \ENDIF
        \ENDFOR
        \STATE $\mathcal{L}_{\mathrm{CPGO}}\gets S/\max(1,n)$
        \STATE Update $\theta$ via $\nabla_\theta\!\left[\mathcal{J}_{\mathrm{GRPO}}(\theta) + \omega \cdot \mathcal{L}_{\text{CPGO}}(\theta)\right]$
    \ENDFOR
\ENDFOR
\end{algorithmic}
\end{minipage}%
}
\end{algorithm}

Algorithm~\ref{alg:DGPO_app} outlines the training pipeline of DGPO in a single pass per iteration. For each prompt, we initialize $K$ fine- or coarse-grained trajectories and perform SDE sampling over the coarse window $t\in[t_s,T]$, collecting intermediate states, velocities, and coarse ODE predictions. We embed and cluster these predictions to select a representative subset $\mathcal{G}_1$ and its complement $\mathcal{G}_2$. For $\mathcal{G}_1$, we continue SDE sampling from $t_s$ to $0$ and record velocities over $t\in[0,t_s]$. Finally, we compute rewards, normalize advantages separately for the two subsets, and optimize their combined GRPO losses with the CPGO consistency term, using the cached states and velocities to construct adjacent-step teacher targets.

\subsection{Details about Clustering-based Selection \& Fine-Grained Refinement}
\label{app:clustering_selection}
\noindent\textbf{Coarse-Grained Perception.} Here, we revisit the main text and restate our objective: to sample a representative group. Starting from Gaussian noise $\mathbf{x}_T \sim \mathcal{N}(\mathbf{0}, \mathbf{I})$, we first perform $t_{{\mathrm{s}}}$ sampling steps to obtain the noisy samples and their corresponding coarse noise-aware perceptions, as follows:
\begin{equation}
\mathbf{x}_{t_{\mathrm{s}}}^k
= \mathbf{x}_T^k
- \int_{t_{\mathrm{s}}}^{T}\left(
  \mathbf{v}_t^k-\tfrac{1}{2}\varepsilon_t^2
  \nabla\log p_t(\mathbf{x}_t^k)
\right)\mathrm{d}t
+ \int_{t_{\mathrm{s}}}^{T}
  \varepsilon_t\,\mathrm{d}\mathbf{w}_t^k,
\qquad
\hat{\mathbf{x}}_0^k
= \mathbf{x}_{t_{\mathrm{s}}}^k
  -s_{t_{\mathrm{s}}}\mathbf{v}_{t_{\mathrm{s}}}^k.
\label{app:eq:single_step_sde}
\end{equation}
where $\mathbf{v}_t^k = v_\theta(\mathbf{x}_t^k, t)$ denotes the model-predicted velocity, $\varepsilon_t$ is the noise schedule, $\mathbf{w}_t^k$ is a standard Brownian motion, and $t_{{\mathrm{s}}}$ is the intermediate perception step. The estimate $\hat{\mathbf{x}}_0^k$ corresponds to the single-step ODE prediction $\mathbf{x}_{t_{\mathrm{s}}\rightarrow 0}^{\theta}$.

\noindent\textbf{Clustering-based Selection.} By Eq.~\eqref{app:eq:single_step_sde}, we construct the group $\mathcal{G} = \{\mathbf{x}_t^k \mid k \in [1:K],\, t \in [t_{\mathrm{s}}:T]\}$, which aggregates noisy intermediate states together with their coarse noise-aware perceptions, as illustrated in Fig.~\ref{supp:app:coarse_opt}. We first project all coarse predictions $\{\hat{\mathbf{x}}_0^k\}_{k=1}^K$ into a semantic embedding space using the visual encoder of the reward model $\mathrm{Enc}(\cdot)$, and subsequently perform $k$-means clustering to identify $K_1$ structural mode centroids $\{\mathbf{c}_i\}_{i=1}^{K_1}$:
\begin{equation}
    \begin{aligned}
        \mathcal{C} = \{\mathbf{c}_i\}_{i=1}^{K_1}
        = \mathrm{KMeans}\!\left(\left\{\mathrm{Enc}\!\left(\hat{\mathbf{x}}_0^k\right)\right\}_{k=1}^K,\, K_1\right),
    \end{aligned}
\end{equation}
where $K_1 < K$ denotes the target number of representative rollout trajectories.

We then construct the active refinement subset $\mathcal{G}_1$ by selecting the sample closest in Euclidean latent distance to each respective cluster centroid, and assign the remaining candidates to the complement group $\mathcal{G}_2$:
\begin{equation}
    \begin{aligned}
        \mathcal{G}_1 &= \{\mathbf{x}_t^{i}\}_{i=1}^{K_1} = \mathrm{Select}\!\big(\mathcal{G}, \mathcal{C}\big), 
        \qquad
        \mathcal{G}_2 = \mathcal{G} \setminus \mathcal{G}_1,
    \end{aligned}
\end{equation}
where $\mathrm{Select}(\mathcal{G}, \mathcal{C})$ selects the single nearest candidate for each centroid $\mathbf{c}_i$.

\noindent\textbf{Fine-Grained Trajectory Continuation.} To enrich $\mathcal{G}_1$ with complete fine-grained trajectory information down to pixel-level synthesis, we continue the forward SDE integration exclusively for candidates in $\mathcal{G}_1$ from $t_{\mathrm{s}}$ to the terminal timestep $t=0$:
\begin{equation}
\mathbf{x}_{0}^k = \mathbf{x}_{t_\text{s}}^k
- \int_{0}^{t_\text{s}} \!\left(\mathbf{v}_t^k - \tfrac{1}{2}\varepsilon_t^2 \nabla\log p_t(\mathbf{x}_t^k)\right)\!\mathrm{d}t
+ \int_{0}^{t_\text{s}} \!\varepsilon_t \,\mathrm{d}\mathbf{w}_t^k,
\quad
\mathcal{G}_\text{f} \leftarrow \{\mathbf{x}_{t}^k \mid k \in \mathcal{G}_1,\, t \in [0:t_{\mathrm{s}}]\},
\quad
\mathcal{G}_1 \leftarrow \mathcal{G}_1 \;{\cup}\; \mathcal{G}_\text{f}.
\end{equation}
In contrast, candidates in $\mathcal{G}_2$ terminate at $t_{\mathrm{s}}$ and retain only their coarse intermediate predictions. Both subsets are then utilized to compute policy gradient updates at their respective granularity levels.

\noindent \textbf{Dual-Granularities Optimization.} Through cluster-based selection and fine-grained continuation, we obtain two complementary groups at distinct granularities: $\mathcal{G}_1$, comprising representative full-trajectory rollouts with complete temporal and high-frequency pixel details, and $\mathcal{G}_2$, retaining coarse, early-step perceptions. To leverage both information streams, we compute group-relative policy gradient objectives across both subsets independently:
The trajectory groups and the sum of their separately normalized objectives are described in Sec.~\ref{sec:intra_group_rollout}. Each component computes its respective group-centered advantage $\mathcal{A}^k$ relative to the empirical mean and standard deviation of its own candidate subset. This dual-objective structure prevents the coarse intermediate candidates in $\mathcal{G}_2$ from skewing the advantage normalization of the fully denoised candidates in $\mathcal{G}_1$, while simultaneously allowing early trajectory transitions at $t \in [t_{\mathrm{s}}, T]$ across all $K$ candidates to shape macroscopic policy updates.

\subsection{Coarse- and Fine-Grained Optimization}
\label{app:coarse_grained_optimization}
Existing post-training methods for rectified flow models differ in how they balance coarse- and fine-grained exploration. Sec.~\ref{sec:motivation} describes the fine-grained rollout strategy adopted by DanceGRPO~\cite{xue2025dancegrpo} and its variants, whereas methods such as FlowGRPO~\cite{liu2025flow} perform coarser exploration. \textit{This discrepancy is largely driven by the underlying task and reward design.} DanceGRPO primarily targets \textit{continuous, model-based rewards}, e.g., ImageReward~\cite{xu2023imagereward} or HPS, where even subtle changes in local details can be reflected in the reward signal, making fine-grained exploration effective. In contrast, FlowGRPO adopts a \textit{coarse-grained exploration scheme} that mainly adjusts global scene composition and object layout, making it difficult to focus on specific local regions. Fig.~\ref{supp:app:blurred_img} illustrates the resulting blur, and Fig.~\ref{supp:app:coarse_opt} shows detail under-optimization.

Beyond pure coarse- or fine-grained rollouts, dynamically mixing these granularities combines their strengths: coarse-grained exploration encourages semantic diversity at the scene level, while fine-grained refinement sharpens local details.

\subsection{Discussion about DGPO for Online DPO, DDPO \& GRPO}
\label{app:consistent_rft_discussion}

\noindent In App.~\ref{app:preliminaries}, we presented online DPO, GRPO, and DDPO from a unified perspective as schemes that follow a common exploration, comparison, and update pattern over model-generated samples. Here we elaborate on this connection and explain why DGPO can consistently improve all three methods in the online setting.

\noindent \textbf{Shared exploration and comparison paradigm.} In the exploration stage, these methods sample candidate trajectories or images from the current policy. A reward model then scores these samples, and the learning signal is constructed by comparing samples within a shared context. GRPO explicitly performs group-wise comparison: it normalizes rewards within a prompt-level group to obtain advantages and uses these advantages as weights when updating the policy, thereby encouraging samples with higher group-relative rewards. Online DPO, by contrast, forms positive and negative preference pairs and maximizes the probability gap between preferred and unpreferred samples. Similar insights on contrastive learning have also been made in recent studies~\cite{wu2025takes,tong2025delving}.

 The original DDPO formulation aggregates samples at the global level~\cite{xue2025dancegrpo,dance_grpo_issue2}, rather than organizing them into prompt-level groups. While this paradigm is conceptually simple, it becomes difficult to optimize stably at scale, especially when the reward landscape is highly heterogeneous across prompts~\cite{deng2024prdp,dance_grpo_issue2}. Recent works have reported pronounced instability~\cite{deng2024prdp,dance_grpo_issue2}. Motivated by the design of DanceGRPO, we therefore adopt a prompt-level group-based formulation for DDPO, while retaining its objective. This modification brings DDPO structurally closer to GRPO and DPO, and substantially improves optimization stability in our experiments.

\noindent \textbf{Why DGPO unifies improvements across DPO, GRPO, and DDPO.} Given that all three algorithms rely on exploration over trajectories and comparative feedback within groups or pairs, they can all benefit from dynamic-granularity rollout. DGPO provides such a rollout mechanism: it adaptively mixes coarse-grained exploration with fine-grained refinement and augments the base objective with a consistency regularizer that aligns ODE-based predictions with SDE trajectories.  As a result, DGPO serves as a unified plug-in that consistently improves DPO, GRPO, and DDPO in the online setting.

\section{Visual Hallucination Evaluator}
\label{app:vh_evaluator}
In this section, we describe the visual hallucination evaluator used to evaluate the performance of our method.
\subsection{Low-Level Evaluation}
We first describe the low-level image-space and latent-space metrics used by our visual hallucination evaluator to quantify detail over-optimization, artifacts, and noise level.

\subsubsection{Laplacian Variance}
Let $I: \Omega \to \mathbb{R}$ denote a grayscale image defined on domain $\Omega \subset \mathbb{Z}^2$. The Laplacian Variance~\cite{Bansal2016} is defined as the variance of the discrete Laplacian operator applied to $I$:
\begin{equation}
\mathcal{S}(I) = \text{Var}(\nabla^2 I),
\end{equation}
where the discrete Laplacian is given by
\begin{equation}
\nabla^2 I(x,y)
= I(x+1,y) + I(x-1,y) + I(x,y+1) + I(x,y-1) - 4I(x,y),
\end{equation}
and the variance is computed as
\begin{equation}
\text{Var}(\nabla^2 I) = \frac{1}{|\Omega|} \sum_{(x,y) \in \Omega} (\nabla^2 I(x,y) - \mu)^2,
\end{equation}
where $\mu = \frac{1}{|\Omega|} \sum_{(x,y) \in \Omega} \nabla^2 I(x,y)$ is the mean of the Laplacian response.

The Laplacian operator~\cite{Bansal2016} is a second-order differential operator that isolates high-frequency components of the image. In a sharp image, edges exhibit rapid intensity transitions, producing large-magnitude Laplacian responses with high variance. In contrast, blurred images exhibit smooth transitions with attenuated Laplacian responses and lower variance. This metric quantifies the degree of edge definition and is invariant to global intensity shifts.

\subsubsection{High-Frequency Energy}
Let $I: \Omega \to \mathbb{R}$ denote a grayscale image. The high-frequency energy~\cite{abdel2003analysis} is defined as the mean absolute response of a high-pass filter applied to $I$:
\begin{equation}
\mathcal{E}_{{\text{hf}}}(I) = \frac{1}{|\Omega|} \sum_{(x,y) \in \Omega} |H(x,y)|,
\end{equation}
where
\begin{equation}
H(x,y) = (I * K_{{\text{hp}}})(x,y) = \sum_{u,v} I(x-u, y-v) \cdot K_{{\text{hp}}}(u,v)
\end{equation}
is the high-pass filtered image obtained via convolution, and the high-pass filter kernel $K_{{\text{hp}}}$ is defined as
\begin{equation}
K_{{\text{hp}}} = \begin{bmatrix} -1 & -1 & -1 \\ -1 & 8 & -1 \\ -1 & -1 & -1 \end{bmatrix}.
\end{equation}

High-frequency components represent rapid spatial variations in image intensity. The magnitude of high-frequency energy serves as an indicator of image complexity, texture richness, and potential artifacts. In the context of image enhancement, excessive high-pass filtering amplifies both legitimate details and noise, producing characteristic artifacts. The mean absolute response quantifies the overall magnitude of these high-frequency components.

\subsubsection{Edge Metric}
Let $I: \Omega \to \mathbb{R}$ denote a grayscale image. The edge metric quantifies ringing artifacts by measuring intensity variation in edge neighborhoods~\cite{ziou1998edge, abdel2003analysis}.

\noindent \textbf{Edge Detection.} Edges are detected using the Canny edge detector:
\begin{equation}
E(x,y) = \text{Canny}(I; \tau_1, \tau_2),
\end{equation}
where $\tau_1$ and $\tau_2$ are the lower and upper gradient magnitude thresholds, respectively. The Canny detector produces a binary edge map $E: \Omega \to \{0,1\}$.

\noindent \textbf{Morphological Dilation.} The edge map is dilated to capture the edge neighborhood:
\begin{equation}
E_{{\text{dil}}}(x,y) = (E \oplus B)(x,y) = \max_{(u,v) \in B} E(x-u, y-v),
\end{equation}
where $\oplus$ denotes morphological dilation and $B$ is a $3 \times 3$ structuring element. Dilation is applied iteratively $n = 2$ times to expand the edge region.

\noindent \textbf{Artifact Quantification.} The edge artifact metric is defined as the standard deviation of pixel intensities within the dilated edge region:
\begin{equation}
\mathcal{A}_{{\text{edge}}}(I) = \sqrt{\frac{1}{|M|} \sum_{(x,y) \in M} (I(x,y) - \bar{I}_M)^2},
\end{equation}
where $M = \{(x,y) \in \Omega : E_{{\text{dil}}}(x,y) > 0\}$ is the set of pixels in the dilated edge region, and $\bar{I}_M = \frac{1}{|M|} \sum_{(x,y) \in M} I(x,y)$ is the mean intensity in $M$.

Over-sharpening produces characteristic ringing artifacts at edges, manifesting as bright and dark halos adjacent to edge transitions through the Gibbs phenomenon. These artifacts are characterized by abnormally high intensity variation in edge neighborhoods. By restricting the standard deviation computation to edge regions, this metric isolates the contribution of ringing artifacts while minimizing contamination from natural texture variation. The metric is grounded in the theory of edge-preserving filtering and artifact characterization.

\subsubsection{Noise Estimation}
Let $I: \Omega \to \mathbb{R}$ denote a grayscale image. The noise level is estimated from smooth, low-texture regions where pixel variations are primarily attributable to noise~\cite{Shin2005, Rakhshanfar2016}.

\noindent \textbf{Local Texture Characterization.} We first compute the local standard deviation at each pixel to characterize local texture:
\begin{equation}
\sigma_{{\text{local}}}(x,y) = \sqrt{\frac{1}{K^2} \sum_{(u,v) \in W(x,y)} (I(u,v) - \bar{I}_{W})^2},
\end{equation}
where $W(x,y)$ is a square window of size $K \times K$ centered at $(x,y)$, and $\bar{I}_W$ is the local mean within the window.

\noindent \textbf{Smooth Region Identification.} Smooth regions are identified as those with local standard deviation below the 30th percentile:
\begin{equation}
M_{{\text{smooth}}} = \{(x,y) \in \Omega : \sigma_{{\text{local}}}(x,y) < P_{30}(\sigma_{{\text{local}}})\},
\end{equation}
where $P_{30}$ denotes the 30th percentile of the local standard deviation distribution.

\noindent \textbf{Noise Map Computation.} The noise map is computed as the absolute difference between the original image and a Gaussian-smoothed version:
\begin{equation}
N_{{\text{map}}}(x,y) = |I(x,y) - G_{\sigma}(I)(x,y)|,
\end{equation}
where $G_{\sigma}$ is a Gaussian low-pass filter with standard deviation $\sigma = 1.0$ and kernel size $5 \times 5$.

\noindent \textbf{Noise Level Estimation.} The noise level is estimated as the mean of the noise map restricted to smooth regions:
\begin{equation}
\mathcal{L}_{{\text{noise}}}(I) = \frac{1}{|M_{{\text{smooth}}}|} \sum_{(x,y) \in M_{{\text{smooth}}}} N_{{\text{map}}}(x,y).
\end{equation}
This metric quantifies the noise level by averaging noise map values over regions with minimal texture, providing a robust estimate of additive noise.

In smooth image regions with minimal texture, the difference between the original image and a slightly smoothed version approximates the noise component, as texture is attenuated by the smoothing operation while noise remains largely unchanged. By restricting this estimation to regions identified as smooth via local standard deviation, the metric avoids contamination from texture and edge structures. This approach is grounded in the principle that noise is signal-independent and spatially uncorrelated, while texture exhibits spatial correlation and concentrates in high-variance regions.

\subsubsection{Latent Consistency}
Flow matching models learn to transform noise into data by modeling a continuous-time ODE. The latent consistency metric quantifies the deviation of the learned flow from the ideal linear interpolation path, providing a direct geometric measure of trajectory straightness, vector field stability, and generative fidelity.

\noindent \textbf{Reverse Process.} Given a flow matching model parameterized by velocity field $\mathbf{v}_\theta: \mathbb{R}^d \times [0,1] \to \mathbb{R}^d$, the reverse deterministic sampling process follows the probability flow ODE:
\begin{equation}
\frac{\mathrm{d}\mathbf{x}_t}{\mathrm{d}t} = \mathbf{v}_\theta (\mathbf{x}_t,t),
\qquad
\mathbf{x}_t = \mathbf{x}_1 - \int_{t}^{1} \mathbf{v}_\theta (\mathbf{x}_s,s)\; \mathrm{d}s,
\label{eq:reverse}
\end{equation}
where $\mathbf{x}_t \in \mathbb{R}^d$ denotes the continuous latent representation at intermediate timestep $t \in [0,1]$, with $\mathbf{x}_1 \sim \mathcal{N}(\mathbf{0}, \mathbf{I})$ representing the initial Gaussian noise sample and $\mathbf{x}_0$ the generated clean latent representation. The continuous integral is numerically approximated across $T$ discrete timesteps using standard schemes such as first-order Euler or higher-order Runge--Kutta solvers.

\noindent \textbf{Ideal Flow Trajectory.} In the mathematical framework of optimal transport flow matching, the path connecting the standard Gaussian prior distribution $p_1 = \mathcal{N}(\mathbf{0}, \mathbf{I})$ and the complex target data distribution $p_0$ corresponds to linear geodesics. Given a clean terminal latent $\mathbf{x}_0$ synthesized via Eq.~\eqref{eq:reverse} and the corresponding source noise vector $\mathbf{x}_1 \sim \mathcal{N}(\mathbf{0}, \mathbf{I})$, the theoretical ideal unperturbed trajectory state at intermediate time $t$ is:
\begin{equation}
\mathbf{x}_t^{\text{ideal}} = (1-t) \cdot \mathbf{x}_0 + t \cdot \mathbf{x}_1, \quad t \in [0,1].
\label{eq:ideal}
\end{equation}
This linear interpolation path minimizes the dynamic kinetic energy $\int_0^1 \|v(\mathbf{x}_t, t)\|_2^2 \mathrm{d}t$ under the $L^2$ Wasserstein metric, establishing a straight trajectory in latent space between $\mathbf{x}_1$ and $\mathbf{x}_0$.

\noindent \textbf{Latent Consistency Metric.} The latent consistency metric directly measures the mean squared deviation between the actual reverse trajectory $\{\mathbf{x}_t\}_{t \in [0,1]}$ integrated by the model velocity field and the ideal straight-line trajectory $\{\mathbf{x}_t^{\text{ideal}}\}_{t \in [0,1]}$:
\begin{equation}
\text{Latent Consistency}(\mathbf{v}_\theta) = \frac{1}{\text{dim}(\mathbf{x}_0)} \mathbb{E}_{t \sim [0,1]} \left\| \mathbf{x}_{t} - \mathbf{x}_{t}^{\text{ideal}} \right\|_2^2,
\label{eq:consistency}
\end{equation}
where expectation is evaluated across uniformly sampled intermediate timesteps $\{t_i\}_{i=1}^T$ in $[0,1]$, with intermediate states $\mathbf{x}_{t_i}$ computed via Eq.~\eqref{eq:reverse}.

The corresponding finite-sample calculation, with $d=\mathrm{dim}(\mathbf{x}_0)$, is
\begin{equation}
\mathcal{C}_{\mathrm{traj}} =
\frac{1}{Td}\sum_{i=1}^{T}
\left\|\mathbf{x}_{t_i}-\big[(1-t_i)\mathbf{x}_0+t_i\mathbf{x}_1\big]\right\|_2^2.
\label{eq:vh_trajectory_main}
\end{equation}

This metric serves as a rigorous diagnostic tool for assessing flow field integrity: a higher latent consistency value indicates severe curvature and trajectory divergence away from optimal transport paths—a pathology directly induced by SDE-based reinforcement fine-tuning that correlates with visual artifacts and off-manifold hallucinations. Conversely, a lower latent consistency value signifies that the learned velocity field faithfully adheres to the deterministic straight-path geometry of flow matching.

\subsection{High-Level Evaluation}
\label{app:vh_high_level}
Similar to PrefGRPO~\cite{wang2025pref}, we complement the low-level metrics with a high-level evaluation based on a pretrained multimodal large language model (MLLM) instantiated as Qwen2.5-VL-72B~\cite{bai2025qwen25vl}. The prompt below is the verbatim instruction. It includes an example comparison, shown in Fig.~\ref{supp:vh_example}, so the model can separate over-optimization from moderate detail. The conditioning text may be omitted when the image is not generated from a prompt.

\begin{figure}[t]
  \centering
  \includegraphics[width=\linewidth]{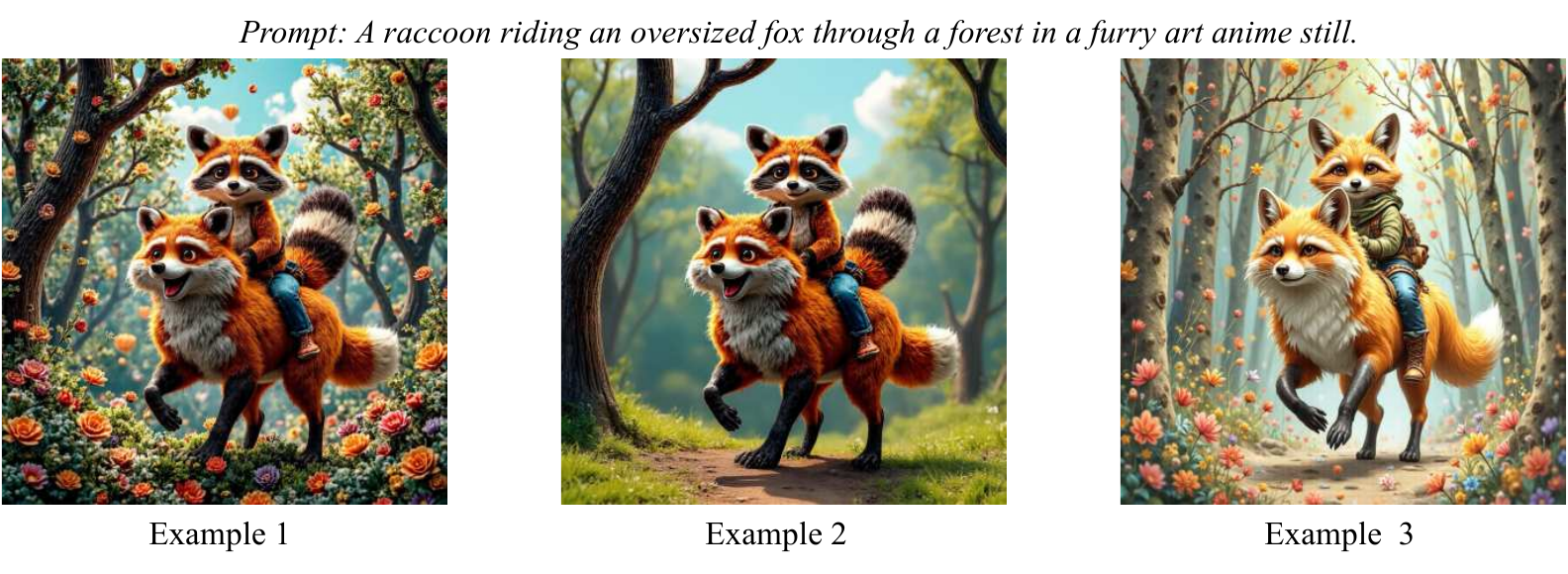}
  \caption{\textbf{Reference examples used in the VH-Evaluator prompt.} Example 1 versus Example 2 contrasts irrelevant detail accumulation with moderate detailing. Example 1 versus Example 3 contrasts over-sharpened details with natural sharpening. These comparisons correspond to the ``Example Comparison'' section in the prompt.}
  \label{supp:vh_example}
\end{figure}

\begin{promptbox}{VH-Evaluator instruction \hfill \normalfont\footnotesize verbatim}
\footnotesize\ttfamily
You assess over-optimization hallucinations in the image for ``<prompt>''.\\
The image to score is attached below.\\[3pt]
\# Hallucination types\\
1. Grid patterns: regular grids, repetitive textures, or other structural artifacts.\\
2. Texture over-enhancement: detail that makes a surface look unnaturally rough, smooth, or harsh.\\
3. Irrelevant details: decoration or objects that the prompt does not ask for.\\[3pt]
\# Scores\\
Scores range from 0 (none) to 5 (extreme); decimal values are allowed:\\
- detail\_sharpness\_score: edge harshness, unnatural sharpening, high-frequency noise.\\
- irrelevant\_details\_score: extra details, over-decoration, clutter.\\
- grid\_pattern\_score: a grid laid over the generated image.\\
overall\_hallucination\_score is the mean of the three.\\
confidence is a number in [0, 1].\\[3pt]
\# Example comparison\\
Example 1 is over-optimized. Example 2 has moderate detail. Example 3 has natural sharpening.\\
- Irrelevant details: Example 1 should score well above Example 2.\\
- Sharpness: Example 1 should score well above Example 3.\\
- Clutter like Example 1 raises irrelevant\_details\_score.\\
- Harsh edges like Example 1 raise detail\_sharpness\_score.\\
- Examples 2 and 3 should receive lower hallucination scores.\\[3pt]
\# Output\\
Return strict JSON only, with no markdown fence:\\
\{\\
\ \ "image": \{\\
\ \ \ \ "has\_over\_optimization": true/false,\\
\ \ \ \ "over\_optimization\_level": 0-5,\\
\ \ \ \ "detail\_sharpness\_score": 0-5,\\
\ \ \ \ "irrelevant\_details\_score": 0-5,\\
\ \ \ \ "grid\_pattern\_score": 0-5,\\
\ \ \ \ "overall\_hallucination\_score": 0-5,\\
\ \ \ \ "confidence": 0-1,\\
\ \ \ \ "detailed\_analysis": "short justification"\\
\ \ \}\\
\}
\end{promptbox}

\subsection{Visual Hallucination Evaluator Results}
\label{app:vh_evaluator_results}

Fig.~\ref{supp:vh_result} displays qualitative evaluation results from the VH-Evaluator, showcasing images at varying degrees of over-optimization hallucination. The results illustrate how the MLLM-based evaluator successfully identifies and differentiates between different hallucination severity levels, from minor over-optimization artifacts to severe detail accumulation and sharpening artifacts.

\begin{figure}[t]
  \centering
  \includegraphics[width=\linewidth]{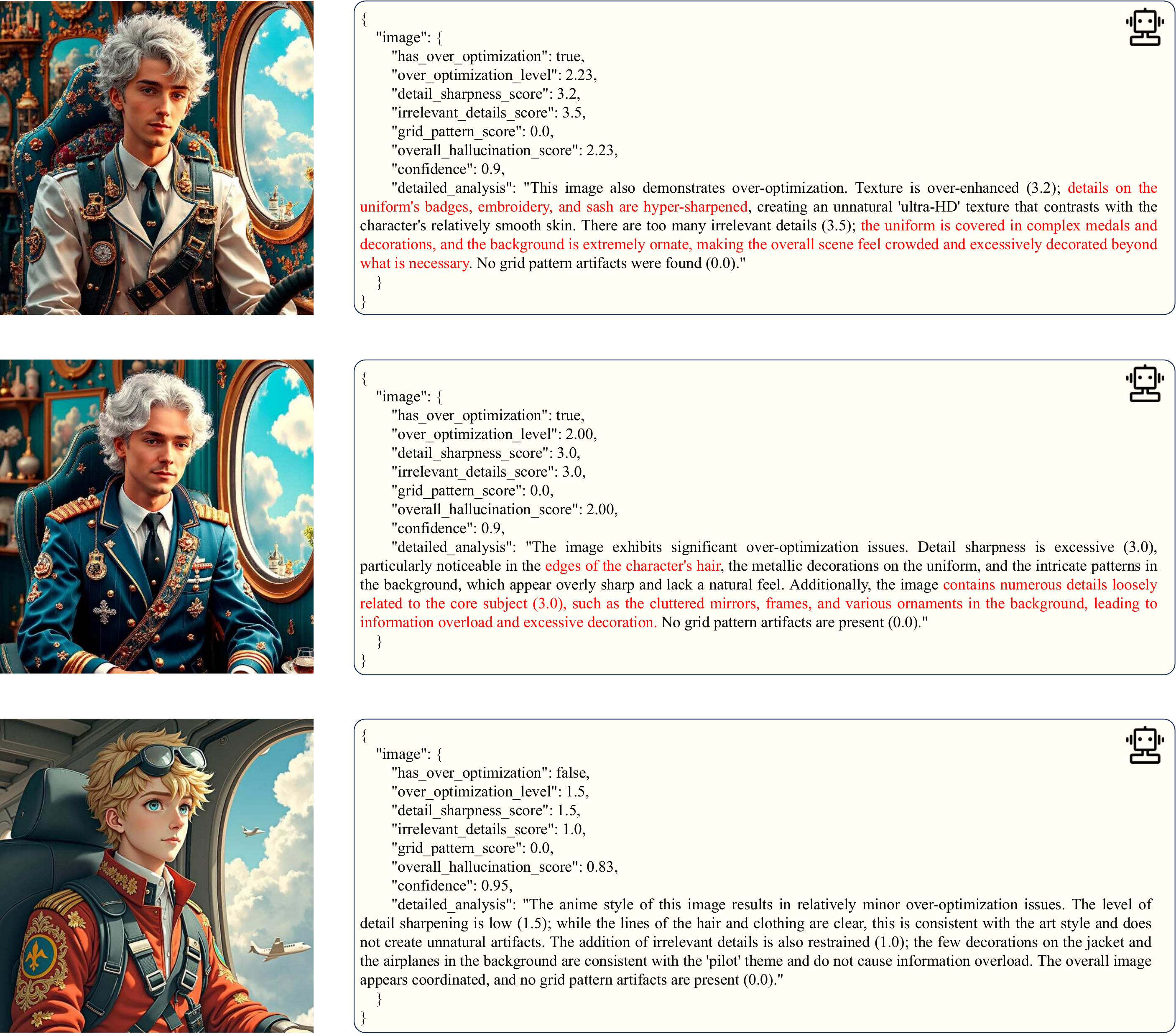}
  \caption{\textbf{Qualitative results from the VH-Evaluator demonstrating detection of over-optimization hallucinations at varying levels.}}
  \label{supp:vh_result}
\end{figure}

\section{Theoretical Justification}
\label{app:theoretic}

\subsection{Proof of Corollary~\ref{cor:local_group_update}}
\label{app:proof_local_group_update}

\newcounter{corollarybackup}

\begin{proof}[Proof of Corollary~\ref{cor:local_group_update}]
The standardized advantages satisfy $\sum_k\mathcal{A}^k=0$. Hence, for any reference trajectory $\bar{\mathbf{x}}_{0:T}$, subtracting the common term $\boldsymbol{\psi}_\theta(\bar{\mathbf{x}}_{0:T},c)$ gives Eq.~\eqref{eq:group_difference_identity}. Let $\boldsymbol{\delta}^k=\mathbf{x}_{0:T}^k-\bar{\mathbf{x}}_{0:T}$. Taylor's theorem gives
\begin{equation}
\boldsymbol{\psi}_\theta(\mathbf{x}_{0:T}^k,c)-\boldsymbol{\psi}_\theta(\bar{\mathbf{x}}_{0:T},c)
=\mathbf{J}_{\bar{\mathbf{x}}}\boldsymbol{\delta}^k+\mathbf{e}^k,
\qquad
\|\mathbf{e}^k\|_2\leq\frac{M}{2}\|\boldsymbol{\delta}^k\|_2^2.
\end{equation}
Substitution into Eq.~\eqref{eq:group_difference_identity}, with $\mathbf{e}_{\mathcal{G}}=\frac{1}{K}\sum_k\mathcal{A}^k\mathbf{e}^k$, yields Eq.~\eqref{eq:fine_grained_gradient}. Finally, the triangle and Cauchy--Schwarz inequalities give
\begin{equation}
\|\mathbf{e}_{\mathcal{G}}\|_2
\leq\frac{M}{2K}\sum_k|\mathcal{A}^k|\|\boldsymbol{\delta}^k\|_2^2
\leq\frac{M}{2}\max_k\|\boldsymbol{\delta}^k\|_2^2,
\end{equation}
where $\frac{1}{K}\sum_k|\mathcal{A}^k|\leq\sqrt{\frac{1}{K}\sum_k(\mathcal{A}^k)^2}\leq1$.
\end{proof}

\subsection{Proof of Corollary~\ref{thm:reinterpretation}}
\label{app:thm:reinterpretation}
\setcounter{corollarybackup}{\value{corollary}}
\begingroup
\renewcommand{\thecorollary}{\ref{thm:reinterpretation}}
\renewcommand{\theHcorollary}{appendix.grpo-imitation}
\begin{corollary}[Reinterpretation of GRPO]
\label{cor:grpo_imitation_appendix}
For Gaussian discretized SDE transitions with covariance $\eta^2\Delta t\,\mathbf{I}$ independent of $\theta$, hold sampled states and advantages fixed. Away from clipping boundaries, the GRPO gradient satisfies
\begin{equation*}
\nabla_{{\theta}}\mathcal{J}_\mathrm{GRPO}
=\mathbb{E}\!\left[
  -\mathbb{I}(\neg u)\,\omega(t)\,\rho_{t,i}\,A_i\,\nabla_{\theta}
  \left\|\mu_{\theta}(\mathbf{x}_t,t,\mathbf{c})-\mathbf{x}_{t-1}\right\|_2^2
\right].
\end{equation*}%
where $\mu_\theta(\cdot)$ is the mean of $p_\theta (\mathbf{x}_{t-1} \mid \mathbf{x}_{t},\mathbf{c})$, $\omega(t)=1/(2\eta^2\Delta t)$, $u=(u_1\wedge v_1)\vee(u_2\wedge\neg v_1)$ is evaluated separately for each transition, and
\begin{equation*}
u_1 \leftarrow \mathbb{I}(\rho_{t,i}\leq1-\epsilon),\quad
u_2 \leftarrow \mathbb{I}(\rho_{t,i}\geq1+\epsilon),\quad
v_1 \leftarrow \mathbb{I}(A_i<0).
\end{equation*}%
\end{corollary}
\endgroup
\setcounter{corollary}{\value{corollarybackup}}

\begin{proof}[Proof of Corollary~\ref{thm:reinterpretation}]
\label{app:theoretic:cannyor}
The clipped GRPO objective is
\begin{equation}
\mathcal{J}_{{\mathrm{GRPO}}}(\theta) = \mathbb{E}_{\mathbf{c}\sim \mathcal{D}_c,\,
\tau^i \sim \pi_{\theta_{\mathrm{old}}}(\cdot\mid\mathbf{c})} \left[
\min\left( \rho_{t,i} A_i, \mathrm{clip}\left( \rho_{t,i}, 1-\epsilon, 1+\epsilon \right) A_i \right)
\right],
\qquad
\rho_{t,i} = \frac{p_{\theta}(\mathbf{x}_{t-1}^i\mid\mathbf{x}_{t}^i, \mathbf{c})}{p_{\theta_{\mathrm{old}}}(\mathbf{x}_{t-1}^i\mid\mathbf{x}_{t}^i, \mathbf{c})}.
\label{app:eq:grpo_objective}
\end{equation}%
With fixed cached rollouts and advantages, the expectation is a finite empirical average and can be differentiated term by term. The population version additionally assumes an integrable local bound on the gradient. Define the indicators
\begin{equation}
u_1 \leftarrow \mathbb{I}(\rho_{t,i} \leq 1-\epsilon), \quad
u_2 \leftarrow \mathbb{I}(\rho_{t,i} \geq 1+\epsilon), \quad
v_1 \leftarrow \mathbb{I}(A_i < 0),
\label{app:eq:indicator_vars}
\end{equation}%
the per-sample surrogate decomposes as
\begin{equation}
\min\left( \rho_{t,i} A_i, \mathrm{clip}\left( \rho_{t,i}, 1-\epsilon, 1+\epsilon \right) A_i \right)
= (1-\epsilon) A_i \;\text{if } u_1 \wedge v_1,\quad
(1+\epsilon) A_i \;\text{if } u_2 \wedge \neg v_1,\quad
\rho_{t,i} A_i \;\text{otherwise}.
\label{app:eq:clipped_decomposition}
\end{equation}%
In the first two cases the surrogate is constant in $\theta$ and its gradient vanishes. Otherwise, since $\nabla_\theta\rho_{t,i}=\rho_{t,i}\nabla_\theta\log p_\theta(\mathbf{x}_{t-1}^i\mid\mathbf{x}_t^i,\mathbf{c})$,
\begin{equation}
\nabla_\theta\big(\rho_{t,i}A_i\big)
= \rho_{t,i}\,A_i\,\nabla_\theta \log p_{\theta}(\mathbf{x}_{t-1}^i\mid\mathbf{x}_{t}^i, \mathbf{c}),
\label{app:eq:gradient}
\end{equation}%
which reduces to the on-policy form $A_i\nabla_\theta\log p_\theta$ at $\theta=\theta_{\mathrm{old}}$, where $\rho_{t,i}=1$.
Under the assumed Gaussian discretization, the transition density is
\begin{equation}
p_\theta (\mathbf{x}_{t-1}\mid\mathbf{x}_{t}, \mathbf{c}) = \mathcal{N}\left(\mathbf{x}_{t-1};\,\mu_\theta(\mathbf{x}_{t}, t, \mathbf{c}),\, \eta^2\Delta t\,\mathbf{I}\right),
\qquad
\log p_\theta (\mathbf{x}_{t-1}\mid\mathbf{x}_{t}, \mathbf{c}) = - \frac{\Vert \mu_\theta(\mathbf{x}_{t}, t, \mathbf{c}) - \mathbf{x}_{t-1}\Vert_2^2}{2\eta^2 \Delta t} + \text{const},
\label{app:eq:log_prob}
\end{equation}%
where the constant does not depend on $\theta$. Substituting Eq.~\eqref{app:eq:log_prob} into Eq.~\eqref{app:eq:gradient} with $\omega(t)=1/(2\eta^2\Delta t)$ gives the stated form.
\end{proof}

\begin{proof}[Derivation of Eq.~\eqref{eq:reward_noise}]
At $\theta=\theta_{\mathrm{old}}$, write the sampled state as $\mathbf{x}_{t-1}=\mu_{\theta_{\mathrm{old}}}(\mathbf{x}_t,t,\mathbf{c})+\eta\sqrt{\Delta t}\,\boldsymbol{\xi}_t$ with $\boldsymbol{\xi}_t\sim\mathcal{N}(\mathbf{0},\mathbf{I})$ independent of $\mathbf{x}_t$. Then $\mu_\theta-\mathbf{x}_{t-1}=-\eta\sqrt{\Delta t}\,\boldsymbol{\xi}_t$ and
\begin{equation}
-\omega(t)\,A_i\,\nabla_\theta\|\mu_\theta-\mathbf{x}_{t-1}\|_2^2
=-2\omega(t)A_i\Big(\frac{\partial\mu_\theta}{\partial\theta}\Big)^{\!\top}(\mu_\theta-\mathbf{x}_{t-1})
=\frac{A_i}{\eta\sqrt{\Delta t}}\Big(\frac{\partial\mu_\theta}{\partial\theta}\Big)^{\!\top}\boldsymbol{\xi}_t,
\end{equation}%
Taking expectations gives Eq.~\eqref{eq:reward_noise}.
\end{proof}

\subsection{Corollary for DDPO \& DPO}
\label{app:theoretic:cannyor_dpo}

\subsubsection{Corollary for DDPO}
\begin{corollary}[DDPO Trajectory Imitation]
\label{cor:ddpo_imitation}
Given a flow model $\theta$, a reward model $r(\mathbf{x},c)$, and trajectories sampled via the SDE in Eq.~\eqref{eq:sde}, the on-policy DDPO gradient satisfies
\begin{equation*}
\nabla_{\theta}\mathcal{J}_\mathrm{DDPO} = \mathbb{E}\!\left[ -\omega(t) \cdot r(\mathbf{x}_0^k, \mathbf{c}) \cdot \nabla_{\theta} \left\| \mu_{\theta}(\mathbf{x}_{t}, t, \mathbf{c}) - \mathbf{x}_{t-1} \right\|_2^2 \right],
\end{equation*}%
where $\mu_\theta(\cdot)$ denotes the mean of $p_\theta(\mathbf{x}_{t-1} \mid \mathbf{x}_{t}, \mathbf{c})$ and $\omega(t)=1/(2\eta^2\Delta t)$.
\end{corollary}

\begin{proof}[Proof of Corollary~\ref{cor:ddpo_imitation}]
With rollouts drawn from $\pi_{\theta_{\mathrm{old}}}$, the importance-sampled DDPO objective is
\begin{equation}
\mathcal{J}_{{\mathrm{DDPO}}}(\theta) = \mathbb{E}_{\mathbf{c}\sim \mathcal{D}_c,\,
\tau^k \sim  \pi_{\theta_{\mathrm{old}}}(\cdot\mid\mathbf{c}),\, t \sim \mathcal{U}(0,T),\, k \sim \mathcal{U}(1,K)}
\left[ \frac{p_{\theta}(\mathbf{x}^{k}_{t-1} \mid \mathbf{x}^{k}_t, \mathbf{c})}{p_{\theta_\text{old}}(\mathbf{x}^{k}_{t-1} \mid \mathbf{x}^{k}_t, \mathbf{c})}\, r(\mathbf{x}^k_0,\mathbf{c}) \right].
\label{app:eq:ddpo_objective}
\end{equation}%
As in the GRPO case, the gradient of the ratio is the ratio times the log-likelihood gradient. At $\theta=\theta_{\mathrm{old}}$ the ratio equals one, giving
\begin{equation}
\nabla_\theta \mathcal{J}_{{\mathrm{DDPO}}}(\theta) = \mathbb{E}\left[ r(\mathbf{x}^k_0,\mathbf{c}) \cdot \nabla_\theta \log p_{\theta}(\mathbf{x}^{k}_{t-1} \mid \mathbf{x}^{k}_t, \mathbf{c}) \right].
\label{app:eq:ddpo_gradient_1}
\end{equation}%
Using the Gaussian log-likelihood in Eq.~\eqref{app:eq:log_prob},
\begin{equation}
\nabla_\theta \log p_\theta (\mathbf{x}_{t-1}\mid\mathbf{x}_{t}, \mathbf{c}) = - \omega(t)\,\nabla_\theta \Vert \mu_\theta(\mathbf{x}_{t}, t, \mathbf{c}) - \mathbf{x}_{t-1}\Vert_2^2,
\label{app:eq:ddpo_log_prob}
\end{equation}%
Substituting into Eq.~\eqref{app:eq:ddpo_gradient_1} gives the stated form.
\end{proof}

\subsubsection{Corollary for DPO}

\begin{corollary}[DPO Trajectory Imitation]
\label{cor:dpo_imitation}
Given a flow model $\theta$, a reference model $\theta_{{\mathrm{ref}}}$, and preference pairs $(\mathbf{x}^1, \mathbf{x}^2)$ where $\mathbf{x}^1 \succ \mathbf{x}^2$ sampled via the SDE in Eq.~\eqref{eq:sde}, the gradient of DPO satisfies
\begin{equation*}
\nabla_{{\theta}} \mathcal{J}_\mathrm{DPO}(\theta)
= \mathbb{E} \Big[ (1 - \sigma(m_t)) \cdot \Big( -\omega(t) \, \nabla_\theta \big( \big\Vert \mu_\theta(\mathbf{x}_{t}^1, t, \mathbf{c}) - \mathbf{x}_{t-1}^1\big\Vert_2^2
- \big\Vert \mu_\theta(\mathbf{x}_{t}^2, t, \mathbf{c}) - \mathbf{x}_{t-1}^2\big\Vert_2^2 \big) \Big) \Big]
\end{equation*}%
where
\begin{equation*}
 m_t = \beta \log \frac{p_{{\theta}}(\mathbf{x}^{1}_{t-1} \mid \mathbf{x}^{1}_t, c)}{p_{{\theta_{{\mathrm{ref}}}}}(\mathbf{x}^{1}_{t-1} \mid \mathbf{x}^{1}_t, c)} - \beta \log \frac{p_{{\theta}}(\mathbf{x}^{2}_{t-1} \mid \mathbf{x}^{2}_t, c)}{p_{{\theta_{{\mathrm{ref}}}}}(\mathbf{x}^{2}_{t-1} \mid \mathbf{x}^{2}_t, c)}
\end{equation*}%
denotes the per-timestep preference margin, $\beta > 0$ is the temperature parameter, and $\omega(t)>0$ depends only on $\beta$ and the Gaussian variance, not on $\theta$. For example, the parameterization in Eq.~\eqref{app:eq:ddpo_log_prob} gives $\omega(t) = \tfrac{\beta}{2\eta^2\Delta t}$.
\end{corollary}

\begin{proof}[Proof of Corollary~\ref{cor:dpo_imitation}]
The DPO objective for preference pairs is:
\begin{equation}
\mathcal{J}_{{\mathrm{DPO}}}(\theta) = \mathbb{E}_{c\sim\mathcal{D},\mathbf{x}^{1}\succ \mathbf{x}^{2}, t\sim\mathcal{U}(0,T)} \left[ \log \sigma\left( \beta \log \frac{p_{{\theta}}(\mathbf{x}^{1}_{t-1} \mid \mathbf{x}^{1}_t, c)}{p_{{\theta_{{\mathrm{ref}}}}}(\mathbf{x}^{1}_{t-1} \mid \mathbf{x}^{1}_t, c)} - \beta \log \frac{p_{{\theta}}(\mathbf{x}^{2}_{t-1} \mid \mathbf{x}^{2}_t, c)}{p_{{\theta_{{\mathrm{ref}}}}}(\mathbf{x}^{2}_{t-1} \mid \mathbf{x}^{2}_t, c)} \right) \right].
\end{equation}%

Define the per-timestep preference margin
\begin{equation}
 m_t := \beta \log \frac{p_{{\theta}}(\mathbf{x}^{1}_{t-1} \mid \mathbf{x}^{1}_t, c)}{p_{{\theta_{{\mathrm{ref}}}}}(\mathbf{x}^{1}_{t-1} \mid \mathbf{x}^{1}_t, c)} - \beta \log \frac{p_{{\theta}}(\mathbf{x}^{2}_{t-1} \mid \mathbf{x}^{2}_t, c)}{p_{{\theta_{{\mathrm{ref}}}}}(\mathbf{x}^{2}_{t-1} \mid \mathbf{x}^{2}_t, c)},
\end{equation}%
so that $\mathcal{J}_{{\mathrm{DPO}}}(\theta) = \mathbb{E}[\log \sigma(m_t)]$. Taking the gradient and applying the chain rule with the sigmoid derivative $\sigma'(z) = \sigma(z)(1-\sigma(z))$ yields
\begin{equation}
\nabla_\theta \mathcal{J}_{{\mathrm{DPO}}}(\theta) = \mathbb{E} \left[ (1 - \sigma(m_t)) \cdot \nabla_\theta m_t \right].
\label{app:eq:dpo_grad_margin}
\end{equation}%

Next, we express $m_t$ in terms of the Gaussian likelihoods used by the flow model. For flow-matching models, the conditional distribution at step $t$ is Gaussian with variance $\eta^2\Delta t$: 
\begin{equation}
p_\theta (\mathbf{x}_{t-1}\mid\mathbf{x}_t, \mathbf{c}) = \mathcal{N}\left(\mathbf{x}_{t-1};\,\mu_\theta(\mathbf{x}_{t}, t, \mathbf{c}),\, \eta^2\Delta t\,\mathbf{I}\right),
\end{equation}%
so that
\begin{equation}
\log p_\theta (\mathbf{x}_{t-1}|\mathbf{x}_t, \mathbf{c}) = - \frac{\big\Vert \mu_\theta(\mathbf{x}_{t}, t, \mathbf{c}) - \mathbf{x}_{t-1}\big\Vert_2^2}{2\eta^2 \Delta t} + \text{const}.
\end{equation}%
where the constant is independent of $\theta$. Substituting the Gaussian log-likelihoods and collecting the fixed reference terms into the constant gives
\begin{equation}
 m_t = -\frac{\beta}{2\eta^2\Delta t} \Big( \big\Vert \mu_\theta(\mathbf{x}_{t}^1, t, \mathbf{c}) - \mathbf{x}_{t-1}^1\big\Vert_2^2 - \big\Vert \mu_\theta(\mathbf{x}_{t}^2, t, \mathbf{c}) - \mathbf{x}_{t-1}^2\big\Vert_2^2 \Big) + \text{const}.
\end{equation}%
Taking the gradient with respect to $\theta$ gives
\begin{equation}
\nabla_\theta m_t = -\frac{\beta}{2\eta^2\Delta t} \, \nabla_\theta \Big( \big\Vert \mu_\theta(\mathbf{x}_{t}^1, t, \mathbf{c}) - \mathbf{x}_{t-1}^1\big\Vert_2^2 - \big\Vert \mu_\theta(\mathbf{x}_{t}^2, t, \mathbf{c}) - \mathbf{x}_{t-1}^2\big\Vert_2^2 \Big).
\end{equation}%

Substituting this expression into Eq.~\eqref{app:eq:dpo_grad_margin} and defining $\omega(t) := \tfrac{\beta}{2\eta^2\Delta t} > 0$, we obtain
\begin{equation}
\nabla_\theta \mathcal{J}_{{\mathrm{DPO}}}(\theta)
= \mathbb{E} \left[ (1 - \sigma(m_t)) \cdot \left( -\omega(t) \, \nabla_\theta \Big( \big\Vert \mu_\theta(\mathbf{x}_{t}^1, t, \mathbf{c}) - \mathbf{x}_{t-1}^1\big\Vert_2^2 - \big\Vert \mu_\theta(\mathbf{x}_{t}^2, t, \mathbf{c}) - \mathbf{x}_{t-1}^2\big\Vert_2^2 \Big) \right) \right]
\end{equation}%
The proof is completed.
\end{proof}

\clearpage
\section{Visualization}
\label{app:visualization}
Here, we provide visualizations corresponding to the additional experimental results reported in App.~\ref{app:add_exp_results}.

\begin{figure}[!htbp]
  \centering
  \includegraphics[width=\linewidth]{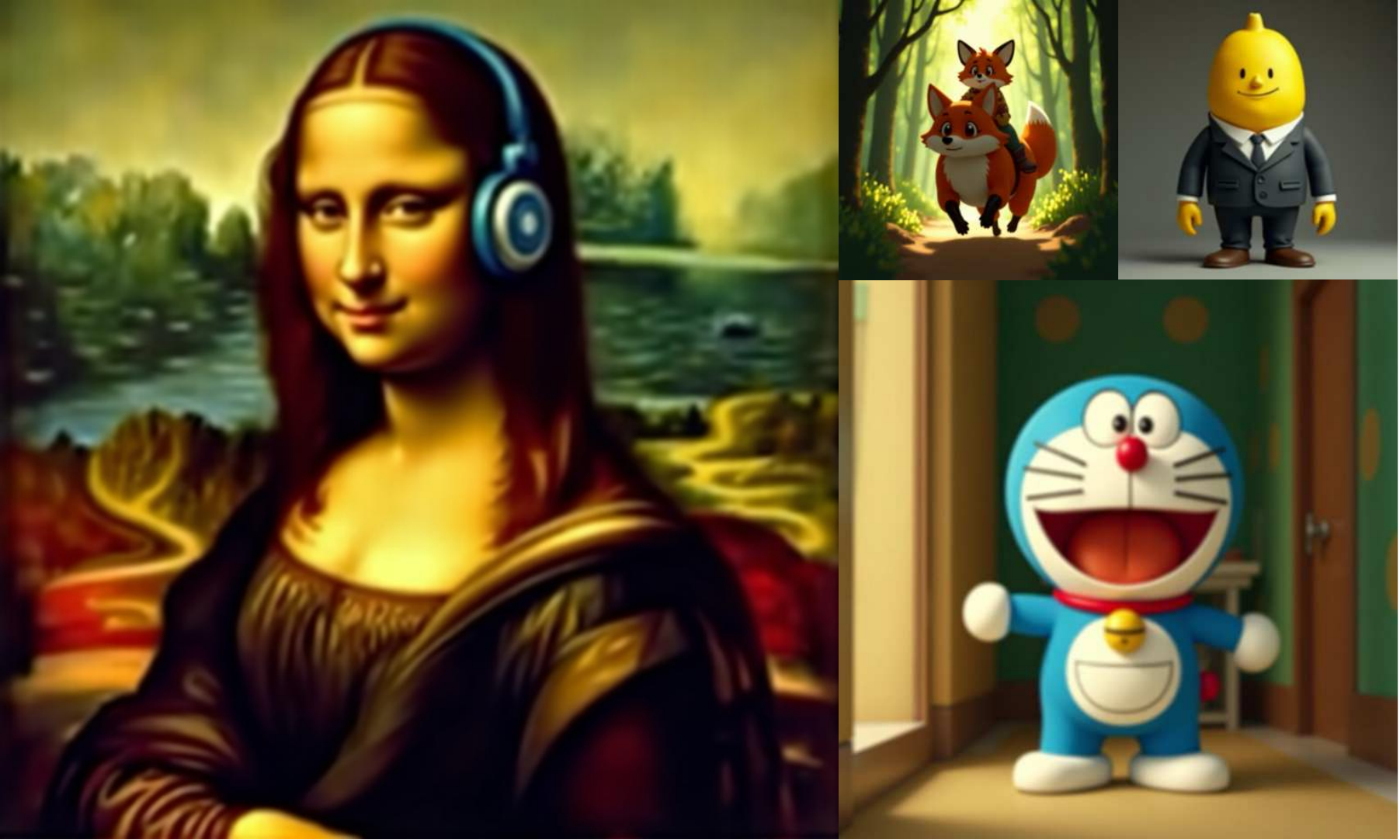}
  \caption{\textbf{Enlarged view of blurred image.} \textit{{These results are reproduced using official model weights or implementation settings of existing methods, with 8 GPUs.}}}
  \label{supp:app:blurred_img}
\end{figure}

\begin{figure}[t]
  \centering
  \includegraphics[width=\linewidth]{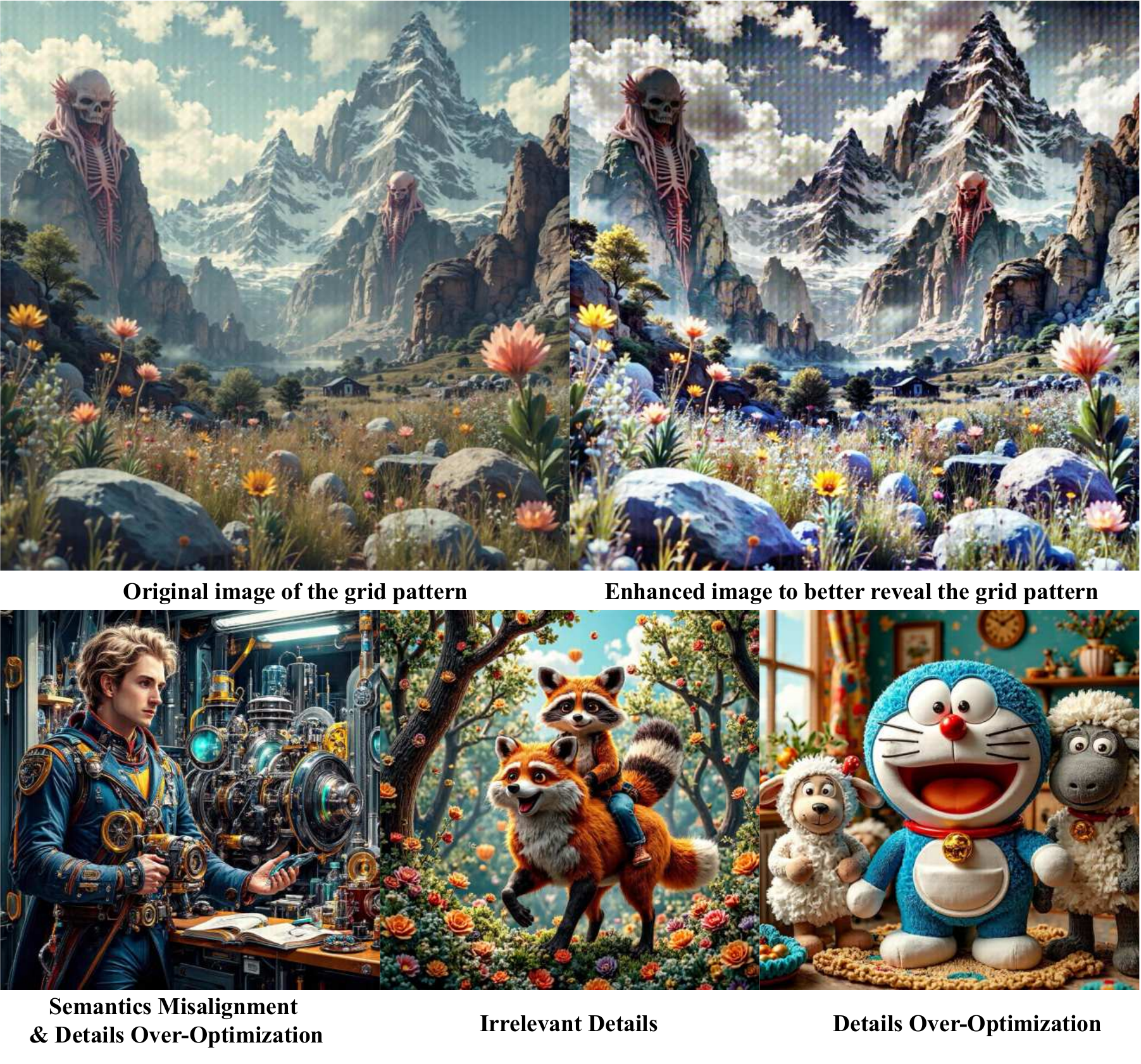}
  \caption{\textbf{Enlarged view of Fig.~1 from the Sec. Introduction.} }
  \label{supp:vh_large}
\end{figure}

\clearpage

\begin{figure}[t]
  \centering
  \includegraphics[width=0.75\linewidth]{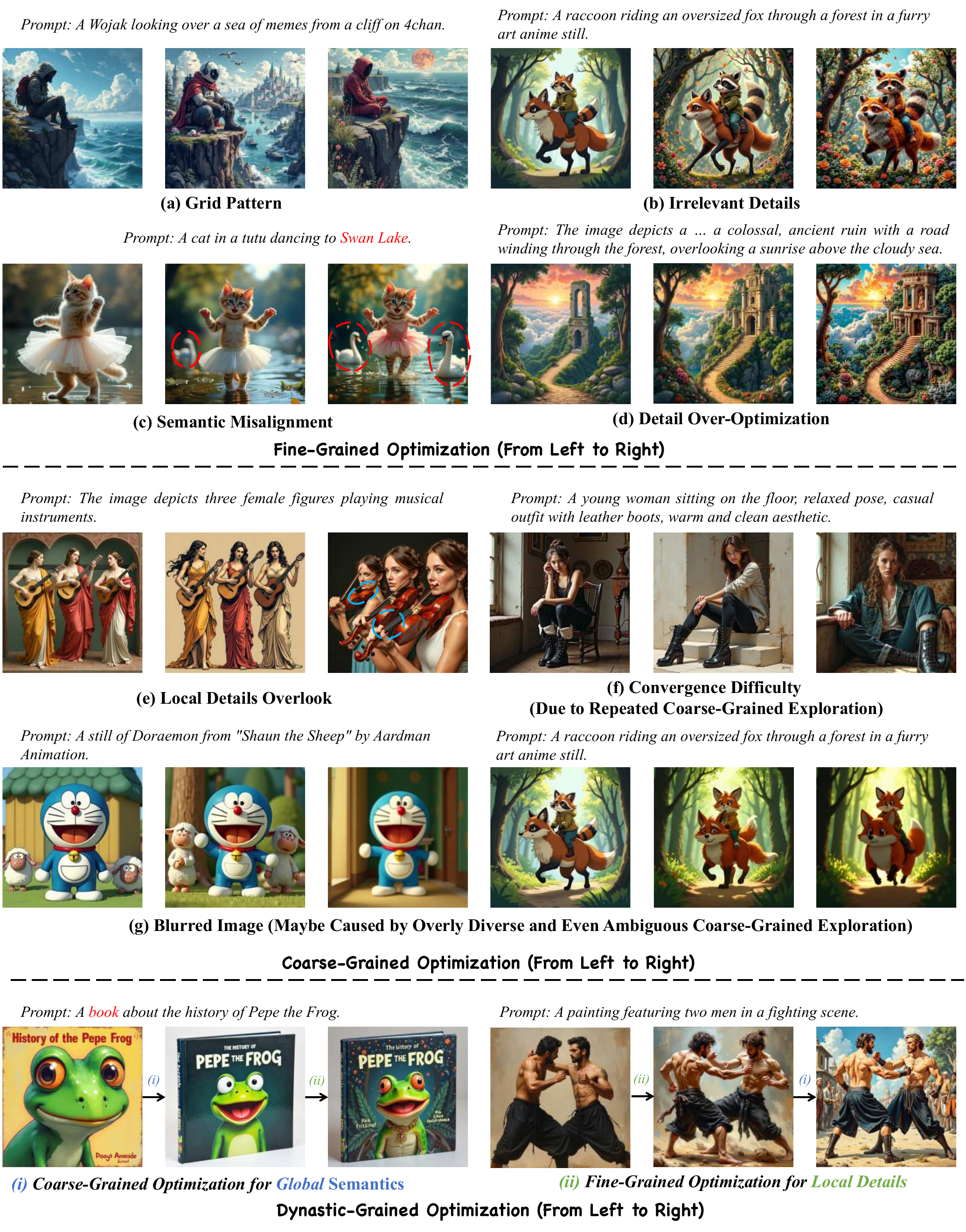}
  \caption{\textbf{Qualitative comparison of coarse-, fine-, and dynamic-grained optimization.} }
  \label{supp:app:coarse_opt}
\end{figure}

\end{document}